\documentclass{article}
\usepackage[margin=1.95cm]{geometry}
\usepackage{graphicx}
\usepackage[hidelinks]{hyperref}
\usepackage{amsmath,amsthm,amsfonts,amssymb,amscd,mathrsfs}
\usepackage{enumerate}
\usepackage{mathrsfs}
\usepackage{graphicx}
\usepackage[shortlabels]{enumitem}
\usepackage{bbm}
\usepackage{float}
\usepackage{fancyvrb}
\usepackage{booktabs}
\usepackage{subcaption}
\usepackage[T1]{fontenc}
\usepackage{yhmath}

\usepackage[ruled,vlined]{algorithm2e}
\def\algbackskip{\hskip-\ALG@thistlm}
\newtheorem{theorem}{Theorem}[section]
\newtheorem{definition}[theorem]{Definition}
\newtheorem{notation}[theorem]{Notation}
\newtheorem{corollary}[theorem]{Corollary}
\newtheorem{proposition}[theorem]{Proposition}
\newtheorem{lemma}[theorem]{Lemma}

\newtheorem{assumption}[theorem]{Assumption}
\newtheorem{example}[theorem]{Example}

\newtheorem*{theorem*}{Problem}
\theoremstyle{remark}
\newtheorem{remark}[theorem]{Remark}

\newcommand{\R}{{\mathbb R}}
\newcommand{\E}{{\mathbb E}}
\newcommand{\N}{{\mathbb N}}

\newcommand{\inprod}[2]{\left\langle #1, #2 \right\rangle}

\title{Universal Approximation of Nonlinear Operators and Their Derivatives}

\author{ Filippo de Feo\footnote{Institut für Mathematik, Technische Universität Berlin. Email: defeo@math.tu-berlin.de}}
\begin{document}

\maketitle

\begin{abstract}
We show that Universal Approximation (UA) of nonlinear operators and their derivatives via Operator Learning (OL) architectures fails in ${C^k_F}$ (Fr\'echet) compact-open topologies and in Fr\'echet--Sobolev norms (i.e. under operator norms). We solve this obstruction by restoring UA  in  natural weaker topologies: $C^k_B$ (Bastiani) compact-open topologies and (novel) weighted Bastiani--Sobolev spaces for general finite input measures. In full Banach-space generality, these are the first complete generalizations of the corresponding influential classical results in [Hornik, 1991] to infinite-dimensional spaces and OL. Based on our UATs, we formulate Bastiani--Sobolev training in DIOL. These results launch Derivative-Informed Operator Learning (DIOL) (i.e. learning nonlinear operators and their derivatives) on general Banach spaces. 
We parameterize nonlinear operators via Encoder-Decoder Architectures, classical OL architectures available in general Banach spaces; these include DeepONets, Deep-H-ONets, and PCA-Nets, which our UATs cover.
  
A key mathematical result is that our new weighted Bastiani--Sobolev spaces generalize classical Gaussian (Malliavin) Sobolev spaces on Banach spaces.

   Open frontiers where DIOL and our UATs find applications are: high-order accuracy in OL; fast constrained optimization in Banach spaces (e.g. optimal control of PDEs, inverse problems) via Learn-Then-Optimize; numerical methods for infinite-dimensional PDEs (e.g. HJB PDEs on Banach spaces from infinite-dimensional optimal control via Optimize-Then-Learn, such as optimal control of PDEs, SPDEs, path-dependent systems, partially observed systems, mean-field control).
\end{abstract}
\begin{small}
\tableofcontents
\end{small}
\section{Introduction}
\subsection{\texorpdfstring{{Main Results}}{Model Problem and Main Results}}\label{subsec:main_results}
Let $X,Y$ be Banach spaces and let $F:X\to Y$ be an unknown {(nonlinear)}
$k$-times differentiable operator. In \emph{Derivative-Informed Operator Learning} (DIOL), one uses Operator Learning (OL) architectures to \[
\textit{Learn } F \textit{ and Its Derivatives } D^iF, \textit{ } 1\le i\le k.
\]
DIOL arises naturally in PDE-constrained optimization,
inverse problems, and numerical methods for infinite-dimensional PDEs (see Section
\ref{subsec:motivations} for motivations and examples; see Section \ref{subsec:state_art} for a literature review).

In this manuscript, we focus on approximation guarantees for OL architectures, namely Universal Approximation (UA): can they approximate $F$ together with its derivatives, and in which topologies?
\paragraph{Universal Approximation is classical when $X=\mathbb R^N, Y=\mathbb R^{N'}$.} Classical Universal Approximation Theorems (UATs) {in} Machine Learning (ML) guarantee that 
neural networks can simultaneously approximate $F$, together with {its derivatives}, {in the following two senses} \cite{hornik1991approximation} {(see also \cite{hornik-stinchcombe-white,pinkus1999approx})}:
\begin{enumerate}[(A), leftmargin=3em]
    \item {in} $C^k$ compact-open topologies (i.e. topologies of uniform convergence  of functions and their derivatives on compacts){, generated by the seminorms $\max_{0\leq i\leq k}\sup_{x\in K}|D^iF(x)|$, $K\subset X$ compact (where $D^0F:=F$)};
    \item  in weighted Sobolev spaces, for any finite input measure\footnote{{As in \cite{hornik1991approximation}, here and in \ref{eq:sobolev_op_intro}, \ref{eq:Sobolev_norm_intro} below, a weighted Sobolev space is a space of $k$-times continuously differentiable maps with finite weighted Sobolev (semi)norm, endowed with this (semi)norm. For a general finite measure, this is not a Sobolev space in the usual sense: derivatives are classical rather than weak, the space is not complete in general, and a completion is not taken. See also the comments on the finite-dimensional case in Subsection \ref{subsec:motivation_sobolev}.}} $\nu$ on $X${, i.e. under the weighted Sobolev norm $\big(\sum_{i=0}^k\int_X|D^iF(x)|^p\,d\nu(x)\big)^{1/p}$}.
    \end{enumerate}
\paragraph{Universal Approximation fails under operator norms.} Let now $X,Y$ be infinite-dimensional.  For $k=0$, the approximation of $F$ is
the standard problem of Operator Learning, supported by numerous UATs for Encoder-Decoder Architectures
DeepONets, Fourier Neural Operators, and related
models.

When $k\geq 1$, however,  the problem is substantially more delicate. Let $F:X\to Y$ admit  directional derivatives $D^iF(x)(h^1,\ldots,h^i)$ (i.e. $DF(x)h:=\lim_{t\to 0} \frac{F(x+th)-F(x)}{t},$ and iterating for $i\geq 2$), lying in the spaces of multilinear bounded operators $\mathcal L^i(X,Y)$:
$$D^iF(x)\in \mathcal L^i(X,Y).$$
The operator norm $\|L\|_{\mathcal L^i(X,Y)}:=\sup \left\{\left|L\left(h^1, \ldots, h^i\right)\right |_Y: \left|h^j\right|_X = 1, \forall j\leq i\right\}, L\in \mathcal L^i(X,Y)$ makes the spaces $\mathcal L^i(X,Y)$ Banach spaces. {Requiring that  $X\ni x\mapsto D^iF(x)\in \mathcal L^i(X,Y) $ is continuous (i.e. in operator norm) for all $i\leq k$, leads us to  the space of $k$-times continuously Fréchet differentiable maps {${C^k_F}(X,Y)$.}}

Topologies induced by  operator norms are often too strong (even for $i=1$), as distances are measured with respect to worst deviations in infinite-dimensional unit spheres.   In fact, {our first result (Theorem \ref{th:counterexample}) reveals} a fundamental obstruction to the direct extensions of (A)--(B) to Banach spaces, namely
$$\textit{Universal Approximation fails under operator norms:}$$
\begin{enumerate}[label={(\Alph*$_{\mathrm{F}}$)}, ref={(\Alph*$_{\mathrm{F}}$)}, leftmargin=3em]
    \item\label{eq:frechet_CO_intro} in standard {${C^k_F}$ (Fréchet) compact-open topologies}, i.e. generated by the family of seminorms\footnote{in the following, we use the notation $D^0F(x)=F(x)$, see Notation \ref{not:notation_D0f}}
    \[\max_{0\leq i \leq k}\sup_{ x\in K}  \|D^iF(x)\|_{\mathcal L^i(X,Y)}, \quad K\subset X\textit{ compact subset};\]
    \item\label{eq:sobolev_op_intro} in weighted Fréchet--Sobolev spaces, i.e. constructed under the weighted Fréchet--Sobolev norm  ({$\nu\neq0$ is any finite measure on $X$ with finite $p$-th moment}),
    \[\|F\|_{\mathcal W^{k,p}_{\nu}(X,Y)}=\left(\sum_{i=0}^k\int_{X} \|D^iF(x)\|_{\mathcal L^i(X,Y)}^pd\nu(x)\right)^{1/p}.\]
\end{enumerate} {This occurs even for linear maps, showing that the obstruction is structural, rather than pathological.}
{This makes the development of these UATs a delicate problem in nonlinear functional analysis and raises the fundamental question:}
$$
\textit{can we restore Universal Approximation in infinite dimensions?}
$$

\paragraph{Encoder-Decoder Architectures.}
To restore  UA we represent nonlinear operators in general Banach spaces via  Encoder-Decoder Architectures (EDAs) \cite{kovachki-lanthaler-stuart} (see Section \ref{sec:EDN}), classical OL architectures available on general Banach spaces. 
These have the form
\begin{equation}\label{eq:EDN_intro}
F^{{\mathcal E}^X,\theta,{\mathcal D}^Y}
=
{\mathcal D}^Y\circ f^{N,\theta,N'}\circ{\mathcal E}^X,
\end{equation}
where ${\mathcal E}^X \in \mathcal L( X , \R^{N})$  encodes the input,
$f^{N,\theta,N'}:\R^N\to\R^{N'}$ is a finite-dimensional approximator, and
${\mathcal D^{Y}} \in \mathcal L( \R^{N'} , Y),$ decodes its output.    EDAs include many classical OL architectures,  such as \textbf{DeepONets} (typically defined on Banach spaces of continuous functions), Deep-H-ONets/HGNOs, Reduced-Basis Neural Operators, and PCA-Nets (see Section \ref{sec:EDN}), which we discuss in detail, as we want to cover these architectures with our approach.

\paragraph{Universal Approximation is restored under Bastiani topologies.} {To overcome the obstruction, we weaken the topology in which derivatives are approximated,
seeking a natural formulation in functional analysis and compatible with OL training.  To this purpose, we view each derivative $D^iF$ through its  joint evaluation map}
\begin{equation}\label{eq:bastiani_intro}
 X\times X^i\ni (x,h^1,\ldots,h^i)
\mapsto
 D^iF(x)(h^1,\ldots,h^i)\in Y .
\end{equation}
{If $F\in {C^k_F}(X,Y)$ then  \eqref{eq:bastiani_intro} is jointly continuous on $X\times X^i$, for all $0\le i\le k$.
Requiring the maps in \eqref{eq:bastiani_intro} to be jointly continuous for
all $0\le i\le k$, leads us to the class of $k$-times continuously Bastiani
differentiable maps}
\cite{bastiani1964} (see also
\cite{keller2006}){, denoted by}
$C^k_B(X,Y)$. {The class $C^k_B(X,Y)$ is a very convenient framework in functional analysis and in OL, as it provides full calculus under minimal regularity:}
when
$X,Y$ are infinite-dimensional, this is in general a strictly weaker
notion than $k$-times continuous Fréchet differentiability, i.e.
${C^k_F}(X,Y)\mathrel{{\subsetneq}} C^k_B(X,Y)$. 
 For instance, Nemytskii/superposition operators (in OL these appear in Neural Operators) are $C^1_B(X,X)$ but not ${C^1_F}(X,X)$ on $X=L^2(\mathbb T^d)$, see Subsection \ref{subsec:nemistkii-bastiani}; whereas $C^k_B(\mathbb R^N,\mathbb R^{N'})=C^k_F(\mathbb R^N,\mathbb R^{N'})$. 

The formulation \eqref{eq:bastiani_intro} is especially well suited to the UA, leading us  to weaker, more permissive, topologies where UA is achieved naturally{, namely to the following Bastiani versions of (A)--(B)}:
\begin{enumerate}[label={(\Alph*$_{\mathrm{B}}$)}, ref={(\Alph*$_{\mathrm{B}}$)}, leftmargin=3em]
    \item\label{eq:CkB_bastiani_compactopen_intro} {in} {$C^k_B$ (Bastiani) compact-open topologies} {(Theorem \ref{th:UAT_K_Banach})}, i.e. generated by the family of seminorms\footnote{recall that the unit spheres used in \ref{eq:frechet_CO_intro} to compute operator norms are not compact sets when $X$ is infinite-dimensional}
    \[\max_{0\leq i \leq k}\sup_{ x\in K}\sup_{h^1,...,h^i\in K'}  |D^iF(x)(h^1,...,h^i)|_Y, \quad K,K'\subset X\textit{ compact subsets}.\]
    Hence, here, we are endowing $\mathcal L^i(X,Y)$ with a compact-open topology; $X$ is required to have the Approximation {Property} (AP, Assumption \ref{ass:approx_pro})\footnote{AP {provides} Moore--Smith sequences of finite-rank operators converging to the identity uniformly on compacts. It is satisfied by most Banach spaces arising in applications, including Banach spaces with Schauder frames or bases and, therefore, separable Hilbert spaces; see, however, \cite{enflo1973counterexample} for a counterexample.}{;}
            \item\label{eq:Sobolev_norm_intro} {in} weighted Bastiani--Sobolev spaces {(Theorem \ref{th:UAT_Sobolev_Banach} and its corollaries)}{, obtained from a natural product-space viewpoint on weighted Sobolev norms}, i.e. constructed under the weighted Bastiani--Sobolev norm (or variants, see Appendix \ref{subsec:Sobolev})
    \[\|F\|_{\mathcal W^{k,p,r}_{B,\nu,\eta}(X,Y)}=\left(\sum_{i=0}^k\int_{X}\left( \int_{ X^i} |D^iF(x)(h^1,...,h^i)|^r_Yd\eta^{1:i}(h^1,\ldots,h^i)\right)^{p/r}d\nu(x)\right)^{1/p},\]
    where $\nu$ is a finite input measure on $X$, $\eta$ is a measure on $X^k$ with finite $r$-moments  measuring norms in directions $h^j,j\leq k$, $\eta^{1:i}$ denote its marginals. Hence, here, we are endowing $\mathcal L^i(X,Y)$ with a weighted $L^r$-norm $\|L\|_{L^r(X^i,Y,\xi)}:=\left(\int_{X^i} |L(h^1,...,h^i)|^r_Yd\xi(h^1,\ldots,h^i)\right)^{1/r}, L\in \mathcal L^i(X,Y)$ for $\xi=\eta^{1:i}$ (see Appendix \ref{subsec:norm_Lp-multilinear}). {The construction of weighted Bastiani--Sobolev spaces is one of the main mathematical results of the paper:  UA  is achieved naturally, and, as we will see below, they are natural from a training perspective and they are consistent with finite-dimensional UA theory \cite{hornik1991approximation}   and with Gaussian Sobolev spaces over Banach spaces \cite{bogachev1998gaussian}. We motivate the construction in Subsection \ref{subsec:motivation_sobolev}.}
\end{enumerate}
Using these tools, 
\textit{we prove UATs for nonlinear $k$-times
differentiable operators and all their derivatives.}
All these results apply to \textbf{DeepONets}, Deep-H-ONets, and Schauder Operator Nets (Remark \ref{rem:deeponets-deephonets-etc}); for instance, we obtain UA of the forcing-to-solution operator of the 2D incompressible Navier--Stokes equations (Example \ref{ex:NS_operator}).
\paragraph{Recovery of finite-dimensional theory.}
Our UATs are exact infinite-dimensional counterparts of the
classical statements (A)--(B)  : the topology in \ref{eq:CkB_bastiani_compactopen_intro} reduces in
    finite dimensions to the classical $C^k$ compact-open topology and Bastiani--Sobolev spaces in \ref{eq:Sobolev_norm_intro}  recover the weighted Sobolev space, used in \cite{hornik1991approximation}.
In particular, to the best of our knowledge:
\begin{center}
    \textit{our Universal Approximation Theorems are the first complete generalizations of the corresponding\\ influential classical results in Hornik \cite{hornik1991approximation} to infinite-dimensional spaces and Operator Learning.}
\end{center}
\paragraph{Recovery of Gaussian Sobolev spaces on Banach spaces.}
In the Gaussian case, weighted Bastiani--Sobolev spaces recover the classical Gaussian Sobolev spaces \cite[Section 5.2]{bogachev1998gaussian} (Theorem \ref{th:characterization_bogachev}). As a
consequence, we obtain density of EDAs in $W^{k,p}_\gamma(X,Y)$ for every
$p\geq1$ under the assumptions stated in the paper.

\paragraph{Bastiani--Sobolev Training.} Supported by these UATs, we formulate  \textit{Bastiani--Sobolev Training} methods  (see Section~\ref{sec:bastiani-sobolev-training}),   extending previous DIOL training paradigms based on Hilbert spaces, Fréchet (or, in the Gaussian case, stochastic Gâteaux/Malliavin)  differentiability, $k=1$. In particular, the Bastiani--Sobolev viewpoint in \ref{eq:Sobolev_norm_intro} naturally supports alternative formulations, which are compatible with Sobolev training in high, but finite, dimensional spaces~\cite{czarnecki2017}. 

In view of all the above:
\begin{center}
    \textit{our results launch DIOL on general Banach spaces, raising new theoretical and computational challenges.}
\end{center}
\subsection{Relation to the \texorpdfstring{{State}}{State} of the Art}\label{subsec:state_art}
\paragraph{Operator Learning.}
{Over roughly the last seven years, OL has become a central approach to learning maps between infinite-dimensional function spaces. It learns maps from coefficient fields, forcing terms, initial or boundary data, and other functional inputs to partial differential equation (PDE) solutions or observables, providing fast surrogates for many-query scientific computing. Further applications include scientific discovery, data-driven and nonlinear reduced-order modeling, Koopman/operator inference, solver acceleration, closure and constitutive modeling, turbulence and continuum models, computer-graphics shape and surface maps, and weather forecasting.}  Standard architectures include the Encoder-Decoder Architectures \cite{kovachki-lanthaler-stuart}, DeepONet \cite{chen-chen,lu2020deeponet}, Fourier Neural Operator \cite{kovachki2023neuraloperator}, PCA-Net \cite{bhattacharya2021}, Deep-H-ONet/Hilbert--Galerkin Neural Operator \cite{castro2022} (see also \cite{cohen-defeo-hebner-sirignano}), Laplace Neural Operator \cite{cao2024}, Spectral Neural Operator \cite{fanaskov2024}, Convolutional Neural Operator \cite{raonic2023cno}, and Averaging Neural Operator \cite{lanthaler-li-stuart}.  Their use for these tasks is theoretically supported by a series of influential UATs for nonlinear operators on infinite-dimensional spaces \cite{chen-chen,kovachki2023neuraloperator,bhattacharya2021,hua-lu,lu2020deeponet,castro2022,castro2024calderon,huang2024,kovachki-lanthaler-mishra,jin-shuai-lu,lanthaler2022,zhang-wing-leung,lanthaler-li-stuart,shih-peyvan-zhang-etal,benth-detering-galimberti,galimberti2026,saini2026,ismailov2026,zappala2024,cuchiero-schmocker-teichmann}. We refer to  \cite{kovachki-lanthaler-stuart} for a detailed review of OL\footnote{We also refer to \cite{cuchiero-schmocker-teichmann,cuchiero_primavera_svaluto,ceylan_kwossek_promel,hager-harang-pellizari-tindel,lyons_nejad_perez} for signature learning and UATs there, and to \cite{lauriere_perrin_perolat_etal,pham_warin,mekkaoui-pham-warin} for OL on Wasserstein spaces.}.
\paragraph{Derivative-Informed Operator Learning.}
Even more recently,  DIOL has been rapidly growing \cite{cao_chen_brennan_olearyroseberry-marzouk-youssef-ghattas,cao-Roseberry-Ghattas,cohen-defeo-hebner-sirignano,go_chen,go-chen2025,luo-Roseberry-chen-Ghattas,gong-luo-roseberry-etal,olearyrosemberry-villa-chen-ghattas,Roseberry-peng-villa-ghattas,qiu_bridges_chen,yao-luo-cao-Kovachki-Roseberry-ghattas}. Due to the theoretical and computational challenges, numerical results are mostly restricted to the separable Hilbert setting and to the first derivative case\footnote{Among these the only paper that studies DeepONets is \cite{qiu_bridges_chen}, providing empirical studies. The only paper with theoretical results for arbitrary $k$ is \cite{luo-Roseberry-chen-Ghattas} on separable Hilbert spaces, while $k=2$ is covered by \cite{cohen-defeo-hebner-sirignano} on separable Hilbert spaces;  we will review \cite{luo-Roseberry-chen-Ghattas,cohen-defeo-hebner-sirignano} below.  
We also refer to \cite{alger_christierson_chen_ghattas2026} {for} Taylor series surrogate models for covariance preconditioned high dimensional mappings that depend implicitly on the solution of a system of nonlinear equations.}.
Here, the goal is to learn  an operator and its (infinite-dimensional) derivatives. This provides (we refer to Section \ref{subsec:motivations} for extensive examples):
\begin{enumerate}[(i)]
    \item higher-order accuracy than standard OL with  a significantly lower training sample size and generation cost \cite{Roseberry-peng-villa-ghattas,cao-Roseberry-Ghattas,qiu_bridges_chen,cao_chen_brennan_olearyroseberry-marzouk-youssef-ghattas,luo-oleary-chen-ghattas,go_chen}.
\item Applications of OL to constrained optimization in infinite-dimensional spaces  (such as optimal control of PDEs,
{inverse problems}, and optimal design) \cite{yao-luo-cao-Kovachki-Roseberry-ghattas,luo-oleary-chen-ghattas,go_chen,go-chen2025}, where derivatives encode fundamental information on minimizers, thereby allowing for  fast online optimization.
\item Development of numerical methods for infinite-dimensional PDEs \cite{cohen-defeo-hebner-sirignano}, such as functional differential equations in Physics and Kolmogorov and Hamilton--Jacobi--Bellman {(HJB)} PDEs from infinite-dimensional optimal control and games  (controlled PDEs, SPDEs, path-dependent systems, mean-field control, and partially observed stochastic systems), where the (infinite-dimensional) derivatives of the solution appear naturally in the infinite-dimensional PDE.
\end{enumerate}

However, the available literature on Derivative-Informed Operator learning is mainly empirical and UATs for the simultaneous approximation of nonlinear operators and their infinite-dimensional derivatives are not understood. Indeed, to the best of our knowledge, the only previous results (all appeared within about one year with respect to the first posting of the present manuscript on arXiv) are proven in the following very restrictive settings:
\begin{itemize}
    \item The only universal approximation result for $k$-times differentiable operators ($k\in \mathbb N$ arbitrary) available can be found in \cite{luo-Roseberry-chen-Ghattas}. However, this holds for operators $F$ belonging to Gaussian Sobolev spaces $W_{\gamma}^{k,2}(X,Y)$   (\cite[Section 5.2]{bogachev1998gaussian}, see also \eqref{eq:gaussian_sobolev}), where $X,Y$ are  separable \textbf{Hilbert spaces},  $p=2$, $\gamma \sim N(0,Q)$ is a \textbf{Gaussian} measure ($Q$ trace class on $H$), and  \textbf{derivatives} of $F$ are taken \textbf{along directions in the Cameron–Martin space}, i.e.  $X_{\gamma}:= Q^{1/2}(X)$, and $D_{\gamma}^k F$ denotes the $k$-th Gaussian Sobolev (or Malliavin) derivative restricted to $X_{\gamma}$. The result is therefore tied to the classical Gaussian Sobolev structure, which already encodes the finite-dimensional  approximation used for the UA:
$W^{k,2}_{\gamma}(X,Y)$ is defined as the completion of smooth cylindrical (finite-dimensional)
maps. 
 A detailed analysis of approximation errors is provided in $W_{\gamma}^{1,2}(X,Y)$ ($k=1$). Consequently, \textit{the case of  general finite input measures on Hilbert spaces, the case $p\neq 2$, and UATs on ${C^k_B}$ compact-open topologies, as well as the Banach case (including, e.g., DeepONets), are not addressed. 
 Our {UATs} address all these points.} {Furthermore, since our Bastiani--Sobolev spaces generalize Gaussian Sobolev spaces on Banach spaces,  our Corollary \ref{th:UA_gaussian} generalizes the  UAT in \cite{luo-Roseberry-chen-Ghattas}  to Banach spaces and to $p\geq 1$.}
    \item The paper  \cite{yao-luo-cao-Kovachki-Roseberry-ghattas} proved UATs in differentiable compact-open topologies and weighted Sobolev norms via FNOs up to \textbf{first order Fréchet} derivatives ($k=1$) on \textbf{specific Hilbert spaces} $X=H^s\left(\mathbb{T}^d \right), Y=H^{r}\left(\mathbb{T}^d \right)$, $s,r\geq 0$. There   $D F(x) \in  \mathcal L_2 \left(X_\delta, Y\right)$ is assumed to be  \textbf{Hilbert--Schmidt}, from some smoother space $X_\delta:=H^{s+\delta}\left(\mathbb{T}^d \right)$, $\delta \geq 0$ and approximation in the $\|\cdot\|_{\mathcal{L}_2\left(X_\delta, Y\right)}$-norm. If $F\in {C^1_F}(X,Y)$ then $DF(x)\notin \mathcal L_2 \left(X_\delta, Y\right)$, in general. For $\delta>d/2$, $X_\delta \hookrightarrow X$ with dense Hilbert--Schmidt embedding; therefore, $D F$ is continuous from $X$ to  $\mathcal L_2 \left(X_\delta, Y\right)$ and the result can be applied\footnote{The authors impose this requirement  since they explain that one should not expect UA of Fréchet derivatives in operator norm for FNOs.}. However, with this approach, UA is restricted to directions in the (smaller) space $X_\delta.$    {An even more structural limitation there is the choice of the Fréchet class ${C^1_F}$ itself: \textbf{FNOs are not Fréchet differentiable}, in general. Indeed, standard {FNO layers}
\[
\mathcal L_\ell(v)=\sigma\!\left(
W_\ell v+b_\ell+
\mathcal F^{-1}\!\big(P_\ell\odot\mathcal F(v)\big)
\right)
\]
are not ${C^1_F}$ on $L^2(\mathbb T^d;\mathbb R^{d_v})$, in general, even for smooth $\sigma$ \cite[Sections 3.1 and 3.3]{yao-luo-cao-Kovachki-Roseberry-ghattas}.} Their proof circumvents this issue by constructing special ${C^1_F}$ FNOs with an activation-free  Fourier-projection first layer, i.e. with trivial $W_1\equiv0$ (and finite-rank ${\mathcal P}_1=\mathbf 1_{\{|k|_\infty\le N\}}I$) \cite[Lemma C.7]{yao-luo-cao-Kovachki-Roseberry-ghattas}. The approximating FNO thus depends on $x$ only through $\mathcal P_Nx$: it sees only finitely many input Fourier modes and is therefore input-cylindrical, and, thus, $C^1_F$. However, the reliance on this procedure, highlights the structural limitation of the Fréchet ${C^1_F}$ framework for FNOs used there.  We will comment on these points in Remarks \ref{rem:comparison_UA_compact_yao}, \ref{rem:comparison_UA_sobolev_yao}, showing that \textit{our intrinsic approach based on Bastiani differentiability provides a more natural  framework even for FNOs}.
    \item The paper \cite{cohen-defeo-hebner-sirignano} proved UATs for maps in ${C^2_F}(X)$\footnote{Here, the ${C^2_F}$ Fréchet requirement is less restrictive than in \cite{yao-luo-cao-Kovachki-Roseberry-ghattas} as this is a standard requirement in the definition of classical solutions of second-order PDEs in Hilbert spaces \cite{fabbri2017book} addressed (via these UATs) in \cite{cohen-defeo-hebner-sirignano}.} ($k=2$, $Y=\mathbb R$), where $X$ is a separable \textbf{Hilbert space}.
By viewing the Hessian\footnote{It was similarly observed there that, in general, UA in operator norm for $D^2F$ cannot be expected; see also \cite{Lesmes1974}} as the joint map $(x,h)\mapsto D^2F(x)h$, the authors achieved UA in related compact-open topologies and weighted Sobolev norms.
The latter are constructed by integrating the input variable $x$ and derivative directions $h$ against finite measures, under suitable moment conditions on the measures and polynomial growth assumptions on $F,DF,D^2F$\footnote{UATs were also proved there for terms involving unbounded operators applied to the first derivative.}.
{Compared with \cite{cohen-defeo-hebner-sirignano}, our construction of weighted Bastiani--Sobolev spaces goes considerably further: it applies to Banach spaces, arbitrary derivative orders $k$, and Banach-valued maps.
We impose no polynomial growth assumptions on $F$ and its derivatives.
We use continuous Bastiani differentiability in place of Fréchet differentiability, and define our spaces for general finite input measures without moment conditions.}
\end{itemize}
{In particular,  when $X,Y$ are finite dimensional, none of the above UATs is able to recover the corresponding finite-dimensional UATs in \cite{hornik1991approximation}, making existing generalizations fundamentally incomplete.}

\paragraph{Infinite-dimensional Nachbin Theorem.} {Related results to the Universal Approximation Theorems of nonlinear operators and their derivatives can be found in \cite{llavona1986} (see also \cite{gil-llavona,nachbin76,prolla1974,prolla-guerreiro}), where infinite-dimensional extensions of Nachbin's theorem\footnote{we recall that Nachbin's theorem in finite-dimensions generalizes the Stone-Weierstrass theorem to derivative approximation, i.e. it states that polynomials can approximate $C^k$ functions and {their derivatives}, uniformly on compact sets.} were proved, i.e. classes of polynomial algebras were shown to be able to approximate $k$-times differentiable operators and their derivatives  in  compact-open topologies generated by seminorms of the form \ref{eq:CkB_bastiani_compactopen_intro}.} {Theorem \ref{th:UAT_K_Banach} generalizes the corresponding result in \cite{llavona1986}  (finite-type polynomial algebras being included in our framework) to OL architectures. A related density failure of polynomial algebras in ${C^k_F}(X,Y)^{co}$ was pointed out in \cite{llavona1986} (although general assumptions for such a failure to occur were not stated; see also \cite{Lesmes1974}). Our Theorem \ref{th:counterexample} provides {general sufficient conditions for such a failure at the level of first derivatives,} and covers general OL architectures (including finite-type polynomial algebras).}
\paragraph{\texorpdfstring{{Sobolev spaces in infinite dimensions.}}{Sobolev spaces in infinite dimensions.}} {To the best of our knowledge, no available notion of Sobolev space on a Banach space is suitable for UATs under general finite input measures. Sobolev spaces under Gaussian measures, as in \cite{luo-Roseberry-chen-Ghattas} (see, e.g., \cite{bogachev1998gaussian}), or under differentiable measures (see, e.g., \cite{bogachev-mayerwolf,bogachev2010,berezansky-kondratiev}) are tied to specific classes of measures. Sobolev spaces on metric measure spaces (e.g. \cite{bjorn-bjorn,heinonen-koskela-pekka-nageswaru-tyson,fornasier-savare-sodini2023,gigli-pasqualetto,savare2021}) are primarily first-order theories ($k=1$); furthermore, their first-order object is defined through the Cheeger energy and minimal relaxed slopes/weak upper gradients and, even for a smooth cylindrical map, it need not coincide with the classical Fréchet derivative of primary interest in OL. This motivates our weighted Bastiani--Sobolev spaces (Appendix \ref{subsec:Sobolev}).}
\paragraph{\texorpdfstring{{$L^p_\xi$-norms on multilinear operators.}}{Lp-norms on multilinear operators.}} {Our Bastiani--Sobolev spaces are built over the norms
 $$\|L\|_{L^p(X^i,Y,\xi)}^p:=\int_{X^i}|L(h^1,\ldots,h^i)|_Y^p\,d\xi,\quad L \in \mathcal L^i(X,Y),$$ on $\mathcal L^i(X,Y)$ (Appendix \ref{subsec:norm_Lp-multilinear}). These can be seen as restrictions of Bochner $L^p$-seminorms, for measures $\xi$ on $X^i$ with finite mixed $p$-moments. However, by continuity, no quotient is needed if $\operatorname{supp}\xi$ contains the product of unit spheres.
 To the best of our knowledge, this family of $L^p_\xi$-norms has not previously been introduced in this form on spaces of bounded multilinear operators between Banach spaces. They are weaker than operator norms and, for $p=2$, reduce to Hilbert--Schmidt or $\gamma$-radonifying norms \cite{van-neerven2010} on Reproducing Kernel Hilbert Spaces (Examples \ref{ex:HS_norms_RKHS}, \ref{ex:gaussian_measures_radonifying}).}
{\paragraph{Outline.} The manuscript is organized as follows. In Section \ref{subsec:motivations} we present motivations for Derivative-Informed Operator Learning. In Section \ref{sec:EDN}, we introduce Encoder-Decoder Architectures and present classical examples. In Section \ref{sec:UAT}, we prove failure of Universal Approximation in ${C^k_F}(X,Y)^{co}$, while we restore Universal Approximation in $C_B^k(X,Y)^{co}$ and in Bastiani--Sobolev spaces. In Section \ref{sec:bastiani-sobolev-training}, we formulate Bastiani--Sobolev training. In Appendix \ref{app:basics}, we collect some basic notations, used throughout the manuscript. In Appendix \ref{subsec:norm_Lp-multilinear}, we introduce $L^p_\xi$-norms on spaces of multilinear operators. In Appendix \ref{app:smooth}, we recall spaces of differentiable maps on Banach spaces and  present some results used in the manuscript. In Appendix \ref{subsec:Sobolev}, we introduce and study Bastiani--Sobolev spaces.}
\section{Motivations for Derivative-Informed Operator Learning}\label{subsec:motivations}
{This section surveys applications motivating DIOL and its UATs. The discussion is mainly expository, leaving rigorous development of these applications to future research.} However, in Example \ref{ex:NS_operator} we will apply Theorem \ref{th:UAT_K_Banach} to obtain UA via HGNOs (Subsection \ref{sec:Deep-HONet}) to a nonlinear ${C^\infty_F}$ operator mapping forcing functions to solutions of the Navier--Stokes equations in 2D. We will use the following terminology:
\begin{definition}[Query task]\label{def:query}
    A query task is a task where fast evaluation of a target operator $F:X\to Y$, and possibly, its derivatives $D^iF,$ $i\leq k$, at a given input $x\in X$ is required, while direct computation of $F(x),$ $ D^iF(x),$ $ i\leq k$ is expensive. Given a query task, our aim is to obtain a fast surrogate in the form of an OL architecture.
\end{definition}
\subsection{High-Order Accuracy in Operator Learning and Applications}
\subsubsection{General Framework}
\paragraph{Operator Learning.}  Let $X,Y$ be Banach spaces and let $F:X\to Y$ be an unknown   nonlinear operator. The goal of Operator Learning is to learn $F$. A standard choice in the literature is to train an operator learning architecture $F^\theta$ to approximate $F$; here  $\theta \in  \Theta$ are the parameters of the architecture. In particular:
 \begin{enumerate}[(A)]
        \item  given $F\in C^0(X,Y)$, $K\subset X$ compact,  $\epsilon>0,$ we want to find $\theta$ such that
    $\sup_{x\in K}|F(x)-F^\theta(x)|_Y<\epsilon;$ 
    \item given $F \in L^p(X,Y,\nu)$, $\epsilon>0,$ we want to find $\theta$ such that $ \int_X |F(x)-F^\theta(x)|_Y^p d\nu(x)<\epsilon.$
    \end{enumerate}
This expressivity is supported by universal approximation theorems, typically in compact-open topologies and in $L^p$-norms, for standard operator-learning architectures (recall Subsection \ref{subsec:state_art}). Extensive numerical tests confirm the success of these approximations, by training $F^\theta$ according to (A)-(B) (typically (B)) \cite{kovachki-lanthaler-stuart}.
\paragraph{Derivative-Informed Operator Learning.} Assume now to have access to further information, i.e. that $F:X\to Y$ is $k$-times differentiable (in some suitable sense). In this case, similarly to the classical finite-dimensional case \cite{hornik-stinchcombe-white,hornik1991approximation,pinkus1999approx}, it would be desirable to achieve a higher-order accuracy of approximation by 
\textit{simultaneously learning }$F$ \textit{and Its derivatives } $D^iF, i\leq k.$
This would be achieved by replacing (A)-(B) by
    \begin{enumerate}
        \item[{(A$_{p_i}$)}] given $F$,  $K\subset X$ compact,  $\epsilon>0,$ we want to find $\theta$ such that
         $\sum_{0\leq i \leq k} \sup_{x\in K}p_i(D^iF(x)-D^iF^\theta(x))<\epsilon;$
    \item[{(B$_{q_i}$)}]  given $F$ (possibly, with some integrability on $D^iF$, for all $0\leq i\leq k$), $\epsilon>0,$ we want to find $\theta$ such that $\sum_{0\leq i \leq k} \int_X [q_i(D^iF(x)-D^iF^\theta(x))]^p d\nu(x)<\epsilon.$
    \end{enumerate}
In the above $p_i,q_i:\mathcal L^i(X,Y)\to [0,\infty)$ are suitable (but non-specified, for the moment) norms/seminorms to measure errors in the derivatives  $D^iF(x) \in \mathcal L^i(X,Y)$. Indeed, as mentioned in the introduction spaces of multilinear operators $\mathcal L^i(X,Y)$ can be endowed with many non-equivalent topologies and the choice has {to} be done with particular care{: for instance, $p_i=q_i=\|\cdot\|_{\mathcal L^i(X,Y)}$ leads to \ref{eq:frechet_CO_intro}--\ref{eq:sobolev_op_intro}, where UA fails (Theorem \ref{th:counterexample}), while our choices lead to \ref{eq:CkB_bastiani_compactopen_intro}--\ref{eq:Sobolev_norm_intro}}.

Numerical experiments  have indeed shown that training $F^\theta$ according to some form of {(A$_{p_i}$)-(B$_{q_i}$)} (at least for $k=1$) significantly improves the accuracy of the approximation (see, e.g., the references in Subsection \ref{subsec:PDE_OL}) with  a significantly lower training sample size and generation cost. 

However, as seen in the literature review (Subsection \ref{subsec:state_art}) there are no existing UATs supporting this expressivity on Banach spaces, while only very special cases in the Hilbert space case are covered, leaving also here the problem open. Since many choices of norms/seminorms $p_i,q_i:\mathcal L^i(X,Y)\to [0,\infty)$  are possible,  this directly affects the numerical tests, including the training, making the problem of the development of UATs even more fundamental.
\subsubsection{Applications to partial differential equations}\label{subsec:PDE_OL}
    A standard task in OL is to learn the solution map $F$ of a {PDE}, as computing  $F(x)$ at a given input $x$ (e.g., through Galerkin methods) is expensive:  ``typical examples of such operators  arise from physical models expressed in terms of partial differential equations
(PDEs). In this context, such approximate operators hold great potential as efficient surrogate models to complement traditional numerical methods in many query tasks'' \cite{kovachki-lanthaler-stuart} (recall Definition \ref{def:query}). These surrogates are typically trained \textit{offline} via data obtained via standard numerical methods for PDEs.

However,  many interesting PDE solution operators turn out to be $k$-times differentiable, e.g. see Example \ref{ex:NS_operator}. Therefore, a higher-order accuracy is desired. Numerical tests  successfully confirm the learning of PDE operators and their first-order derivatives  with significantly lower training sample size and generation cost \cite{olearyrosemberry-villa-chen-ghattas,luo-oleary-chen-ghattas,luo-Roseberry-chen-Ghattas,qiu_bridges_chen,Roseberry-peng-villa-ghattas,westermann2026}.  

To get derivative data one typically solves {sensitivity equations}  \cite{hinze2008,manzoni_quarteroni_salsa,troltzsch2010} as we now recall. Let $X,Y,W$ be Banach spaces. Let $G: Y \times X \rightarrow W$ be a suitable PDE operator. Given an input  $u\in X$, consider the PDE 
\begin{equation}\label{eq:G=0}
    G(y, u)=0.
\end{equation}
Assume that the PDE \eqref{eq:G=0} defines  {an} input-to-state operator
$F:X\to Y, F(u)=y$, that is $G(F(u), u)=0$ for any $u \in X$. Analytically, $D^iF$ can be computed via the {so-called} \textbf{sensitivity equations}. {The sensitivity equation for $DF$ is obtained by differentiating \eqref{eq:G=0}, i.e. $D_yG(F(u),u)DF(u)+D_uG(F(u),u)=0$, so that for all $h\in X$ we obtain the following linear equation for $Z(u):=DF(u)h$
\begin{align}\label{eq:sensitivity_eq}
    D_yG(F(u),u)Z+D_uG(F(u),u)h=0.
\end{align}
Iterating the procedure, one obtains the sensitivity equations for higher derivatives.}

Examples where we have differentiable PDE solution operators are the following.
\begin{example}[Navier-Stokes Equations in 2D]\label{ex:NS_operator}
Following \cite[Section 1.8]{hinze2008} (and using standard notations defined there), let $\Omega\subset\mathbb R^2$ be a bounded Lipschitz domain, $I=(0,T)$ {with $0<T<\infty$}, and set
$
V:=\overline{\{\varphi\in C_0^\infty(\Omega)^2:\nabla\!\cdot\varphi=0\}}^{H_0^1},$ $
H:=\overline V^{L^2(\Omega)^2}
$, with $V\hookrightarrow  H= H^*\hookrightarrow V^*$,
$
Y:=\{y\in L^2(I;V):\partial_t   y\in L^2(I;V^*)\}
\hookrightarrow C([0,T];H)   .$
Let $X$ be a separable Hilbert space of forcing functions (e.g. $X=L^2(\Omega)$), let
$B\in\mathcal L(X,L^2(I;V^*))$.
{Fix $\nu>0$ and $y_0\in H$.}
 Consider the weak divergence-free formulation of the  instationary incompressible Navier--Stokes equation
\begin{equation}\label{eq:NS_eq}
    \partial_t y-\nu\Delta y+(y\cdot\nabla)y+\nabla p=Bu,\qquad
\nabla\!\cdot y=0,\quad \textit{on }I\times \Omega \qquad
y=0,\quad \textit{on }I\times\partial\Omega\qquad
y(0)=y_0 .
\end{equation}
For every $u\in X$, there exists a unique solution $y=y(u)\in Y$ of \eqref{eq:NS_eq}, defining a forcing-to-state operator
$F:X\to Y, F(u)=y.$
Remarkably, it turns out that
$F\in {C^\infty_F}(X,Y)$ {\cite[Lemma 1.19]{hinze2008}} with Lipschitz continuous Fréchet derivatives on bounded sets. In this setup, $X,Y$ are separable Hilbert spaces, therefore, we can apply Theorem \ref{th:UAT_K_Banach} and Remark \ref{rem:HGNO_UA} to obtain UA via HGNOs/Deep-H-ONets (recall Subsection \ref{sec:Deep-HONet}).
\end{example}
\begin{example}[Elliptic PDEs]\label{ex:elliptic_PDE} Consider the following PDE on {a bounded Lipschitz domain} $\Omega \subset \mathbb{R}^N${:}
$$
\begin{array}{rlr}
-\Delta y+d(x, y)=u ,\quad \text {in } \Omega \qquad 
\partial_\nu y=0,\quad \text {on }   \partial \Omega ,
\end{array}
$$
where $d: \Omega \times \mathbb{R} \rightarrow \mathbb{R}${.}
{Let $r\geq1$ with $r>N/2$, and assume that $d$ satisfies the conditions on $d$ in \cite[Assumption 4.14]{troltzsch2010}. Then,} for any forcing $u \in X:=L^r(\Omega)$, there exists a unique state $y$ in the Banach space $Y:=H^1(\Omega) \cap C(\bar{\Omega})$, defining the forcing-to-state operator $F: X \rightarrow Y, F(u)=y$, which is Fréchet differentiable {\cite[Theorem 4.17]{troltzsch2010}}. Moreover, $F$  is twice continuously Fréchet differentiable from $\tilde X:=L^\infty(\Omega)$ into $Y$ {\cite[Theorem 4.24]{troltzsch2010}}. 
\end{example}

\subsection{Learn-Then-Optimize for Fast Online Constrained Optimization in Banach spaces}\label{subsec:Learn-Then-Optimize}
Constrained Optimization Problems in Banach spaces (COPB) are classical  problems in applied mathematics, modeling many key problems, such as inverse problems, optimal control problems of PDEs.

Let $X,Y,W$ be Banach spaces. Let $J: Y \times X \rightarrow \mathbb{R}$ and $G: Y \times X \rightarrow W$. We consider the following constrained optimization problem\footnote{In many cases, the optimization is only performed on suitable admissible subsets of $X$, but for simplicity, here we will restrict to optimizing over the whole $X$.}:
    \begin{equation}\label{eq:constrained_opt}
\begin{gathered}
\min_{(y, u) \in Y \times X} J(y, u)  \quad \textit{ subject to } G(y, u)=0.
\end{gathered}
    \end{equation} 
Assume that the constraint $G(y, u)=0$ defines a map $F:X\to Y,u \mapsto F(u)$, that is $G(F(u), u)=0$ for any $u \in X$. Then, we can consider the reduced problem
\begin{equation}\label{eq:constrained_opt_reduced}
\min_{u\in X} \tilde{J}(u), \quad \textit{where } \tilde J:X \to \mathbb R, \quad   \tilde{J}(u):=J(F(u), u) .
    \end{equation} 
    {Assuming (twice) differentiability (in some suitable sense), one could use an iterative method for the reduced optimization problem \cite{hinze2008,manzoni_quarteroni_salsa}:  any minimizer $\bar x\in X$ satisfies $D\tilde J(\bar x)=0$. Therefore, to find $\bar x,$ given an initial guess $u^0$, one may construct a sequence 
$$ u^{n+1}=u^n+\tau_n d_n,\quad n \in \mathbb N,$$
where $d_n \in X$ is a descent direction and $\tau_n>0$ is  a suitable stepsize. For instance (but many other possibilities are available, also relying only on $D\tilde J$),  {a} Newton-type method in optimization uses, formally, the rule:
$$D^2\tilde J(u^n)(d_n,(\cdot))=-D\tilde J(u^n).$$
We compute $D\tilde J,D^2\tilde J$ by chain-rule as}
\begin{align}
&D \tilde{J}(u)h = D_y J(F(u),u)( D F(u)h)+D_u J(F(u),u)h\label{eq:COB_gradient}\\
&D^2 \tilde{J}(u)\left(h_1, h_2\right)=  D_{y y}^2 J(F(u), u)\left(D F(u) h_1, D F(u) h_2\right)  +D_{y u}^2 J(F(u), u)\left(D F(u) h_1, h_2\right)\label{eq:COB_hessian} \\
& \quad \quad \quad \quad\quad \quad\quad \quad+D_{u y}^2 J(F(u), u)\left(h_1, D F(u) h_2\right)  +D_{u u}^2 J(F(u), u)\left(h_1, h_2\right)\nonumber  \\
&\quad \quad \quad \quad\quad \quad\quad \quad+D_y J(F(u), u)\left(D^2 F(u)\left(h_1, h_2\right)\right) .\nonumber 
\end{align}
Therefore, \textit{\eqref{eq:COB_gradient}-\eqref{eq:COB_hessian} contain relevant information,  such as first- and second-order optimality conditions, with {much} of this encoded by $DF,D^2F$ (recall Subsection \ref{subsec:PDE_OL} for the derivation of the sensitivity equations for $DF,D^2F$)}. 
\paragraph{{Traditional Numerical Methods and Optimize-Then-Learn.}} {Computing $DF$ and $D^2F$ explicitly is typically very costly for standard numerical methods. Adjoint methods avoid forming these operators when computing reduced gradients and Hessian actions, but still require repeated state, adjoint, and linearized solves \cite{hinze2008,manzoni_quarteroni_salsa}. In \textbf{Discretize-Then-Optimize (DTO)}, one discretizes the optimization problem before deriving its optimality conditions; in \textbf{Optimize-Then-Discretize (OTD)}, one derives the continuous optimality conditions before discretizing them. The resulting discrete systems need not coincide \cite{manzoni_quarteroni_salsa}. Methods termed \textbf{Optimize-Then-Learn (OTL)} in \cite{cohen-defeo-hebner-sirignano} derive the continuous optimality conditions analytically, as in OTD, and learn their solution using ML or OL architectures. Applications to PDE-constrained optimization include \cite{alla-bertaglia-calzola,zhang-liu-alla-karniadakis}. Optimal control can also be treated through HJB equations, either for the discretized dynamics or directly in infinite dimensions.\footnote{See \cite[Section 5]{sirignano2018} for DGM applied to the finite-dimensional HJB equation obtained by spatially discretizing a linear-quadratic stochastic heat control problem.}}

\paragraph{Learn-Then-Optimize via Offline Derivative-Informed Operator Learning.} {For fast online optimization, we learn $F$, $DF$, and, when needed, $D^2F$ offline. Accurate derivatives are essential for reliable surrogate minimizers \cite{yao-luo-cao-Kovachki-Roseberry-ghattas}. Online, the surrogate and its derivatives are used in \eqref{eq:COB_gradient}--\eqref{eq:COB_hessian}, avoiding repeated state and adjoint solves and reducing optimization to a query task (Definition \ref{def:query}). We call this strategy \textbf{Learn-Then-Optimize (LTO)}; see \cite{luo-oleary-chen-ghattas} for reduced-basis methods and \cite{yao-luo-cao-Kovachki-Roseberry-ghattas} for FNO methods. The following examples illustrate applications to PDE control, inverse problems, and optimal design.}

\begin{example}[Optimal control of PDEs]\label{ex:contrained_opt_control}Typical constrained optimization problems on infinite-dimensional spaces are Optimal  control problems of PDEs \cite{manzoni_quarteroni_salsa,hinze2008,troltzsch2010}.
In the analytic framework of Example \ref{ex:NS_operator} \cite{hinze2008}, 
for $y_d\in L^2(I\times\Omega)^2$, $u_d\in X$, and $\gamma>0$, consider the optimal control problem, where controls are forcing functions $u\in X$:
$$
\min_{u\in X,\,y\in Y}
J(y,u),\quad \textit{where } J(y,u):=\frac12\|y-y_d\|_{L^2(I\times\Omega:\mathbb R^2)}^2
+\frac{\gamma}{2}|u-u_d|_X^2
$$
subject to \eqref{eq:NS_eq}. This is of the form \eqref{eq:constrained_opt}. Since $F:X\to Y, F(u)=y$ is well-defined we can consider the reduced problem \eqref{eq:constrained_opt_reduced}. 

{Other PDEs, including the semilinear equations in Example \ref{ex:elliptic_PDE}, similarly lead to constrained optimization on Banach spaces. For LTO methods using offline DIOL in PDE control, see \cite{luo-oleary-chen-ghattas}.}
\end{example}
\begin{example}[Inverse Problems]\label{ex:inverse_problem_reinterpretation}Typical constrained optimization problems on infinite-dimensional spaces  also arise from inverse problems \cite{troltzsch2010}. In the Navier--Stokes setting of Example \ref{ex:NS_operator}, let $F:X\to Y$ denote the forcing-to-state map, assigning to a forcing/control $u$ the corresponding velocity field $y=F(u)$ on $I\times\Omega$.
Given observations $y_d\in Z$ of the velocity evolution, possibly partial in space-time, and an observation operator $P\in\mathcal L(Y,Z)$, where $Z$ is the observation space, the inverse problem is
$$\min_{u\in X}\frac12\|PF(u)-y_d\|_Z^2+\frac\gamma2|u-u_d|_X^2.$$
Here $u_d\in X$ is a reference or prior forcing and $\gamma>0$ is a regularization parameter.

{For LTO methods and their theoretical justification in PDE inverse problems via offline DIOL, see \cite{yao-luo-cao-Kovachki-Roseberry-ghattas}; for PDE-constrained Bayesian inverse problems, see \cite{cao-Roseberry-Ghattas}.}
\end{example}
\begin{example}[Optimal Design]\label{ex:design}Constrained optimization problems on infinite-dimensional spaces arise naturally
in optimal design \cite{hinze2008}. A simple example is the
optimal design of a thin elastic membrane, where the design variable is the
spatially varying thickness $u$ and the state variable $y$ is the membrane
displacement. Given a load $f$, the equilibrium state is modeled by
the PDE
$-\operatorname{div}(u\nabla y)=f $ $ \text{in }\Omega,$
$y=0 $ $ \text{on }\partial\Omega .$ The  optimal design problem is to find an optimal thickness $u$ that minimizes the compliance $J(y):=\int_{\Omega} f(x) y(x) d x$ of the membrane:
$$
\begin{aligned}
&\min J(y), \quad \text { subject to }-\operatorname{div}(u \nabla y)  =f, \textit{ on } \Omega, \quad
y  =0, \textit { on } \partial \Omega, \  a  \leq u \leq b, \textit { on } \Omega,   
\int_{\Omega} u(x) d x  \leq V ,
\end{aligned}
$$
subject to the thickness and volume constraints above.

We refer to \cite{go_chen,go-chen2025}  for numerical results for related Bayesian optimal experimental design problems via DIOL.
\end{example}
\subsection{Numerical Methods for \texorpdfstring{{Infinite-Dimensional}}{Infinite-Dimensional} PDEs}\label{subsec:numerical_PDEs} Further motivations for DIOL and our UATs {come} from numerical methods for infinite-dimensional PDEs, such as functional differential equations in physics (e.g. the celebrated Hopf equation in turbulence theory and Schwinger–Dyson equation in quantum field theory) and HJB PDEs from infinite-dimensional optimal control (including problems of PDEs, SPDEs, path-dependent systems, partially observed stochastic systems, and mean-field SDEs)\footnote{For HJB PDEs on infinite-dimensional spaces, see, e.g., \cite{Crandall_lions1985,crandall_lions_infinite_dim_II,fuhrman-masiero-tessitore,masiero2008,mete1988hamilton,crandall1991viscosity,crandall_lionsVII,fabbri2017book,li2012optimal,lions1989viscosity,defeo_phd,defeo_federico_swiech,defeo_gozzi_swiech_wessels,bolli_defeo,CohenKnochenhauerMerkel2025,swikech2025finite}.  For the Hopf equation in turbulence theory see \cite{hopf1952statistical,monin2013statistical}{, and \cite{peskin2018introduction} for the Schwinger–Dyson equation in quantum field theory.}}. {For optimal control, deriving the HJB equation by dynamic programming and then learning an approximation of the value function follows the \textbf{Optimize-Then-Learn (OTL)} strategy of Section~\ref{subsec:Learn-Then-Optimize}.} Indeed, in  \cite{cohen-defeo-hebner-sirignano}, we developed Deep Hilbert--Galerkin Methods for infinite-dimensional PDEs  on  separable Hilbert spaces: we parameterize the solution $F$ of such a PDE, together with its derivatives $DF,D^2F$ via Deep-H-ONets/HGNOs $F^{{\mathcal E}_d^X,\theta,1},DF^{{\mathcal E}_d^X,\theta,1},D^2F^{{\mathcal E}_d^X,\theta,1}$ {(Definition \ref{def:deep_H_net} with $Y=\mathbb R$)} and minimize an $L^2(X,\mu)$-loss function over the Hilbert space $X${, for a suitable measure $\mu$ on $X$}. These generalize Deep Galerkin {Methods}/Physics-Informed Neural Networks {(PINNs)} \cite{sirignano2018,raissi2019physics} to infinite-dimensional PDEs. UATs proved there were shown to support this training procedure\footnote{We also refer to  \cite{castro2022,miyagawa2024physics,rodgers_venturi,venturi_dektor,venturi2018numerical} for other types of numerical methods for related classes of  PDEs on Hilbert spaces.}{.}

\section{Encoder-Decoder Architectures and Classical Examples}\label{sec:EDN}
{Encoder--decoder architectures apply to general Banach spaces \cite{kovachki-lanthaler-stuart}. We review their formulation and classical examples: DeepONets, Deep-H-ONets/HGNOs, Reduced-Basis Neural Operators, and PCA-Nets.}
\subsection{Encoder-Decoder Architectures}\label{subsec:EDA} 
{An EDA represents a nonlinear operator through an encoder, a trainable finite-dimensional map (typically a neural network), and a decoder \cite{kovachki-lanthaler-stuart}.} In particular, we assume for each domain and codomain dimension $N, N' \in \N$, there exists a class of finite-dimensional functions $\{{f^{N,\theta,N'}}  : \R^{N} \to \R^{N'}\}_{\theta \in \Theta_{N,N'}}$, where $\theta \in \Theta_{N,N'}$ represents the trainable parameters, that universally approximate in the following senses.

Recall Definition \ref{def:compact_open_C2}.
\begin{assumption}[Universal approximation on compacts in finite dimensions]\label{ass:finite_dim_approx}Let $k\in \mathbb N$.
    For all $N,N' \in \N$, the set of functions $\{ f^{N,\theta,N'}: \mathbb{R}^N \rightarrow \mathbb{R}^{N'}\}_{\theta \in \Theta_{N,N'}}$ is dense in $C^k(\R^N;\ \R^{N'})^{co}$. That is, for any $f\in C^k(\R^N, \R^{N'}),$ $K \subset \R^N$ compact, $\epsilon > 0$, there exists $\theta \in \Theta_{N,N'}$ such that $  {\mathbf p}_{K}(f- f^{N,\theta,N'}) < \epsilon,$ i.e.
    \begin{equation*}
     \sup _{x \in K}\|D^if(x)-D^i f^{N,\theta,N'}(x)\|_{\mathcal L^i(\mathbb R^N,\mathbb R^{N'})} < \epsilon,\quad \forall i\leq k.
    \end{equation*}
\end{assumption}

Recall \eqref{eq:growth_v_UATsobolev-op_norm} and \eqref{eq:sobolev_norm}.
\begin{assumption}[{Universal approximation in Sobolev norms in finite dimensions}]\label{ass:finite_dim_sobolev_approx}Let $k\in \mathbb N$, $p\geq 1$.
    For all $N,N' \in \N$, for all finite measures $\nu$ on $\R^N$, the set $\{ f^{N,\theta,N'}: \mathbb{R}^N \rightarrow \mathbb{R}^{N'}\}_{\theta \in \Theta_{N,N'}}$ is dense in $C^{k;p}_{\nu}(\mathbb{R}^{N},\mathbb{R}^{N'})$. That is, 
     for any $f \in  C^{k;p}_{\nu}(\mathbb{R}^{N},\mathbb{R}^{N'})$,  $\epsilon > 0$, there exists $\theta \in \Theta_{N,N'}$ such that 
     $$\|f-f^{N,\theta,N'}\|_{\mathcal  W^{k,p}_{\nu}(\mathbb{R}^{N},\mathbb{R}^{N'})}<\epsilon.$$
\end{assumption}

Our motivating class of functions with these approximation capabilities {is} neural networks.

\begin{definition}[Neural networks]\label{def:neural_structure}
    Let $\mathfrak{m} \in C^k(\R)$ be an `activation' function. Define the set of neural network parameters of an arbitrary number $\mathfrak{L} \in \N$ of hidden layers by ${\Theta}_{N, N'}:=\bigcup_{\mathfrak{L} \in \mathbb{N}} {\Theta}_{\mathfrak{L},N,N'}$, where 
    $$\begin{aligned} & {\Theta}_{\mathfrak{L}, N, N'}:=\left\{\theta=\left(A_0, \ldots, A_\mathfrak{L}\right): A_j: \mathbb{R}^{d_j} \rightarrow \mathbb{R}^{d_{j+1}} \text { is an affine function}, j=0, \ldots, \mathfrak{L}, d_0=N, d_{\mathfrak{L+1}}=N'  \right \}.
    \end{aligned}$$ For each $\theta \in {\Theta}_{N, N'}$, the corresponding deep neural network is the function $ f^{N,\theta,N'}: \mathbb{R}^N \rightarrow \mathbb{R}^{N'}$ given by $${f}^{N, \theta, N'}(y):=A_{\mathfrak{L}} \circ \mathfrak{m} \circ A_{\mathfrak{L}-1} \circ \cdots \circ \mathfrak{m} \circ A_0(y),$$ where $\mathfrak{m}$ is understood to apply componentwise.
\end{definition}
Fully connected neural networks typically obey Assumption \ref{ass:finite_dim_approx} (resp. \ref{ass:finite_dim_sobolev_approx}). In particular, the subclasses of
\begin{itemize}
    \item wide neural networks with a single hidden layer ($\mathfrak{L} = 1$, $d_1$ arbitrarily large) with $\mathfrak{m} \in C^k(\R)$ nonconstant and bounded {(resp. bounded together with its derivatives)}; this is the classic result \cite{hornik1991approximation} Theorem 3 (resp. Theorem 4, equivalently stated for Fréchet derivatives (and extended to the vector-valued case){)}.
    \item deep neural networks of finite width ($\mathfrak{L}$ arbitrarily large, $d_1, \dots, d_\mathfrak{L} = N+N'+{2}$) with $\mathfrak{m} \in C^{k+1}(\R)$ non-affine \cite{kidger2020universal, JMLR:v24:22-1191} may each individually be rich enough to satisfy the two assumptions.
    \item Other architectures, such as transformers \cite{Yun2020Are}, convolutional neural networks \cite{HWANG2026101833}, and kernel methods \cite{JMLR:v7:micchelli06a}, {may} also be shown to have results of this type, under appropriate assumptions, along with the classical examples of approximation in Fourier or polynomial bases. 
\end{itemize}

Let $(X,|\cdot|_X),(Y,|\cdot|_Y)$ be Banach spaces. We define Encoder-Decoder Architectures  \cite{kovachki-lanthaler-stuart}. 
\begin{definition}[Encoder-Decoder Architectures]\label{def:EDN}Let Assumption \ref{ass:finite_dim_approx} or Assumption  \ref{ass:finite_dim_sobolev_approx} hold.
    We  define the Encoder-Decoder Architecture (EDA)  by   $$F^{{\mathcal E}^X,\theta,{\mathcal D}^Y}\colon X \to Y, \quad F^{{\mathcal E}^X,\theta,{\mathcal D}^Y}(x):= {\mathcal D}^Yf^{N,\theta,N'} ( {\mathcal E}^Xx)=\sum_{j=1}^{N'} f^{N,\theta,N'}_j\left[ (\langle a^{X}_i,x\rangle_{X^*,X}  )_{i=1}^{N}\right] e^{Y}_j,\quad  \theta\in {\Theta_{N,N'}}$$
for some encoder operator on $X$, ${\mathcal E}^X \in \mathcal L( X , \R^{N}),{\mathcal E}^X(x) =  (\langle a^{X}_i,x\rangle_{X^*,X}  )_{i=1}^{N}$, for some $N\in \mathbb N$; decoder operator into $Y$, ${\mathcal D^{Y}} \in \mathcal L( \R^{N'} , Y),  {\mathcal D}^{Y}((x_j)_{j=1}^{N'}) = \sum_{j=1}^{N'} x_j e^{Y}_j$, for some $N'\in \mathbb N$ (recall Definition \ref{def:encoder_decoder}), and
$f^{N,\theta,N'}_j , j\leq N'$ denote the components of $f^{N,\theta,N'}$. We denote by $\mathcal {A}_{ED}(X,Y)=\{F^{{\mathcal E}^X,\theta,{\mathcal D}^Y}:{\mathcal E}^X\in \mathcal L(X,\mathbb R^N),{\mathcal D}^Y\in \mathcal L(\mathbb R^{N'},Y),\theta\in {\Theta}_{N, N'}\}$. Note that $\mathcal {A}_{ED}(X,Y)\subset Cyl(X,Y)$ (recall Definition \ref{def:cyl_map}).
\end{definition}

Let $f^{N,\theta,N'}\in C^k(\mathbb R^{N}, \mathbb R^{N'})$. Then, by chain rule, for every $1 \leq i \leq k$, 
\begin{align}\label{eq:derivativeDrv_dh^theta_dY_Banach}
    D^i F^{{\mathcal E}^X,\theta,{\mathcal D}^Y}(x)\left[h^1, \ldots, h^i\right]&={\mathcal D}^Y\left(D^i f^{N, \theta,N'}\left({\mathcal E}^X x\right)\left[{\mathcal E}^X  h^1, \ldots, {\mathcal E}^X  h^i\right]\right)
\end{align}

\subsection{DeepONet}\label{sec:DeepONets}
As a first example of EDA, we consider classical DeepONets \cite{lu2020deeponet} (see also \cite{chen-chen}). 
Let  $(\Omega_1,d_1)$ be a compact metric space and $\Omega_2\subset \mathbb R^a$, $a\in \mathbb N$ be compact. 
\begin{definition}[DeepONet]
  Let $d,m\in \mathbb N$ and let  sensors $\left\{y_1^d, \ldots, y_{N_d}^d\right\}\subset \Omega_1$, $N_d,N_m\in \mathbb N$. A  DeepONet is a map
$$
\begin{aligned}
{C}\left(\Omega_1\right) & \longrightarrow C(\Omega_2), \quad u  \longmapsto [\varphi^{d,\theta,m}\left(u\left(y^d_1\right), \ldots, u\left(y^d_{N_d}\right)\right) \cdot\psi^{m}](z)= \sum_{j=1}^{N_m} \varphi_j^{d,\theta,m}\left(u\left(y^d_1\right), \ldots, u\left(y^d_{N_d}\right)\right) \psi^m_j(z),\quad \forall z\in \Omega_2
\end{aligned}
$$
where $N_d,N_m\in \mathbb N,$ $\varphi :\mathbb{R}^{N_d}\to \mathbb{R}^{N_m}$ and $\psi :\Omega_2\to \mathbb{R}^{N_m}$, and 
$\varphi^{d,\theta,m}_j $, $\psi^m_j $, $j\leq N_m$ denote their components. Typically $\varphi_j,\psi_j$ are suitable neural networks.

Defining the separable Banach spaces $X:=C(\Omega_1),$ $ Y:=C(\Omega_2)$ (endowed with the supremum norm), we have that DeepONet is an Encoder-Decoder Architecture, according to Definition \ref{def:EDN},  as shown in  \cite{godeke_fernsel} as follows: 
\begin{itemize}[leftmargin=*]  
    \item ${\mathcal S}^X_d,{E}_d^{X},{D}_d^{X}$ have the form, respectively,    \eqref{eq:sum_op_X}, \eqref{eq:coordinate_op_X}, \eqref{eq:embedding_op_X}  with   $a_i^{X,d}(u)=u(y_i^d),$ $ e_i^{X,d}=P_{\frac 1 {d},i}$. Here,  $\{y_i^d\}_{i=1}^{N_d}\subset \Omega_1$ is a $\frac 1 d$-covering of $\Omega_1$ (i.e. 
$
\Omega_1 \subseteq \bigcup_{i=1}^{N_d} B^{\Omega_1}_{\frac 1 d}(y_i^d),
$ where $B^{\Omega_1}_\epsilon(y)$ denotes the ball of radius $\epsilon>0$ centered in $y\in \Omega_1$;
note that since $\Omega_1$ is compact, there exists a finite $\varepsilon$-covering for every $\varepsilon>0$) and for $\varepsilon>0$,
$$
\begin{aligned}
P_{\varepsilon, i}(y):=\frac{\tilde{P}_{\varepsilon, i}(y)}{\sum_{\ell=1}^{N_d} \tilde{P}_{\varepsilon, \ell}(y)} ,\quad \textit{with }&\tilde{P}_{\varepsilon, i}: \Omega_1  \longrightarrow[0, \infty),\quad  
y  \longmapsto \begin{cases}\exp \left(-\frac{1}{\varepsilon^2-d_1(y, y_i^d)^2}\right) & \text { if } d_1\left(y, y_i^d\right)<\varepsilon \\
0 & \text { else }\end{cases}.
\end{aligned}
$$
By \cite[Theorem 3.23]{godeke_fernsel}, this choice guarantees that ${\mathcal S}_d^X$ satisfy the BAP for $X$ (Assumption \ref{ass:bounded_approx_pro}).
\item It is known that $ Y=C(\Omega_2)$ is a separable Banach space {with} BAP. Therefore there exists $\{\tilde {\mathcal S}_m^Y\}_{m\in \mathbb N}\subset \mathcal L(Y)$ satisfying Assumption \ref{ass:bounded_approx_pro}.  As $\tilde {\mathcal S}_m^Y$ maps into a finite dimensional space, we can choose a basis for $\tilde {\mathcal S}^Y_m(Y)$, denoted $\left\{b_j^{m}\right\}_{j =1 }^{N_m}$, for some $N_m\in \mathbb N$. Therefore, we can write
$
\tilde  {\mathcal S}_m^Y x=\sum_{j =1}^{N_m} c_j^{m}\left(\tilde {\mathcal S}^Y_m x\right) b_j^{m}
$
where $c_j^m\in \mathcal L(\tilde {\mathcal S}^Y_m(Y);\mathbb R)$. Let $\Psi$ be a dense subset in $Y$ (e.g. neural networks).
Then, by \cite[Theorem 3.26]{godeke_fernsel} there exists $\{\psi_j^m\}_{j=1}^{N_m}\subset \Psi$ such that, denoting 
${\mathcal S}_m^Y x=\sum_{j =1}^{N_m} c_j^{m}\left(\tilde {\mathcal S}^Y_m x\right) \psi_j^{m}$
we have that $\{{\mathcal S}^Y_m\}_{m\in \mathbb N}$ satisfies the BAP for $Y$ (Assumption \ref{ass:bounded_approx_pro}). Then ${\mathcal S}^Y_m,{E}_m^{Y},{D}_m^{Y}$ have the form, respectively,    \eqref{eq:sum_op_X}, \eqref{eq:coordinate_op_X}, \eqref{eq:embedding_op_X}  with   $\langle a_j^{Y,m},x\rangle_{Y^*,Y}=c_j^m(\tilde {\mathcal S}^Y_mx)$,  $ e_j^{Y,m}=\psi_j^m$. 
\end{itemize}
We denote by $\mathcal {DON}(X,Y)$ the set of DeepONets. Notice that $\mathcal {DON}(X,Y)\subset \mathcal {A}_{ED}(X,Y)$.
\end{definition}
In the above setup{, Theorem \ref{th:counterexample} applies  to $\mathcal {DON}(X,Y)\subset \mathcal {A}_{ED}(X,Y)$ {whenever $\mathcal A(X,Y)\subsetneq\mathcal L(X,Y)$}, through Theorem \ref{th:counterexample}(a), and} we  adapt the statements of {Theorem} \ref{th:UAT_K_Banach} (ii)  and Corollary \ref{th:UAT_Sobolev_K_Banach} to DeepONets by replacing  $\mathcal {A}_{ED}(X,Y)$ with $\mathcal {DON}(X,Y)$. The statement of Theorem \ref{th:UAT_Sobolev_Banach} can also be  adapted.

We remark that generalizations to the vector valued case (i.e. ${C}\left(\Omega_i;\mathbb R^{p_i}\right)$) are possible as well as taking $Y$ to be an arbitrary Banach space with BAP as in \cite[Section 5.1]{godeke_fernsel}.

\subsection{Deep-H-ONet, HGNO, Reduced-Basis Neural Operator, PCA-Net}\label{sec:Deep-HONet}{Deep-H-ONets/HGNOs, Reduced-Basis Neural Operators, and PCA-Nets use orthonormal bases of separable Hilbert spaces and are EDAs \cite[Subsection 4.1]{kovachki-lanthaler-stuart}. Our theory therefore applies to these architectures.}

Let $(X,<\cdot,\cdot>_X),(Y,<\cdot,\cdot>_Y)$ be separable Hilbert spaces with orthonormal bases $\{e_i^X\}$, $\{e_i^Y\}$, respectively (recall Subsection \ref{subsec:H}). Therefore, the Bounded Approximation Property of $X,Y$ (Assumption \ref{ass:bounded_approx_pro}) is satisfied.
\begin{definition}[HGNO/Deep-H-ONet]\label{def:deep_H_net}
     Let Assumption \ref{ass:finite_dim_approx} or Assumption  \ref{ass:finite_dim_sobolev_approx} hold. 
    We  define the Deep-$H$-ONet/Hilbert-Galerkin Neural Operator (HGNO) by   $$F^{{\mathcal E}_d^X,\theta,{\mathcal D}^Y_m}\colon X \to Y, \quad F^{{\mathcal E}_d^X,\theta,{\mathcal D}^Y_m}(x):= {\mathcal D}^Y_{m}f^{d,\theta,m} ( {\mathcal E}^X_{d}x)=\sum_{j=1}^{m} f^{d,\theta,m}_j\left[ (\langle x,e^{X}_i\rangle_{X}  )_{i=1}^{d}\right] e^{Y}_j,\quad d,m \in \mathbb N, \theta\in {\Theta_{d,m}}$$
where  $f^{d,\theta,m}_j , j\leq m$ denote the components of $f^{d,\theta,m}$. We denote by $\mathcal {HGNO}(X,Y)=\{F^{{\mathcal E}_d^X,\theta,{\mathcal D}^Y_m}:d,m\in \mathbb N,\theta\in {\Theta}_{d, m}\}$. Notice that $\mathcal {HGNO}(X,Y)\subset \mathcal {A}_{ED}(X,Y)$.
\end{definition}
\begin{example}[PCA Net]
  The PCA-Net is a particular architecture of the form of Definition 3.6, where the input encoders ${\mathcal E}_d^X$ are defined by projection onto PCA modes computed from the input measure $\mu$, while the output decoders ${\mathcal D}_m^Y$ are defined by PCA modes computed from the corresponding output distribution, e.g. the pushforward of $\mu$ under the target operator, or its empirical approximation. We refer to  \cite{bhattacharya2021} for the details of the architecture; see also \cite{kovachki-lanthaler-stuart}. 
\end{example}
We also refer to \cite{luo-Roseberry-chen-Ghattas} for the related  architecture of Reduced-Basis Neural Operators. 

\begin{remark}\label{rem:HGNO_UA}
In the above setup{, Theorem \ref{th:counterexample} applies  to $\mathcal {HGNO}(X,Y)\subset \mathcal {A}_{ED}(X,Y)$ whenever $\mathcal A(X,Y)\subsetneq\mathcal L(X,Y)$, through Theorem \ref{th:counterexample}(a), and} we adapt the statements of {Theorem} \ref{th:UAT_K_Banach} (ii)  and Corollaries \ref{th:UAT_Sobolev_K_Banach}, \ref{th:UAT_Sobolev_K_Banach_bump}  to HGNOs/Deep-H-ONets by replacing  $\mathcal {A}_{ED}(X,Y)$ with $\mathcal {HGNO}(X,Y)$ (recall that separable Hilbert spaces satisfy Assumption \ref{ass:bump} by Proposition \ref{prop:bump_rescaled} and Example \ref{ex:norm_Fréchet}). The statement of Theorem \ref{th:UAT_Sobolev_Banach} can also be adapted. {Therefore our statements also {hold} for Reduced-Basis Neural Operators and PCA-Net.}
\end{remark}
\subsection{Schauder Operator Nets}\label{subsec:Svahuder_operator_nets}
{Schauder Operator Nets are EDAs for Banach spaces with Schauder frames, including spaces with Schauder bases or Hilbert frames \cite{godeke_fernsel}, and generalize Definition \ref{def:deep_H_net}.} 

Let $X,Y$ be Banach spaces with Schauder frames $\left\{\left(e_i^X, a_i^X\right)\right\}_{i \in \mathbb{N}}\subset X \times X^*$, $\left\{\left(e_i^Y, a_i^Y\right)\right\}_{i \in \mathbb{N}}\subset Y \times Y^*$ (recall Definition \ref{def:schauder_frame}). We recall that this framework includes Banach spaces with Schauder bases (Definition \ref{def:schauder_basis}).
\begin{definition}[Schauder Operator Nets]\label{def:sch_op-Net}
     Let Assumption \ref{ass:finite_dim_approx} or Assumption  \ref{ass:finite_dim_sobolev_approx} hold.
    We  define the Schauder Operator Net (SON) by   $$F^{{\mathcal E}_d^X,\theta,{\mathcal D}^Y_m}\colon X \to Y, \quad F^{{\mathcal E}_d^X,\theta,{\mathcal D}^Y_m}(x):= {\mathcal D}^Y_{m}f^{d,\theta,m} ( {\mathcal E}^X_{d}x)=\sum_{j=1}^{m} f^{d,\theta,m}_j\left[ (\langle a^{X}_i,x\rangle_{X^*,X}  )_{i=1}^{d}\right] e^{Y}_j,\quad d,m \in \mathbb N, \theta\in {\Theta_{d,m}}$$
where  $f^{d,\theta,m}_j , j\leq m$ denote the components of $f^{d,\theta,m}$. We denote by $\mathcal {SON}(X,Y)=\{F^{{\mathcal E}_d^X,\theta,{\mathcal D}^Y_m}:d,m\in \mathbb N,\theta\in {\Theta}_{d, m}\}$. Notice that $\mathcal {SON}(X,Y)\subset \mathcal {A}_{ED}(X,Y)$.
\end{definition}
As usual, {Theorem \ref{th:counterexample} applies  to $\mathcal {SON}(X,Y)\subset \mathcal {A}_{ED}(X,Y)$ whenever $\mathcal A(X,Y)\subsetneq\mathcal L(X,Y)$, through Theorem \ref{th:counterexample}(a), and} we adapt the statements of {Theorem} \ref{th:UAT_K_Banach} (ii)  and Corollary \ref{th:UAT_Sobolev_K_Banach} to Schauder Operator Nets by replacing  $\mathcal {A}_{ED}(X,Y)$ with $\mathcal {SON}(X,Y)$. The statement of Theorem \ref{th:UAT_Sobolev_Banach} can also be  adapted. 
\section{Universal Approximation Theorems}\label{sec:UAT}
{We prove failure of UA under operator norms (Subsection \ref{subsec:counterex}), and establish UATs for nonlinear operators and their derivatives in $C^k_B$ compact-open topologies (Subsection \ref{subsec:UAT_compacts}) and weighted Bastiani--Sobolev spaces (Subsection \ref{subsec:sobolevUAT}). Proofs are given in Appendix \ref{sec:proofs}.} 

We will use the notations in Appendices \ref{app:basics}, \ref{app:smooth}, \ref{subsec:Sobolev} (in particular, recall Sections \ref{subsubsec:borel_meas}, \ref{subsec:compact-open}, and Notation \ref{not:notation_D0f}). 
\subsection{Universal Approximation Fails under Operator Norms}\label{subsec:counterex}
We show that Universal Approximation fails under operator norms, for two independent reasons: derivatives of EDAs have finite rank, and derivatives of any continuously parametrized architecture form a separable set. We postpone the proof to Appendix \ref{sec:proofs}.
Let $(X,|\cdot|_X),(Y,|\cdot|_Y)$ be  Banach spaces.
{In practice, an architecture is specified by countably many discrete choices (type, widths, depths, encoder and decoder dimensions, numbers of modes, etc.) and, for each of them, by parameters in a separable metric space: typically a Euclidean one for the weights, possibly spaces of operators for encoders and decoders (see, e.g., \cite{petersen2021topological} for neural networks in finite dimensions).}
\begin{definition}[{Continuously parametrized class}]\label{def:CP_class}
{Let $k\in\mathbb N$, $k\geq1$. A set $\mathcal N\subseteq C^k_F(X,Y)$ is a \emph{continuously parametrized class} if $\mathcal N=\bigcup_{c\in\mathbb N}\{F^{c,\zeta}:\zeta\in\mathcal Z_c\}$, where each $\mathcal Z_c$ is a separable metric space and, for each $c\in\mathbb N$, $\|DF^{c,\zeta_n}(x_n)-DF^{c,\zeta}(x)\|_{\mathcal L(X,Y)}\to0$ whenever $\zeta_n\to\zeta$ in $\mathcal Z_c$ and $x_n\to x$ in $X$.}
\end{definition}
\begin{remark}\label{rem:CP_examples}
{Countable unions of continuously parametrized classes are again continuously parametrized classes. So a single such $\mathcal N$ can contain any countable collection of operator-learning architectures satisfying the above continuity assumption. The index $c$ then encodes the architecture type and its configuration: widths, depths, latent dimensions, numbers of Fourier modes, and so on. For instance, $\mathcal N$ can contain:}
\begin{enumerate}[(a)]
    \item {EDAs $F^{{\mathcal E}^X,\theta,{\mathcal D}^Y}$ of Definition \ref{def:EDN} with neural networks (see Definition \ref{def:neural_structure}) with activation $\mathfrak m\in C^k(\mathbb R)$, e.g., if $X^*$ and $Y$ are separable (with $c$ enumerating $(N,N',\mathfrak L,d_1,\dots,d_{\mathfrak L})$ and $\mathcal Z_c:=\mathcal L(X,\mathbb R^N)\times\Theta_c\times\mathcal L(\mathbb R^{N'},Y)$, where $\Theta_c$ is the Euclidean space of parameters in $\Theta_{\mathfrak L,N,N'}$ with widths $d_1,\dots,d_{\mathfrak L}$), in which case $\mathcal A_{ED}(X,Y)$ itself is a continuously parametrized class. This follows from \eqref{eq:derivativeDrv_dh^theta_dY_Banach}, since $(\theta,y)\mapsto Df^{N,\theta,N'}(y)$ is jointly continuous on $\Theta_c\times\mathbb R^N$;\footnote{{Indeed, $(\theta,y)\mapsto f^{N,\theta,N'}(y)$ is of class $C^k$, $k\geq1$, on $\Theta_c\times\mathbb R^N$, as a composition of $\mathfrak m$, applied componentwise, and of the maps $(A_j,z)\mapsto A_j(z)$, which are polynomial in $z$ and in the coefficients of $A_j$.}}}
    \item {FNOs \cite{kovachki2023neuraloperator,yao-luo-cao-Kovachki-Roseberry-ghattas} with smooth activation $\sigma$ on $X=H^s(\mathbb T^n;\mathbb R^{d_a})$, $Y=H^r(\mathbb T^n;\mathbb R^{d_b})$, $s>n/2$, $0\leq r\leq s$, whose layers $\mathcal L_\ell(x)=\sigma(W_\ell x+b_\ell+\mathfrak K_\ell x)$ combine pointwise maps with Fourier multipliers $\mathfrak K_\ell$ retaining finitely many modes (with $c$ enumerating the architecture hyperparameters, i.e.\ depth, channel width and number of retained modes, and $\mathcal Z_c$ the Euclidean space of all weights, biases and retained Fourier coefficients). Superposition operators with smooth nonlinearities are of class $C^\infty_F$ on $H^s$ (see, e.g., \cite[Chapter 4, Section 1, Exercise 7]{taylor2010}), and so is the realization $(\zeta,x)\mapsto F^{c,\zeta}(x)$, jointly in the parameters and the input, which yields the required continuity\footnote{On $L^2(\mathbb T^n)$, by contrast, FNO layers are in general not Fr\'echet differentiable \cite{appell-zabrejko} (cf.\ Remark \ref{rem:comparison_UA_compact_yao}).}. {When $r=s$, derivatives of FNOs are in general not compact (the local terms $W_\ell$ produce multiplication-type operators), hence need not belong to $\mathcal A(X,Y)$; thus FNOs are covered by Theorem \ref{th:counterexample}(b), but not in general by Theorem \ref{th:counterexample}(a). When $r<s$, each derivative factors through the compact embedding $H^s(\mathbb T^n;\mathbb R^{d_b})\hookrightarrow H^r(\mathbb T^n;\mathbb R^{d_b})$, so it is compact and approximable; in this case, both Theorem \ref{th:counterexample}(a) and (b) apply.}}
\end{enumerate}
\end{remark}
\begin{theorem}[{Universal approximation fails under operator norms}]\label{th:counterexample}
{Let $k\in\mathbb N$, $k\geq1$, and let $\mathcal N\subseteq C^k_F(X,Y)$ satisfy one of the following conditions:}
\begin{enumerate}
    \item[(a)] $\mathcal A(X,Y)\subsetneq\mathcal L(X,Y)$ and $DG(x)\in\mathcal A(X,Y)$ for all $G\in\mathcal N$ and $x\in X$; e.g.\ $\mathcal N=C^k_{Cyl}(X,Y)$, which contains $\mathcal A_{ED}(X,Y)$ (indeed, $DG(x)\in\mathcal F(X,Y)$ for $G\in C^k_{Cyl}(X,Y)$ by \eqref{eq:derivativeDrv_dh^theta_dY_Banach_cyl});
    \item[(b)] $\mathcal N$ is a continuously parametrized class and $\mathcal L(X,Y)$ is nonseparable.
\end{enumerate}
{Then:}
\begin{enumerate}
    \item[(i)] {$\mathcal N$ is not dense in $C^k_F(X,Y)^{co}$;}
    \item[(ii)] {for any $p\geq1$ and any finite measure $\nu\neq0$ on $X$ such that $\int_X|x|_X^p\,d\nu(x)<\infty$, $\mathcal N\cap C^{k;p}_{F,\nu}(X,Y)$ is not dense in $C^{k;p}_{F,\nu}(X,Y)$, provided that $X$ is separable in case (b).}
\end{enumerate}
{In both cases, the failure is witnessed by bounded linear maps $F(x)=Lx$.}
\end{theorem}
\begin{remark}\label{rem:counterexample}
\begin{enumerate}[(a)]
    \item {Theorem \ref{th:counterexample} does not contradict classical finite-dimensional results \cite{hornik1991approximation}: if $X$ and $Y$ {are} finite-dimensional, then $\mathcal A(X,Y)=\mathcal K(X,Y)=\mathcal L(X,Y)$ and $\mathcal L(X,Y)$ is separable, violating the assumptions in Theorem \ref{th:counterexample}.}
    \item {UA fails already for the most basic operators, i.e.\ bounded linear maps, and, in (ii), for finite measures with finite $p$-th moment. \textit{This shows that the failure is structural.}}
\end{enumerate}
\end{remark}
\begin{remark}\label{rem:counterexampleII} Theorem \ref{th:counterexample}(a) and (b) complement each other:
\begin{enumerate}[(a)]
       \item {The condition $\mathcal A(X,Y)\subsetneq\mathcal L(X,Y)$ in Theorem \ref{th:counterexample}(a) coincides with $\mathcal K(X,Y)\subsetneq\mathcal L(X,Y)$ if $X^*$ or $Y$ has AP (Assumption \ref{ass:approx_pro}), by Remark \ref{rem:AP_implies_mathcalK=mathcalA}. It is satisfied by many infinite-dimensional Banach spaces of interest: e.g., it holds whenever $X=Y$ is infinite-dimensional, since the identity operator $I:X\to X$ is bounded but not compact (see \cite{kalton1974spaces,tong1971existence} for other characterizations). It fails, e.g., for $X=\ell^p$, $Y=\ell^q$, $p>q\geq1$: then $\mathcal L(X,Y)=\mathcal K(X,Y)$ by Pitt's theorem \cite[Theorem 2.1.4]{albiac2006topics}, \cite{pitt1936note}, and $\mathcal K(X,Y)=\mathcal A(X,Y)$, as these spaces have Schauder bases. Another example for which $\mathcal A(X,Y)=\mathcal F(X,Y)=\mathcal L(X,Y)$ is  finite-dimensional outputs $Y=\mathbb R^N$. In this last case, we can apply (b) as we  see below.}
    \item {The nonseparability of $\mathcal L(X,Y)$ holds, e.g., if $X$ and $Y$ are infinite-dimensional Hilbert spaces, as for the FNOs in Remark \ref{rem:CP_examples}, or if $X^*$ is nonseparable and $Y\neq\{0\}$; the latter includes finite-dimensional outputs $Y=\mathbb R$ with, e.g., $X=\ell^1$ or $X=C([0,1])$, for which $\mathcal A(X,Y)=\mathcal L(X,Y)$, so that only Theorem \ref{th:counterexample}(b) applies. However, an example when it fails is when $X=Y$ is the Argyros--Haydon space \cite{argyros-haydon2011}: $X^*\cong\ell^1$, $X$ has a Schauder basis, and every operator on $X$ is of the form $\lambda I+K$ with $\lambda\in\mathbb R$ and $K$ compact; hence $\mathcal L(X)=\mathbb R I\oplus\mathcal K(X)$ is separable, so that Theorem \ref{th:counterexample}(b) never applies, while $I\notin\mathcal A(X)=\mathcal K(X)$, so that Theorem \ref{th:counterexample}(a) does. }
\end{enumerate}
\end{remark}

\subsection{\texorpdfstring{Universal Approximation in {$C^k_B$ Compact-Open Topologies}}{Universal Approximation in Bastiani Compact-Open Topologies}}\label{subsec:UAT_compacts}
In this section we state a universal approximation theorem for nonlinear operators and their derivatives in {$C^k_B$ compact-open topologies}. {The EDA results (Subsection \ref{subsec:EDA}) also apply to DeepONets, Deep-H-ONets/HGNOs, and Schauder Operator Nets; see Subsections \ref{sec:DeepONets}--\ref{subsec:Svahuder_operator_nets}.}

Let $(X,|\cdot|_X),(Y,|\cdot|_Y)$ be  Banach spaces. Let $k\in \mathbb N$.  Recall also Definition \ref{def:compact_open_C2B}.  We postpone the proof to Appendix \ref{sec:proofs}.
\begin{theorem}[Universal Approximation in $C_B^{k}(X,Y)^{co}$]\label{th:UAT_K_Banach} {(i) Let $X$ satisfy Assumption \ref{ass:approx_pro} and let Assumption \ref{ass:finite_dim_approx}  hold.
Then the set $\mathcal{A}_{ED}(X,Y)$ is dense in $C_B^{k}(X,Y)^{co}$.    In particular,  $C^k_{Cyl}(X,Y)$ is dense in $C_B^{k}(X,Y)^{co}$.}

{(ii) In addition, let $Y$ satisfy Assumption \ref{ass:approx_pro}. Then, for all $F \in C_B^k(X;Y)$,  for all $K,K' \subset X$ compact,  for all $\epsilon>0$, there exist $\alpha \in A,\beta \in B,$ and $\theta \in \Theta_{N_\alpha,N_\beta}$ such that $p^1_{ K,  K'}(F-F^{{\mathcal E}^X_\alpha,\theta,{\mathcal D}^Y_\beta})<\epsilon,$ i.e.
    \begin{align}\label{eq:UAT_Ck}    
   &\sup_{x\in K,h^1\in K',...,h^i\in K'}  |[D^iF(x)-D^iF^{{\mathcal E}^X_\alpha,\theta,{\mathcal D}^Y_\beta}(x)](h^1,...,h^i)|_Y<\epsilon, \quad \forall \ 0\leq i \leq k.
    \end{align}}
\end{theorem}
\begin{remark}\label{rem_AP_necessary}
AP is necessary. Indeed, let  $X=Y$ violating AP\footnote{recall \cite{enflo1973counterexample} for a counterexample} and consider $F\in {C^\infty_F} (X,X),F(x):=x$, $DF(x)\equiv I$.  Then, there is $K\subset X$ compact and $\delta>0,$ such that 
 $$ \delta  \leq \inf_{T\in \mathcal F(X)} \sup_{h \in K} |(I-T)h|_Y \leq \sup_{h \in K}|(I-DF^{{\mathcal E}^X,\theta,{\mathcal D}^X}(x))h|_Y =\sup_{h \in K}|(DF(x)-DF^{{\mathcal E}^X,\theta,{\mathcal D}^X}(x))h|_Y ,$$ 
 where we have taken arbitrary $F^{{\mathcal E}^X,\theta,{\mathcal D}^X} \in \mathcal A_{ED}(X,Y)$, $x\in X$ and used that $DF^{{\mathcal E}^X,\theta,{\mathcal D}^X}(x)\in \mathcal F(X)$ (which is implied by \eqref{eq:derivativeDrv_dh^theta_dY_Banach_cyl}). Since  $F^{{\mathcal E}^X,\theta,{\mathcal D}^X} \in {\mathcal A_{ED}(X,Y)}$ is arbitrary, this violates UA in $C^1_B(X,X)^{co}$.
\end{remark}
\begin{remark}\label{rem:finite_dim_compact-op_UA}
    When $X$ and $Y$ are finite dimensional, we recover the standard case \cite{hornik1991approximation} by Remark \ref{rem:finite_dimensional_CO}. 
\end{remark}
\begin{remark}
The result is a generalization to $k>0$ of standard UATs for continuous operators  \cite[Theorem 4.1]{kovachki-lanthaler-stuart} and \cite[Lemma 22]{kovachki2023neuraloperator}.
\end{remark}
\begin{remark}\label{rem:UA_compacts_sharp}
This topological space is chosen carefully: one could endow $C^k_B(X,Y)$ (or ${C^k_F}(X,Y)$), for instance, with {compact-open topologies} constructed from strong or weak operator topologies on $\mathcal L^i(X,Y)$ and still achieve UA. However, the resulting topological spaces would be  weaker than $C_B^k(X; Y)^{co}$,  leading to a weaker UA.
\end{remark}
\begin{remark}\label{rem:deeponets-deephonets-etc}
Point (ii) allows us to achieve UA for specific architectures in Operator Learning: for instance it allows {us} to apply the theorem to DeepONets, {Deep-H-ONets}, or Schauder Operator nets (see Section \ref{sec:EDN}) while (i) would not be enough. 
\end{remark}
\begin{remark}
{The theorem  extends to the case $k=\infty$ using the topological space $C_B^\infty(X,Y)^{co}$.}
\end{remark}
\begin{remark}[{Comparison with UA in {compact-open topologies} in \cite{yao-luo-cao-Kovachki-Roseberry-ghattas}}]\label{rem:comparison_UA_compact_yao}
We compare Theorem \ref{th:UAT_K_Banach} with the corresponding one {in} \cite{yao-luo-cao-Kovachki-Roseberry-ghattas} to show the  advantages of our  approach,  even when we restrict to
 $k=1$, separable Hilbert spaces $X=H^s\left(\mathbb{T}^n ; \mathbb{R}^{d_a}\right), Y=H^{r}\left(\mathbb{T}^n ; \mathbb{R}^{d_b}\right)$, $s,r\geq 0$ (recall the literature review; recall also, e.g., \cite{llavona1986}  for polynomial subalgebras). Indeed, in this case: 
 \begin{enumerate}
     \item our approach is intrinsic: as a consequence, let $K\subset X$ compact measuring input variable $x$.
 Let $F(x)=Lx, DF(x)=L$, for $L\in \mathcal L(X,Y)\cap \mathcal L_2(X,Y)^c$. 
 Then UAT on compact sets in \cite{yao-luo-cao-Kovachki-Roseberry-ghattas} applies with  the error between $DF(x)$ and the surrogate being measured in $\|\cdot \|_{\mathcal L_2(X_\delta,Y)}$, i.e. $DF(x)h$ is tested along directions $h$ restricted to 
 $h\in X_\delta:=H^{s+\delta}\left(\mathbb{T}^n ; \mathbb{R}^{d_a}\right)$, for $\delta>n/2$  (recall the literature review above).
 As a  special case of our approach,  in Theorem \ref{th:UAT_K_Banach} we can choose, e.g., {$K':=\{x\in X_\rho:|x|_{X_\rho}\leq R\},$ $ R>0$, which is compact in $X$ for all $\rho>0$ (by compact embeddings \cite{taylor2010}), leading us to UA in compact-open topologies under  $\|\cdot\|_{\mathcal L(X_\rho,Y)}$ (weaker than $\|\cdot\|_{\mathcal L(X,Y)}$). Choosing $0<\rho\leq n/2$, we have  $X_\rho\supsetneq X_\delta$,}  i.e. we achieve UA in directions of lower
regularity than the ones in $X_\delta$ used in \cite{yao-luo-cao-Kovachki-Roseberry-ghattas}.  Equally importantly, we have the freedom to choose any compact set $K'\subset X$. 
\item {We work with the Bastiani class $C^1_B$, which strictly contains the Fréchet class ${C^1_F}$ used in \cite{yao-luo-cao-Kovachki-Roseberry-ghattas}. This allows us to treat Nemytskii/superposition operators arising in neural operator architectures, which need not be Fréchet differentiable. For instance, standard FNO layers have the form
\[
\mathcal L_\ell(v)=\sigma\!\left(
W_\ell v+b_\ell+
\mathcal F^{-1}\!\big(P_\ell\odot\mathcal F(v)\big)
\right).
\]
Even for smooth $\sigma$, these layers need not be Fréchet differentiable on $X=L^2(\mathbb T^d;\mathbb R^{d_v})$. Under the assumptions of Lemma \ref{lemma:Nemytskii_bastiani}, applied componentwise, composition with the continuous affine map inside the activation gives $\mathcal L_\ell\in C^1_B(X,X)$. The chain rule then implies that the resulting FNO is continuously Bastiani differentiable between the corresponding $L^2$ spaces. Thus, the $C^1_B$ compact-open topologies provide a natural analytic framework for FNOs.}
 \end{enumerate}
 {As noticed in \cite{yao-luo-cao-Kovachki-Roseberry-ghattas}, from a numerical point of view {a} Hilbert--Schmidt norm may be natural. However, typical training in ML/OL/DIOL is in $L^p$/Sobolev spaces. Hence, we refer to the comparison for UA in weighted Sobolev norms in Remark \ref{rem:comparison_UA_sobolev_yao}, where our approach will also give us UA in weighted Sobolev norms under suitable Hilbert--Schmidt norms and, furthermore, naturally lead to the Bastiani--Sobolev Training in DIOL (Section \ref{sec:bastiani-sobolev-training}).}
\end{remark}

\subsection{Universal Approximation in Weighted Bastiani--Sobolev spaces}\label{subsec:sobolevUAT}
We now state universal approximation  in weighted Bastiani--Sobolev spaces. See the space definitions in Appendix \ref{subsec:sobolev_sp} and the inclusions in Theorem \ref{th:relations_sobolve_norms}.
 We postpone the proofs to Appendix \ref{sec:proofs}.
{The EDA results (Subsection \ref{subsec:EDA}) also apply to DeepONets, Deep-H-ONets/HGNOs, and Schauder Operator Nets; see Subsections \ref{sec:DeepONets}--\ref{subsec:Svahuder_operator_nets}.} 
 In the following, let $k\in \mathbb N$, $q\geq 0 ,p,r\geq 1$ be fixed. 
\begin{theorem}[Universal Approximation in Bastiani--Sobolev spaces]\label{th:UAT_Sobolev_Banach}
(i) Let $\mu$ be a   finite measure on $X\times X^k$ 
 (recall Notation \ref{not:marginals}). Let Assumptions \ref{ass:finite_dim_sobolev_approx}, \ref{ass:Radon-nik} hold.  Then $\mathcal {A}_{ED}(X,Y)$ is dense in $C^{k;p}_{B,\mu,A}(X,Y)$. That is, for all $F \in C^{k;p}_{B,\mu,A}(X,Y)$, $\epsilon>0$, there exist $N,N' \in \mathbb N$, ${\mathcal E}^X\in \mathcal L(X,\mathbb R^N),{\mathcal D}^Y\in \mathcal L(\mathbb R^{N'},Y)$, and $\theta \in \Theta_{N,N'}$ such that $\|F-F^{{\mathcal E}^X,\theta,{\mathcal D}^Y}\|_{\mathcal W^{k,p}_{B,\mu}(X,Y)}<\epsilon$, i.e.
    \begin{align*}
   &\sum_{i=0}^k\int_{X \times X^i} |[D^iF(x)-D^iF^{{\mathcal E}^X,\theta,{\mathcal D}^Y}(x)](h^1,...,h^i)|^p_Yd\mu^{0:i}(x,h^1,\ldots,h^i)<\epsilon^p.
    \end{align*}
    
 (ii)  {Let $\nu$ be a finite measure on $X$ and $\eta$ be a probability measure on $X^k$ such that $\|\eta\|_{k,r}<\infty$.} {Let Assumption \ref{ass:finite_dim_sobolev_approx} hold.  Then $\mathcal {A}_{ED}(X,Y)$ is dense in $C^{k;p,r}_{B,\nu,\eta,A}(X,Y)$. That is, for all $F \in C^{k;p,r}_{B,\nu,\eta,A}(X,Y)$, $\epsilon>0$, there exist $N,N' \in \mathbb N$, ${\mathcal E}^X\in \mathcal L(X,\mathbb R^N),{\mathcal D}^Y\in \mathcal L(\mathbb R^{N'},Y)$, and $\theta \in \Theta_{N,N'}$ such that $\|F-F^{{\mathcal E}^X,\theta,{\mathcal D}^Y}\|_{\mathcal W^{k,p,r}_{B,\nu,\eta}(X,Y)}<\epsilon$, i.e.
    \begin{align*}
   &\sum_{i=0}^k\int_{X} \left( \int_{X^i} |[D^iF(x)-D^iF^{{\mathcal E}^X,\theta,{\mathcal D}^Y}(x)](h^1,...,h^i)|^r_Y d\eta^{1:i}(h^1,\ldots,h^i)\right)^{p/r}d\nu(x)<\epsilon^p.
    \end{align*}}
\end{theorem}
\begin{remark}\label{rem:UAT_sobolev}
\begin{enumerate}[(i)]
\item  In point (ii), we interpret $\nu$ as the input measure on $X$ to measure approximation errors in the input variable $x$ and  $\eta$ as auxiliary measures  to measure approximation errors in the additional variables $h^j,j\leq k$ appearing from derivatives $D^iF$. A similar interpretation holds  for point (i) (recall also Example \ref{ex:ass_radon-nik} (i) and  \eqref{eq:sobolev_norm_mu0_Lp}).      \item {When $X=\mathbb R^{N}$ and $Y=\mathbb R^{N'}$, we recover the standard case \cite[Theorem 4]{hornik1991approximation} using Theorem \ref{th:UAT_Sobolev_Banach} (ii) (or Theorem \ref{th:UAT_Sobolev_Banach} (i) with $\mu=\mu^0\otimes \eta$) using  Proposition \ref{prop:relations_sobolev_norms_1}  (ii). Furthermore, we will also recover the finite-dimensional case in Corollary \ref{th:UAT_Sobolev_K_Banach_bump} (see Remark \ref{rem:UA_product_measure}).}
    \item In the following corollaries, we will use Theorem \ref{th:UAT_Sobolev_Banach} to prove UA for concrete classes of maps, under the Bounded Approximation Property or Gaussian measures. {However, the theorem is general as it is formulated for the abstract class $C^{k;p}_{B,\mu,A}(X,Y)$ and can be applied to spaces without the Approximation Property and under non-Gaussian measures, as Example \ref{ex:strict_inclusions} shows (in particular, $\tilde F$ there).}
\end{enumerate}
\end{remark}

Under the Bounded Approximation Property for separable Banach spaces $X,Y$ (Assumption \ref{ass:bounded_approx_pro}),
Theorem \ref{th:relations_sobolve_norms} (in particular, \eqref{eq:Ckwq_subset_Ckwc} there) implies the following corollary.
\begin{corollary}\label{th:UAT_Sobolev_K_Banach}
(i) Let $\mu$ be a   finite measure on $X\times X^k$ 
 (recall Notation \ref{not:marginals}). Let Assumptions \ref{ass:bounded_approx_pro} {(for $X,Y$)}, \ref{ass:finite_dim_sobolev_approx}, \ref{ass:Radon-nik} hold. Assume $\|\mu\|_{k,q,p}  < \infty$.  Then $\mathcal {A}_{ED}(X,Y)$ is dense in $C^{k;p,q}_{B}(X,Y)$. In particular, for all $F \in C^{k;p,q}_{B}(X,Y)$, $\epsilon>0$, there exist $d,m \in \mathbb N$ and $\theta \in \Theta_{N_d,N_m}$ such that $\|F-F^{{\mathcal E}^X_{d},\theta,{\mathcal D}^Y_m}\|_{\mathcal W^{k,p}_{B,\mu}(X,Y)}<\epsilon$.

{(ii)  {Let  $\nu$ be a finite measure on $X$ and $\eta$ be a probability measure on $X^k$ such that $\|\eta\|_{k,r}<\infty$.} Let Assumptions \ref{ass:bounded_approx_pro} {(for $X,Y$)}, \ref{ass:finite_dim_sobolev_approx} hold. Assume $\int_X\left(1+|x|_X^q\right) d \nu(x)<\infty$.  Then $\mathcal {A}_{ED}(X,Y)$ is dense in $C^{k;p,q}_{B}(X,Y)$. In particular, for all $F \in C^{k;p,q}_{B}(X,Y)$, $\epsilon>0$, there exist $d,m \in \mathbb N$ and $\theta \in \Theta_{N_d,N_m}$ such that $\|F-F^{{\mathcal E}^X_{d},\theta,{\mathcal D}^Y_m}\|_{\mathcal W^{k,p,r}_{B,\nu,\eta}(X,Y)}<\epsilon$.}
\end{corollary}

In the following corollary, we work under Bounded Approximation Property for separable Banach spaces $X,Y$ (Assumption \ref{ass:bounded_approx_pro}), as well as Assumption \ref{ass:density_CB}; recall also Subsection \ref{subsubsec:ass_density_CB} (in particular, Proposition \ref{pro:ass_density_holds} and Remark \ref{rem:prop_rescaledd_bump}; so also Proposition \ref{prop:bump_rescaled} and Example \ref{ex:norm_Fréchet}) for natural examples under which Assumption \ref{ass:density_CB} holds. With respect to the previous corollary, this allows us to remove growth assumptions on $D^iF,i\leq k$, as well as moment conditions on the input measure $\nu$ (in the sense that we only assume natural moment conditions in the $h^j$, $j\leq k$ auxiliary variables, $\|\mu\|_{k,0,p}<\infty$ in Assumption \ref{ass:density_CB}).
\begin{corollary}\label{th:UAT_Sobolev_K_Banach_bump}
 Let $\mu$ be a   finite measure on $X\times X^k$ 
 (recall Notation \ref{not:marginals}). Let Assumptions \ref{ass:bounded_approx_pro} {(for $X,Y$)}, \ref{ass:finite_dim_sobolev_approx}, \ref{ass:Radon-nik}, \ref{ass:density_CB} hold.  
Then $\mathcal {A}_{ED}(X,Y)$ is dense in $ \widetilde C^{k;p}_{B,\mu}(X,Y)$. 
In particular, for all $F \in  \widetilde C^{k;p}_{B,\mu}(X,Y)$, $\epsilon>0$, there {exist} $d,m \in \mathbb N$ and $\theta \in \Theta_{N_d,N_m}$ such that $\|F-F^{{\mathcal E}^X_{d},\theta,{\mathcal D}^Y_m}\|_{\mathcal W^{k,p}_{B,\mu}(X,Y)}<\epsilon$. 
\end{corollary}
Similarly to the previous results, an analogous statement in Corollary \ref{th:UAT_Sobolev_K_Banach_bump} is easily proved for $ \widetilde C^{k;p,r}_{B,\nu,\eta}(X,Y)$.
\begin{remark}\label{rem:UA_product_measure}
Let $p\geq 1$ and $\mu=\nu\otimes\eta$, where $\nu$ is a finite measure on $X$ and $\eta$ is a probability measure on $X^k$ with $\|\eta\|_{k,p}<\infty$. Then, Example \ref{ex:ass_radon-nik} holds; moreover, $\|\mu\|_{k,0,p}<\infty$ (consistently with Assumption \ref{ass:density_CB}). Therefore,  in Corollary \ref{th:UAT_Sobolev_K_Banach_bump} as well, we allow the input measure $\nu$ to be any finite measure.
{As a consequence, when $X=\mathbb R^{N}$ and $Y=\mathbb R^{N'}$, we recover again the standard case \cite[Theorem 4]{hornik1991approximation}  using Remark \ref{rem:seminorm_sobolev_bastiani_tilde_reverse-ineq}.}
\end{remark}

{When the underlying measures are regular enough for closing derivative operators, we can complete the weighted Bastiani--Sobolev space. This is possible, e.g. in the Gaussian case:}
the following result follows by Assumption \ref{ass:finite_dim_sobolev_approx}, thanks to the definition of the classical Gaussian Sobolev Space \cite[Section 5.2]{bogachev1998gaussian} given in \eqref{eq:gaussian_sobolev}, or, equivalently, Theorem \ref{th:characterization_bogachev}; that is, it is a corollary of (the proof of) Theorem \ref{th:UAT_Sobolev_Banach} (ii).
\begin{corollary}[Universal Approximation in Gaussian Sobolev spaces]\label{th:UA_gaussian}
Let Assumption \ref{ass:finite_dim_sobolev_approx} hold. Let $p\geq 1 $ and let $\gamma$ be a centered Gaussian measure on  a (separable) Banach space $X$. Let $Y$ be a separable Hilbert space.   Then $\mathcal {A}_{ED}(X,Y)$ is dense in $ W_{\gamma}^{k,p}(X,Y)$. 
\end{corollary}
{\begin{remark}[Comparison with UA in weighted Sobolev norms in \cite{yao-luo-cao-Kovachki-Roseberry-ghattas}]\label{rem:comparison_UA_sobolev_yao}For the sake of the comparison with \cite{yao-luo-cao-Kovachki-Roseberry-ghattas} let us restrict to  the setting there, i.e.  let $k=1$, separable Hilbert spaces $X=H^s\left(\mathbb{T}^n ; \mathbb{R}^{d_a}\right), Y=H^{r}\left(\mathbb{T}^n ; \mathbb{R}^{d_b}\right)$, $s,r\geq 0$. Even in this case:
\begin{enumerate}
    \item our analytic framework is more general and more natural. Indeed, from \eqref{eq:def_sobolev_norm_mu0_Lp_bogachev_HS} (by suitably choosing the covariance operator $Q$) we have that our Bastiani--Sobolev norms provide a powerful generalization of the weighted Sobolev norms used in \cite{yao-luo-cao-Kovachki-Roseberry-ghattas}.
    \item  Furthermore, as already mentioned in Remark \ref{rem:comparison_UA_compact_yao}, we use the larger class of Bastiani $C^1_B$  (in place of ${C^1_F}$ Fréchet in \cite{yao-luo-cao-Kovachki-Roseberry-ghattas}), which allows us to work, e.g., with {Nemytskii-type} operators.
    \item  Furthermore, we note that the proof in \cite{yao-luo-cao-Kovachki-Roseberry-ghattas} relies on approximation steps via Encoder-Decoder/cylindrical maps. 
\end{enumerate}
In view of these points, although our UATs do not cover FNOs (or, more generally, Neural operators) our weighted  Bastiani--Sobolev spaces provide a natural analytic framework for these architectures (recall also Remark \ref{rem:comparison_UA_compact_yao}).
\end{remark}}
\begin{remark}\label{rem:UAT_Sobolev_Banach_supcompact}
Another interesting generalization of the classical finite dimensional case  is where we replace operator norms with the seminorms of compact-open topologies on spaces of multi-linear operators, i.e. statements of these kinds: let $\nu$ {be a} finite measure on $X$ and let $K\subset X$ compact. Then, for all $F\in C^k_B(X,Y)$ (with suitable integrability), for all $\epsilon>0$, there exist $N,N' \in \mathbb N$, ${\mathcal E}^X\in \mathcal L(X,\mathbb R^N),{\mathcal D}^Y\in \mathcal L(\mathbb R^{N'},Y)$, and $\theta \in \Theta_{N,N'}$ such that
    \begin{align*}
\sum_{i=0}^k\int_{X } \sup_{h^1,\ldots,h^i\in  K}|[D^iF(x)-D^iF^{{\mathcal E}^X,\theta,{\mathcal D}^Y}(x)](h^1,...,h^i)|^p_Yd\nu(x)<\epsilon^p.
    \end{align*}
 By taking ${K}$ to be the closed unit ball in $X$ we recover the classical finite-dimensional case \cite{hornik1991approximation}. However, from a training/numerical perspective, the previous formulations are more natural (see the next section).
\end{remark}
\section{Bastiani--Sobolev Training in Derivative-Informed Operator Learning}\label{sec:bastiani-sobolev-training}
{Our UATs support \textit{Bastiani--Sobolev training} on Banach spaces, extending existing DIOL formulations based on Hilbert spaces, Fréchet (or, in the Gaussian case, stochastic Gâteaux/Malliavin) differentiability, and $k=1$. The approach is compatible with Sobolev training in high, but finite, dimensional spaces~\cite{czarnecki2017}. The references below implement these methods in some settings; a full computational study is left to future work.}

Let {$X,Y$ be} Banach spaces, let $k\in \mathbb N, p \geq 1,$ let $\nu$ be a probability measure on $X$  and $\eta$ be a probability measure on $X^k$ such that $\|\eta\|_{k,p}<\infty$.  Let $F\in C^{k;p,p}_{B,\nu,\eta,A}(X,Y)$ (or, when possible, let $F$ belong to the completion of this space, e.g. in the Gaussian case in Theorem \ref{th:characterization_bogachev}, therefore including stochastic Gâteaux/Malliavin  derivatives). Via supervised or unsupervised learning, the goal is to learn $F$ and its derivatives $D^iF$, for all $i\leq k$. To this purpose,
choose {a suitable encoder and decoder} $\mathcal E^X,{\mathcal D}^Y$ (e.g. if $X,Y$ satisfy the BAP, Assumption \ref{ass:bounded_approx_pro} and under Theorem \ref{th:relations_sobolve_norms} (v),  it is enough to take $\mathcal E^X=\mathcal E_d^X,{\mathcal D^Y=\mathcal D^Y_m}$ for $d,m$ large) and let  $F^{{\mathcal E}^X,\theta,{\mathcal D}^Y}\in \mathcal {A}_{ED}(X,Y)$, $\theta\in \Theta_{N_d,N_m}$.  We also observe that, in DIOL over specific function spaces, one can also use other OL architectures such as Fourier Neural Operators as in \cite{yao-luo-cao-Kovachki-Roseberry-ghattas}; in this case, as explained in Remark \ref{rem:comparison_UA_sobolev_yao}, although our UATs do not directly cover these {architectures}, our weighted Bastiani--Sobolev spaces provide the natural analytic framework for UATs for these architectures.

Our UATs in Bastiani--Sobolev {spaces support} the simultaneous learning of $F$, $D^iF$, for all $i\leq k$ in the norm \eqref{eq:sobolev_norm_mu0_Lp_bogachev}; i.e. via supervised or unsupervised learning, we aim to minimize
{(with usual notations\footnote{recall in particular that $D^0F=F$})
\begin{align}\label{eq:sobolev_trainig_bastiani_norm}
 \|F-F^{{\mathcal E}^X,\theta,{\mathcal D}^Y}\|_{\mathcal W^{k,p,p}_{B,\nu,\eta}(X,Y)}^p  =\sum_{i=0}^k\int_{X \times X^i} |[D^iF(x)-D^iF^{{\mathcal E}^X,\theta,{\mathcal D}^Y}(x)](h^1,...,h^i)|^p_Yd\nu(x)d\eta^{1:i}(h^1,\ldots,h^i)
    \end{align}
    or in suitable variants of this, see Section \ref{subsec:Sobolev}; see also \cite{cohen-defeo-hebner-sirignano}. 
     We call the family of methods below that aim to minimize these norms (either via supervised or unsupervised learning) \textbf{Bastiani--Sobolev Training} methods.} 
    \subsection{Supervised Bastiani--Sobolev Training}
\subsubsection{Supervised Bastiani--Sobolev Training}\label{subsec:supervised_bastiani-otf}
Let $X,Y$ be Banach spaces.
Our UATs and \eqref{eq:sobolev_trainig_bastiani_norm} directly lead to the following form of training, which is  consistent with the Sobolev training in  high, but finite, dimensional spaces \cite[Section 2]{czarnecki2017}: in each iteration, we sample $x_l\sim \nu,$ $(h^1_{l,j},\ldots,h^k_{l,j})\sim \eta,$ $l=1,\ldots ,M, j=1\ldots, Q$;
the empirical loss is evaluated via Monte Carlo on $X\times X^k$, i.e.
\begin{align*}
&{\mathcal L}^{k,p}(\theta) \approx \frac 1 {M Q}\sum_{l=1}^{M}\sum_{j=1}^{Q} \sum_{i=0}^k
|D^iF(x_l)(h^1_{l,j},\ldots,h^i_{l,j})-D^iF^{{\mathcal E}^X,\theta,{\mathcal D}^Y}(x_l)(h^1_{l,j},\ldots,h^i_{l,j})|_Y^p;
\end{align*}
 we update parameters by stochastic optimization and we repeat the cycle.

 We call these methods (supervised)  \textbf{Bastiani--Sobolev Training}. {These types of training {address} the problems of high computational costs in an alternative way with respect to the methods in Subsection \ref{subsec:sobolev_training}:
as already in the high, but finite, dimensional case  \cite{czarnecki2017}, we reduce training complexity by training derivatives against sampled directions $(h^1_{l,j},\ldots,h^k_{l,j})$.}  The paper   \cite{alonso2026}  reports a finance-specific implementation of a related first-order directional-derivative loss (we note that \cite{alonso2026} was posted on arXiv after the initial arXiv version of the present manuscript, as well as after \cite{cohen-defeo-hebner-sirignano}, where similar norms to \eqref{eq:sobolev_trainig_bastiani_norm} were introduced and numerically validated). 

\subsubsection{Supervised Bastiani--Sobolev Training via Dimensionality Reduction}\label{subsec:sobolev_training}
{When $X,Y$ {are} separable Hilbert spaces, the standard DIOL literature typically  uses Fréchet (or, in the Gaussian case, stochastic Gâteaux/Malliavin) derivatives and the weighted Sobolev norm  \eqref{eq:def_sobolev_norm_mu0_Lp_bogachev_HS} under a probability measure $\nu$ on a Hilbert space $X$ and  Hilbert--Schmidt norms in Cameron--Martin spaces/RKHS $X_\gamma=Q^{1/2}(X)$, $Q$ trace-class covariance operator on $X$ (e.g. \cite{Roseberry-peng-villa-ghattas,luo-Roseberry-chen-Ghattas,yao-luo-cao-Kovachki-Roseberry-ghattas}). 
For $p=2$ (or using the variant \eqref{eq:sobolev_norm_mu0_Lp_bogachev} with $r=2$ and $p\geq 1$), choosing $\eta:=\otimes_{i=1}^k\gamma$ for a Gaussian measure $\gamma$ on $X$, in view of \eqref{eq:sobolev_norm_mu0_Lp_bogachev_HS_p=2}, we have that \eqref{eq:def_sobolev_norm_mu0_Lp_bogachev_HS} coincides with the Bastiani--Sobolev norm \eqref{eq:sobolev_trainig_bastiani_norm} (or its variants); hence,  our UATs directly support this training procedure.} {Due to the computational  costs (including derivative data generation) of high-dimensional approximations of infinite-dimensional derivatives (already for $k=1$), the DIOL literature compresses information by choosing suitably the encoder and {decoder} via {Dimensionality Reduction}, such as principal component analysis (PCA) and derivative-informed
subspaces (DIS), which use dominant eigenvectors of the covariance of the data or the derivatives as the reduced bases, respectively.} {In each iteration of these algorithms, sample $(x_l){\sim}\nu$, $l=0,\ldots ,M$;  the empirical loss is  evaluated via Monte Carlo on $X$, i.e.
\begin{align*}
&{\mathcal L}^{k,p}(\theta) \approx \frac 1 M\sum_{l=1}^M\sum_{i=0}^k
|D^iF(x_l)-D^iF^{{\mathcal E}^X,\theta,{\mathcal D}^Y}(x_l)|_{\mathcal L_2^i(X_\gamma,Y)}^p;
\end{align*}
 update parameters by stochastic optimization and  repeat the cycle. 
 
 However, our UATs allow us to consider the  more general framework of Bastiani--Sobolev spaces over a separable \textbf{Banach} space $X$ (recall that \eqref{eq:sobolev_norm_mu0_Lp_bogachev_HS_p=2} holds for a separable Banach space $X$), leading, possibly,  to (supervised) \textbf{Bastiani–Sobolev Training} methods via \textbf{Dimensionality Reduction}.
}

\subsection{Physics-Informed Bastiani--Sobolev Training}
Let $X,Y,V$ be Banach spaces, $k\in\mathbb N$, and  consider, formally, a $V$-valued infinite-dimensional PDE
\begin{equation}\label{eq:PDE_training_valuedW}
    G\bigl(x,F,D^1F,\ldots,D^kF\bigr)=0,\quad x\in X ,
\end{equation}
where $G$ is defined from suitable subsets of $X\times\prod_{i=0}^k \mathcal L^i(X,Y)$ to $V$.  In Physics-Informed Bastiani--Sobolev Training, we {aim to learn the solution $F$ of this PDE} by minimizing a suitable weighted $L^2$-norm of the residual. {This can be seen as DIOL as this loss  requires parameterizing $F$ and its derivatives $D^iF, i\leq k$.} {We review these methods and how these UATs support these.}
\subsubsection{Deep Galerkin Methods on Infinite-Dimensional Spaces}\label{subsec:training_DGM}
Consider the  setting of \cite{cohen-defeo-hebner-sirignano} (in \cite{cohen-defeo-hebner-sirignano}, $X$ is {a} Hilbert space) or possible extensions to the Banach case, i.e. we deal with a PDE  of the form \eqref{eq:PDE_training_valuedW} with values in a finite-dimensional space $V=\mathbb R^n$ and $k=2$. Then, Deep Hilbert--Galerkin  Methods minimize a suitable weighted $L^2$-norm of the residual
of \eqref{eq:PDE_training_valuedW} {(see Subsection \ref{subsec:numerical_PDEs})}. In fact UATs proved {in \cite{cohen-defeo-hebner-sirignano}}, strongly related to the ones here, support this training procedure. 
\subsubsection{Bastiani--Sobolev Training On-The-Fly}\label{subsec:unsupervised-onthefly}
Consider the sensitivity equation \eqref{eq:sensitivity_eq} and the setting there. In operator form, it is a first-order equation for $F$, with values in $\mathcal L(X,W)$:
\[
 D_yG(F(u),u)DF(u)+D_uG(F(u),u)=0.
\]
The paper \cite{Mou2026} proposed on-the-fly derivative-informed training using randomly sampled input directions and Jacobian--vector products of the current neural operator (we note that \cite{Mou2026} was first posted on arXiv on 26 September 2026, after the initial arXiv version of the present manuscript, as well as after \cite{cohen-defeo-hebner-sirignano}). {The raw tangent $L^2$-loss introduced there is $\|R_\theta\|_{L^2(\nu\otimes\eta;W)}^2$, where $R_\theta(u,h):=D_yG(F_\theta(u),u)[DF_\theta(u)h]+D_uG(F_\theta(u),u)[h]$ and $F_\theta$ denotes the surrogate. This is estimated by random sampling and requires no offline derivative labels.  For encoder--decoder architectures, it can be expected that, under suitable continuity assumptions on the residual map, this loss can be made arbitrarily small using our UATs.}
\appendix
\section{Basic notations}\label{app:basics}
In this manuscript, we often denote by $C>0$ a generic constant, which may change from line to line. 
\subsection{Moore–Smith sequences and sequences}
{We use the term Moore–Smith sequence for a topological net to distinguish it from a neural network.}
\begin{definition}[Directed set]
A directed set is an index set $A$ equipped with $\preceq
$ such that (i)
 is a pre-order relation (i.e. {a} reflexive and transitive binary relation);
    (ii) for all $\alpha, \beta \in A$, there exists $\gamma \in A$ so that $\alpha \preceq
 \gamma$ and $\beta \preceq
 \gamma$.
\end{definition}
\begin{definition}[Moore–Smith sequences and sequences]
A Moore–Smith sequence (or net) in a topological space $X$ is a mapping from a directed set $A$ to $X$; we denote it by $\left\{x_\alpha\right\}_{\alpha \in A}$. When $A=\mathbb N$, a Moore–Smith sequence is called a sequence, and we denote it by  $\left\{x_n\right\}_{n \in \mathbb N}$.  In this manuscript, to distinguish Moore–Smith sequences from sequences we will use Greek letters (typically $\alpha,\beta$) for indexes of Moore–Smith sequences and roman letters (typically $n,d,m$) for indexes of sequences.
\end{definition}

\begin{definition}[Convergence]
A Moore–Smith sequence $\left\{x_\alpha\right\}_{\alpha \in A}$ in a topological space $X$ is said to converge to a point $x \in X$, denoted $x_\alpha \rightarrow x$, if for any neighborhood $N$ of $x$, there is {an} $\alpha_0 \in A$ so that $x_\alpha \in N$ if $\alpha \succeq \alpha_0$. If $X$ is Hausdorff, the limit is unique.
\end{definition}
\subsection{Locally convex topological vector spaces}
\begin{definition}[Seminorms and norms]
A seminorm on a vector space $W$ is a map $\rho: W \rightarrow[0, \infty)$ such that (i) $\rho(x+y) \leq \rho(x)+\rho(y)$; (ii) $\rho(\lambda x)=|\lambda| \rho(x)$ for $\lambda \in \mathbb{R}$.
A seminorm $|\cdot|$ is a norm if (iii) $|x|=0$ implies $x=0.$
Let $I$ be an index set.  A family of seminorms $\left\{\rho_i\right\}_{i \in I}$ is said to separate points if $\rho_i(x)=0$ for all $i \in I$ implies $x=0$.
\end{definition}
\begin{definition}[Locally convex space]
A locally convex space is a vector space $W$ with a family $\left\{\rho_i\right\}_{i\in I}$ of seminorms. The induced topology is the weakest topology in which all $\rho_i, i\in I$ are continuous and in which the operation of addition is continuous. If the seminorms  separate points the induced topology is Hausdorff. A Moore–Smith sequence $w^\alpha\to w$ if and only if $\rho_{i}(w^\alpha-w)\to 0$ for all $i\in I.$
\end{definition}
  The family  $\left\{\rho_i\right\}_{i \in I}$ is called directed (or fundamental) if, for all $i_1, \ldots, i_n \in I$, $n\in \mathbb N$ there is a $\gamma \in I$ and $C>0$ so that
$
\max(\rho_{i_k}(x),k=1,\ldots,n)\leq C \rho_\gamma(x), $ for all $x \in W$, or equivalently, by induction, restricting $n=2$. The existence of such a family  can be assumed without loss of generality \cite[p. 126]{reed_simon}. If $\left\{\rho_i\right\}_{i \in I}$ is a directed family, then $\left\{\left\{x \in W:\rho_i(x)<\varepsilon\right\} : i \in I, \varepsilon>0\right\}$ is a neighborhood base at $0$ and $\left\{\left\{x\in W :\rho_i(x-w)<\varepsilon\right\} : i \in I, \varepsilon>0\right\}$ is a neighborhood base at ${w}\in W$.
In this case, a subset $\mathcal G$ is dense in $W$ if and only if  for all $w\in W$, $i\in I$, $\epsilon>0$, there exists $g\in \mathcal G$ such that $\rho_{i}(g-w)<\epsilon.$ Let $k\in \mathbb N$ and $W=W_1\times \ldots W_k$, where each $W_j$, $j\leq k$ is a locally convex topological vector space whose topology is generated by a family of directed seminorms $\{\rho_{i_j}^j\}_{i_j \in I_j}$.  We endow $W$ with the product topology generated by the family of directed seminorms $\{\rho_i\}_{i \in I}$, where $\rho_i(w):=\max(\rho_{i_j}^j(w_j),j\leq k)$, $w=(w_1,\ldots,w_k)$, $i=(i_1,\ldots,i_k)\in I:=I_1\times\ldots \times I_k$.
\begin{definition}[Equivalent seminorms]\label{def:eq_seminorms}
    Two families of seminorms $\left\{\rho_i\right\}_{i \in I}$ and $\left\{\nu_\beta\right\}_{\beta \in B}$ on a vector space $X$ are called equivalent if one of the following equivalent conditions is satisfied:
\begin{enumerate}
    \item[(i)] they generate the same topology.
    \item[(ii)] Each $\rho_i$ is continuous in the ${\nu}$-natural topology and each ${\nu_\beta}$ is continuous in the $\rho$-natural topology.
\item[(iii)]  For each $i \in I$, there are $\beta_1, \ldots, \beta_n \in B$ and $C>0$ so that 
$
\rho_i(x) \leq C\left(\nu_{\beta_1}(x)+\cdots+\nu_{\beta_n}(x)\right)
$, for all $x \in X$,
and for each $\beta \in B$, there are $i_1, \ldots, i_m \in I$ and $D>0$ so that  
$
\nu_\beta(x) \leq D\left({\rho_{i_1}(x)}+\cdots+\rho_{i_m}(x)\right)
$, for all $x \in X$.
\end{enumerate}
\end{definition}

\paragraph{Compact-open topologies.} We will often consider the compact-open-topology for locally convex topological vector spaces, so we give a general definition and then specialize it when needed. For a short introduction to this topology, see \cite[Appendix B.2]{schmeding2022}.
\begin{definition}[Compact-open topology]\label{def:compact-op}Let $X$ be a topological space and $Y$ be a locally convex topological vector space, whose topology is generated by a directed  family of seminorms  $\left\{q_i\right\}_{i \in I}$.
Denote the space of continuous functions from $X$ to $Y$ by $C^0(X;Y)$. 
We define the compact-open topology on $C^0(X;Y)$, as   the locally convex topology generated by the directed family\footnote{To see that $\mathcal P$ is directed, define 
$
K:=K_1 \cup K_2, 
$
for compacts $K_1,K_2.$
Choose $\gamma$ and $C>0$ from directedness of $\left\{q_i\right\}$. Then
$
p_{K_1, i_1}(f)+p_{K_2, i_2}(f)  =\sup _{x \in K_1} q_{i_1}(f(x))+\sup _{x \in K_2} q_{i_2}(f(x)) \leq 2 \sup _{x \in K}\left(q_{i_1}(f(x))+q_{i_2}(f(x))\right)\leq C \sup _{x \in K}  q_\gamma(f(x))  =
C p_{K, \gamma}(f).
$} of seminorms $\mathcal P=\{p_{K,i}:K\subset X\text { compact},i\in I\}$, where 
$p_{K,i}(f):=\sup _{x \in K} q_i(f(x))$, $f\in C^0(X;Y)$. 
A Moore–Smith sequence $f^\alpha\to f$ in the compact-open topology if and only if $p_{K,i}(f^\alpha-f)\to 0$, for all $K\subset X$ compact, for all $i\in I.$ 
A set $\mathcal G$ is dense in the compact-open topology of $C^0(X;Y)$ if and only if,  for all $f\in C^0(X;Y)$,  $K\subset X$ compact, $i\in I$, $\epsilon>0$, there exists $g\in \mathcal G$ such that $p_{K,i}(f-g)<\epsilon.$
\end{definition}

\subsection{Banach spaces}
Let $(X,|\cdot|_X)$ be a Banach space. When there is no chance of confusion we simply write $|\cdot|$ for the norm. {We denote the sphere $S^{X,R}=\{x\in X:|x|_X= R\}, R>0$.}
For $k\in \mathbb N$, denote by $X^k$ the product space endowed with the norm $\left|\left(x_1, \ldots, x_k\right)\right|_{X^k}:=\left(\left|x_1\right|^2+\ldots+\left|x_k\right|^2\right)^{1 / 2}$.

Let $(Y,|\cdot|_Y)$ be another Banach space. Given $F:X\to Y$ we define  $\operatorname{supp} F:=\overline{\{x : F(x) \neq 0\}}$.
We denote by $C^0(X,Y)$ the space of continuous functions from $X$ to $Y.$
We denote by $\mathcal L(X,Y)$ the Banach space of linear bounded  operators from $X$ to $Y$, endowed with the operator norm $\|L\|_{\mathcal L(X,Y)} :=\sup_{|x|_X=1}|Lx|_Y , $ $  L  \in \mathcal L(X,Y).$
We denote by $ X^*=\mathcal L(X,\mathbb R)$ the dual space.
When $Y=X$ we simply write $\mathcal{L}(X)=\mathcal{L}(X,X).$  We say that a linear operator $L:X\to Y$ is compact  if it maps bounded sets of $X$ to relatively compact sets in $Y$ (i.e. subsets with compact closure in $Y$). We denote by $\mathcal K(X,Y)$ the set of compact operators from $X$ to $Y$;  $\mathcal K(X,Y)$ turns out to be a closed subspace of $\mathcal L(X,Y)$; we denote $\mathcal K(X)=\mathcal K(X,X)$.
A linear operator $T:X\to Y$ is finite rank {if its range is finite-dimensional, or, equivalently, if} there exists $N\in \mathbb N$, $\left\{e_i\right\}_{i=1}^{N} \subset Y$,  $\left\{a_i\right\}_{i=1}^{N} \subset X^*$ such that
$T x=\sum_{i=1}^{N} \langle a_i,x\rangle_{X^*,X} \  e_i,$ for all $ x \in X.$ We denote by $\mathcal F(X,Y)$ the space of finite rank operators from $X$ to $Y$; we have that  $\mathcal F(X,Y)$ is a subspace of $\mathcal K(X,Y)$; we denote $\mathcal F(X)=\mathcal F(X,X)$. We denote by $\mathcal A(X,Y):=\overline{\mathcal F(X,Y)}^{\|\cdot\|_{\mathcal L(X,Y)}}$ the space of approximable operators via finite rank operators.

We denote by $\mathcal L^i(X;Y)$ 
the  space of  continuous multilinear operators $T: X^i \rightarrow Y$. We recall the following standard result \cite[(A.2.5)]{Flett_1980}.
\begin{proposition}\label{lemma:multilinear}
    Let $T: X^i \rightarrow Y$ be multilinear, then the following are equivalent (i) $T$ is continuous; (ii) $T$ is continuous at one point; (iii) $\|T\|_{\mathcal L^i(X,Y)}=\sup \left\{\left|T\left(h_1, \ldots, h_i\right)\right |_Y:\left|h_1\right|_X = 1, \ldots,\left|h_i\right|_X = 1\right\}<\infty$; (iv) $\mathcal C=\left\{C \geq 0:\left|T\left(h_1, \ldots, h_i\right)\right|_Y \leq C \ \Pi_{j=1}^i|h_j|_X, \  \forall h_1, \ldots, h_i \in X \right\} \neq \varnothing$.
If one (hence, all) of these conditions {holds}, then $\|T\|_{\mathcal L^i(X,Y)} \in \mathcal C$ and $\|T\|_{\mathcal L^i(X,Y)}=\inf \mathcal C$. 
\end{proposition}
The space $\mathcal L^i(X;Y)$  is a Banach space when endowed with the norm $\|\cdot\|_{\mathcal L^i(X,Y)}$. 
We use the notation $\mathcal{L}^i(X)$ when $X=Y$. 
\begin{definition}[Weaker topologies on $\mathcal L^i(X,Y)$]\label{rem:compact_open_LX}
   {In parts of this paper, we endow $\mathcal{L}^i(X,Y)$ with the compact-open topology, i.e.~the locally convex Hausdorff topology generated by  the directed family of seminorms $\mathcal{S}^{co}:=\left\{s^{co}_{K}:K \subset X\right.$ compacts$\},$ with $s^{co}_{K}(Z)=\sup _{h^1,\ldots,h^i \in K}|Z (h^1,\ldots,h^i)|_Y$ for $Z\in \mathcal{L}^i(X,Y) $. We denote by $\mathcal{L}^i(X,Y)^{co}$ this topological space. A Moore–Smith sequence $Z_\alpha \xrightarrow{co} Z\in \mathcal L^i(X,Y)$  in the compact-open topology iff $\sup_{h^1,\ldots,h^i\in K}|(Z_\alpha -Z )(h^1,\ldots,h^i)|_Y\to 0$ for all $K\subset X$ compact.}
   
    {We also consider the strong operator topology on $\mathcal{L}^i(X,Y)$, i.e.~the locally convex Hausdorff topology generated by  the  family of seminorms $\mathcal S^{s}=\{s^{so}_{h^1,\ldots,h^i},h^1,\ldots,h^i\in X\},$ with $s^{so}_{h^1,\ldots,h^i}(Z)=|Z (h^1,\ldots,h^i)|_Y$, $Z\in \mathcal{L}^i(X,Y)$.  We denote by $\mathcal{L}^i(X,Y)^{so}$ this topological space. A Moore–Smith sequence $Z_\alpha \xrightarrow{so}  Z$  in the strong operator topology iff $|(Z_\alpha -Z )(h^1,\ldots,h^i)|_Y\to 0$ for all $h^1,\ldots,h^i \in  X$.}

Compared with the strong operator topology,
   the compact-open topology is stronger.  Both  topologies are weaker than the topology induced by the operator norm.
\end{definition}

The following are standard facts of functional analysis.
\begin{lemma}[{\cite[Lemma 3.7, p. 151]{kato1995perturbation}}]\label{lemma1:uniformcompactsfromstrongandequibounded}
    Let $X, Y$ be Banach spaces, $K \subset X$ a compact set, and $T_n, T \in L(X, Y)$. If $T_n \xrightarrow{s} T$, then $T_n \xrightarrow{co} T$. 
\end{lemma}
\begin{remark}\label{rem:Tn_strongly_bounded}
     Let $X, Y$ be Banach spaces, $K \subset X$ a compact set, and $T_n, T \in L(X, Y)$. If $T_n \xrightarrow{s} T$, then, by {the Banach--Steinhaus theorem}, there exists $C>0$ such that
     $\|T_n\|\leq C.$
\end{remark}
\begin{lemma}\label{rem:relatively-compact_Pd-I_Banach}Let $X$ be a Banach space, $K \subset X$ a compact set, and $T_n \in L(X)$ such that $\sup_{h\in K}|(I-T_n)h|\xrightarrow{n\to \infty}0.$  Then the set $\bigcup_{n \in\mathbb N} T_n(K)={\{T_n(x),x \in K, n \in \mathbb N\}}$ is relatively compact, i.e.~its closure $\tilde K := \overline{\bigcup_{n=1}^\infty T_n(K)}$ is compact.  Since $T_nx\to x$ for all $x \in K$, we have $\tilde K \supset K$.
\end{lemma}
{\begin{proof} We verify the hypotheses of
\cite[Lemma 1.1]{anselone-palmer}.
Set $ \mathcal F:=\{T_n|_K:n\in\mathbb N\}.$  First, writing, 
$ |T_nx-T_ny|
    \leq |T_nx-x|+|x-y|+|y-T_ny|$ with the first and third terms {converging} to zero uniformly on $K$,
we notice that $\mathcal F$ is equicontinuous. Second, since  $T_nx\to x$, for every $x\in K$, we have that
\(\overline{\{T_nx:n\in\mathbb N\}}\)
is compact. Then \cite[Lemma 1.1]{anselone-palmer} implies that
$\mathcal F(K)$
has compact closure.
\end{proof}}
\subsection{Encoder, decoder, and partial sum operators}
\begin{definition}\label{def:encoder_decoder}Let  $(X,|\cdot|)$ be a Banach space. 
\begin{itemize}
    \item An encoder (or coordinate) operator on $X$ is ${\mathcal E}^X \in \mathcal L( X , \R^{N})$, for some $N\in \mathbb N$, i.e. there exists  $\left\{a^{X}_i\right\}_{i=1}^{N} \subset X^*$  such that
${\mathcal E}^Xx =  (\langle a^{X}_i,x\rangle_{X^*,X}  )_{i=1}^{N} $.
\item  A decoder (or embedding) operator into $X$ is ${\mathcal D}^{X} \in \mathcal L( \R^{N} , X)$ for some $N \in \mathbb N$, i.e. there is $\left\{e_i^{X}\right\}_{i=1}^{N} \subset X$ {such that}
${{\mathcal D}}^{X}(x_i)_{i=1}^{N} = \sum_{i=1}^{N} x_i e^{X}_i$.
\end{itemize}
 Then we can define the partial sum operator
 \begin{align}\label{eq:S=DE}
{\mathcal S}^X \in \mathcal L(X),\quad {\mathcal S}^Xx:={{\mathcal D}}^{X}{\mathcal E}^Xx=\sum_{i=1}^{N} \langle a^{X}_i,x\rangle_{X^*,X} \  e^{X}_i.
    \end{align}
\end{definition}

\subsubsection{Approximation Properties, Frames, and Bases in Banach spaces}
We introduce  Approximation Properties, Frames, and Bases in Banach spaces \cite{casazza2001,casazza2008}. These are general properties satisfied by most Banach spaces used in applications. However,  \cite{enflo1973counterexample} provided a counterexample of a separable Banach space failing to have the Approximation Property.

Let  $(X,|\cdot|)$ be a Banach space. 
\begin{assumption}[{Approximation Property \cite[Definition 2.2 and Theorem 2.5]{casazza2001}}]\label{ass:approx_pro}
Assume that $X$ has the  Approximation Property (AP); that is,  for all compact $K\subset X$, $\epsilon>0$ there exists $T_{K,\epsilon}\in  \mathcal F(X)$ such that $\sup_{x\in K}|(I-T_{K,\epsilon})x|\leq \epsilon$. Then, defining the directed set $A:=\{(K,\epsilon):K\subset X \textit{  compact, }\epsilon>0\}$, with $\preceq$ given by $\alpha_1\preceq \alpha_2$, $\alpha_1=(K_1,\epsilon_1),\alpha_2=(K_2,\epsilon_2)$, if $K_1\subset K_2, \epsilon_1\geq \epsilon_2$, and ${\mathcal S}_\alpha^X:=T_{K,\epsilon}$, $\alpha=(K,\epsilon)$, we have  a Moore–Smith sequence $\left\{{\mathcal S}^X_\alpha\right\}_{\alpha \in A }\subset \mathcal L(X)$ such that
\begin{enumerate}[(i)]
\item ${\mathcal S}^X_\alpha \in \mathcal F(X)$, for all $\alpha$, i.e.  there exists $N_\alpha\in \mathbb N$, $\left\{e_i^{X,\alpha}\right\}_{i=1}^{N_\alpha} \subset X$,  $\left\{a^{X,\alpha}_i\right\}_{i=1}^{N_\alpha} \subset X^*$ such that
\begin{equation}\label{eq:sum_op_X_AP}
  {\mathcal S}^X_\alpha x=\sum_{i=1}^{N_\alpha} \langle a^{X,\alpha}_i,x\rangle_{X^*,X} \  e^{X,\alpha}_i, \quad  \forall x \in X.
\end{equation}
    \item  ${\mathcal S}_\alpha^X\xrightarrow{co} I$, i.e.   $\sup_{x\in K}|(I-{\mathcal S}_\alpha^X)x|_X\rightarrow 0,$ for all $ K\subset X $ compact.
\end{enumerate}

In this case, we
define  
    \begin{align}
        &{\mathcal {E}}_\alpha^{X} \in \mathcal L( X , \R^{N_\alpha}) && {\mathcal {E}}_\alpha^{X}(x) =  (\langle a^{X,\alpha}_i,x\rangle_{X^*,X}  )_{i=1}^{N_\alpha} \label{eq:coordinate_op_X},\\
        &{\mathcal {D}}_\alpha^{X} \in \mathcal L( \R^{N_\alpha} , X) && {\mathcal D}_\alpha^{X}((x_i)_{i=1}^{N_\alpha}) = \sum_{i=1}^{N_\alpha} x_i e^{X,\alpha}_i,\label{eq:embedding_op_X}
    \end{align}
    which are  encoder and decoder operators, respectively (see Definition \ref{def:encoder_decoder}) such that ${\mathcal S}^X_\alpha={\mathcal {D}}_\alpha^{X}\circ \mathcal {E}_\alpha^{X}$.
    When there is no ambiguity on the Banach space $X$, we may denote them by ${\mathcal {E}}_\alpha,{\mathcal {D}}_\alpha,{\mathcal S}_\alpha$. 
\end{assumption}
\begin{remark}\label{rem:AP_implies_mathcalK=mathcalA}
If either $X^*$ or $Y$ {has} AP, then $\mathcal A(X,Y)=\mathcal K(X,Y)$ (see \cite[Theorems 16.36, 16.35]{fabian-etal}).
 \end{remark}
A strictly stronger property is the Bounded Approximation Property \cite{casazza2001}.
\begin{assumption}[{Bounded Approximation Property \cite[Definition 3.1, Theorem 3.3, Corollary 3.4]{casazza2001}}]\label{ass:bounded_approx_pro}
Assume that $X$ has the Bounded Approximation Property (BAP); that {is,} it has AP and there exists $\lambda>0$ such that $ \| {\mathcal S}_\alpha^{X}\|_{\mathcal L(X)}\leq \lambda$, for all $\alpha\in A $.  If $X$ is separable,  this requirement is equivalent to the existence of a sequence $\left\{{\mathcal S}^X_d\right\}_{d\in \mathbb N}\subset \mathcal L(X)$ such that
\begin{enumerate}[(i)]
\item ${\mathcal S}^X_d \in \mathcal F(X)$, for all $d$, i.e. there exists $N_d\in \mathbb N$, $\left\{e_i^{X,d}\right\}_{i=1}^{N_d} \subset X$,  $\left\{a^{X,d}_i\right\}_{i=1}^{N_d} \subset X^*$ such that
\begin{equation}\label{eq:sum_op_X}
 {\mathcal S}^X_d\in\mathcal L(X),\quad    {\mathcal S}^X_d x=\sum_{i=1}^{N_d} \langle a^{X,d}_i,x\rangle_{X^*,X} \  e^{X,d}_i, \quad  \forall x \in X.
\end{equation}
    \item  
${\mathcal S}_d^X\xrightarrow{co} I,$ i.e. $  \sup_{x\in K}|(I-{\mathcal S}_d^X)x|_X\rightarrow 0,$ for all $ K\subset X$ compact, or, equivalently, by Lemma \ref{lemma1:uniformcompactsfromstrongandequibounded}, ${\mathcal S}^X_d\xrightarrow{s} I$.
    \item The following bound holds
\begin{equation}\label{eq:SdX_bdd_lambda}
   \| {\mathcal S}_d^{X}\|_{\mathcal L(X)}\leq \lambda, \quad \forall d\in \mathbb N.
\end{equation}
\end{enumerate}
Additionally, it can be assumed ${\mathcal S}^X_m {\mathcal S}^X_d={\mathcal S}^X_d$, for all $d<m$.
The coordinate (or encoder) ${\mathcal E}^X_d$ and embedding (or decoder) ${\mathcal D}^X_d$ operators are defined as in \eqref{eq:coordinate_op_X}, \eqref{eq:embedding_op_X} for $\alpha=d$. 
\end{assumption}
A {property stronger than the} Bounded Approximation Property is the existence of a Schauder frame \cite{casazza2008}. 
\begin{definition}[Schauder frame]\label{def:schauder_frame}
    A Schauder frame, or quasibasis, for $X$ is a sequence $\left\{\left(e_i^X, a_i^X\right)\right\}_{i \in \mathbb{N}}\subset X \times X^*$ such that 
$
x=\sum_{
i=1}^{\infty} \langle a_i^X,x\rangle_{X^*,X} \  e_i^X,$ for all $ x \in X.
$

In this case,  the coordinate, embedding, and partial sum\footnote{also known as frame operator} operators (notice the difference with respect to \eqref{eq:coordinate_op_X}, \eqref{eq:embedding_op_X}, \eqref{eq:sum_op_X}), become, respectively,
    ${E}_d^{X} \in \mathcal L( X , \R^d), $ $ {E}_d^{X}(x) =  (\langle a_i^X,x\rangle_{X^*,X}  )_{i=1}^d$
        ${\mathcal D}_d^{X} \in \mathcal L( \R^d , X), $ $ {{D}}_d^{X}((x_i)_{i=1}^d) = \sum_{i=1}^d x_i e_i^X$ 
          ${\mathcal S}_d^X \in \mathcal L( X  ), $ $ {\mathcal S}_d^Xx = \sum_{i=1}^d \langle a_i^X,x\rangle_{X^*,X} \    e_i^X.$
   As usual, when there is no ambiguity on the Banach space $X$, we will denote them simply by ${E}_d,{\mathcal D}_d,{\mathcal S}_d$. 
\end{definition}
A particular, but very important case is the Schauder basis \cite{singer1981}.
\begin{definition}[Schauder basis]\label{def:schauder_basis}
    A Schauder basis of $X$ is a sequence of linearly independent elements $\{e_i^X\}\subset  X$ such that for all $x \in X$ there exists a unique sequence $\left\{b_i\right\} \subset \mathbb R$ such that $
x=\sum_{i=1}^{\infty} b_i \  e_i^X.
$ It follows that there exists a unique sequence $\{a_i^X\}\subset X^*$, defined by $\langle a_i^X,x\rangle_{X^*,X}=b_i$, such
    that $
x=\sum_{i=1}^{\infty} \langle a_i^X,x\rangle_{X^*,X} \  e_i^X.
$
We define  ${E}_d^{X} \in \mathcal L( X , \R^d), $ 
        ${\mathcal D}_d^{X} \in \mathcal L( \R^d , X), $ 
          ${\mathcal S}_d^X \in \mathcal L( X  )$  as in Definition \ref{def:schauder_frame}.
\end{definition}

\subsection{Hilbert spaces}\label{subsec:H} If   $(X,\langle \cdot ,\cdot \rangle_X)$ is a Hilbert space, we denote the induced norm by   $ |\cdot|_X:=(\langle \cdot ,\cdot \rangle_X)^{1/2}.$ If $X$ is separable   with orthonormal basis $\{e_i^X\}$, we have  
 $x=\sum_{i\in \mathbb N}\langle x,e_i^X \rangle e_i, $ $ |x|^2=\sum_{i\in \mathbb N}|\langle x,e_i^X \rangle|^2.$
An orthonormal basis on a Hilbert space $X$ is a Schauder basis with $a^X_i=e^X_i$ for all $i\in\mathbb N$, i.e. $\langle a^X_i,x\rangle_{X^*,X}=\langle x,e^X_i\rangle_X$. In particular, the Bounded Approximation Property of $X$ (Assumption \ref{ass:bounded_approx_pro}) is satisfied.
Throughout the whole paper we identify $H$ with its dual $H^*.$ 
We denote by $S(X)$ the space of self-adjoint operators in $\mathcal{L}(X)$. 

Let   $(Y,\langle \cdot ,\cdot \rangle_Y)$ be another  Hilbert space.
If $X,Y$ are separable, we denote by $\mathcal{L}_1(X,Y)$ the Banach space of trace-class operators endowed with the norm  
$\|L\|_{\mathcal{L}_1(X,Y)}:=\inf \left \{\sum_{i \in \mathbb N} |a_i|_X |b_i|_Y : Lx=\sum_{i \in \mathbb N} b_i  \langle a_i,x \rangle, \ a_i \in X, b_i \in  Y , \forall i \in \mathbb N \right \} ,  $ $ L \in \mathcal{L}_1(X,Y).$ 
We denote by  $\mathcal{L}_2(X,Y)$ the Hilbert space of Hilbert-Schmidt operators from $X$ to $Y$. The scalar product in $\mathcal{L}_2(X,Y)$ and its induced norm are respectively given by
$\langle L,T \rangle_{\mathcal{L}_2(X,Y)}:= \sum_{k \in \mathbb N} \langle Le_k,Te_k \rangle_Y, $ $   \|L\|_{\mathcal{L}_2(X,Y)}:=\left (\sum_{k \in \mathbb N} |L  e_k|_Y^2 \right)^{1/2} $ $  L,T \in \mathcal{L}_2(X,Y),$
where  $\{ e_k\}$  is any orthonormal basis of $X$. We will consider the space of $i$-linear Hilbert-Schmidt operators $\mathcal L_2^i\left(X,Y\right)$, with inner product 
$
\langle L, T\rangle_{\mathcal L_2^i\left(X, Y\right)}=\sum_{j_1, \ldots, j_i\in \mathbb N}\left\langle L\left(e_{j_1}, \ldots, e_{j_i}\right), T\left(e_{j_1}, \ldots, e_{j_i}\right)\right\rangle_{Y},
$
and induced Hilbert-Schmidt norm $\|\cdot\|_{\mathcal{L}_2^i\left(X, Y\right)}$. The canonical isometric identification gives $
\mathcal L_2\bigl(X,\mathcal L_2^{i-1}(X,Y)\bigr)=\mathcal L_2^i(X,Y).
$
  When $Y=X$ we simply write $\mathcal{L}_1(X)=\mathcal{L}_1(X,X)$, $\mathcal{L}_2(X)=\mathcal{L}_2(X,X)$, $\mathcal{L}^i_2(X)=\mathcal{L}_2^i(X,X)$.
  We denote by $\mathcal{L}^+_1(X)\subset \mathcal{L}_1(X)$ the {cone} of positive operators. If $T\in \mathcal{L}_1(X)$ we can define its trace by  $\operatorname{Tr}(T):=\sum_{k \in \mathbb{N}}\left\langle T e_k, e_k\right\rangle_X$, where $\operatorname{Tr}(T)$ is independent of the   orthonormal basis $\{ e_k\}$.

\begin{definition}[Orthonormal basis]\label{def:orthonormal_basis} Let $(X,\langle \cdot ,\cdot \rangle_X)$ be a separable Hilbert space, with orthonormal basis $\{e_i^X\}$.
In the Hilbert space case, the coordinate, embedding, and partial sum operators, respectively, become
   ${E}_d^{X} \in \mathcal L( X , \R^d)$, $ {E}_d^{X}(x) =  (\inprod{x}{e_i^X})_{i=1}^d $, $
        {\mathcal D}_d^{X} \in \mathcal L(\R^d , X) $, $ {\mathcal D}_d^{X}((x_i)_{i=1}^d) = \sum_{i=1}^d x_i e^X_i,$ $
          {\mathcal S}_d^X \in \mathcal L( X) $, $ {\mathcal S}_d^Xx= \sum_{i=1}^d \langle x,e^X_i \rangle e_i^X,$
    where ${\mathcal S}_d^Xx$ is now the orthogonal projection over $\operatorname{span}(e_1,\ldots,e_d)$.  In the Hilbert space case, \eqref{eq:SdX_bdd_lambda} improves to
        $\|{\mathcal S}_d^X\|_{\mathcal L(X)}= 1, $ $ d\geq 1.$
  As usual,  when there is no ambiguity on the Hilbert space $X$, we will denote them simply by ${\mathcal S}_d,{E}_d,{\mathcal D}_d$.
\end{definition}

\subsection{Measures and spaces $L^p$}\label{subsubsec:borel_meas}
\paragraph{Spaces $L^p$.} Let $(\Omega, \mathcal F, \mu)$ be a measure space with positive measure $\mu$ and  $p \in[1,+\infty)$. We denote by $L^p(\Omega; \mathbb R^N;\mu)$  the standard space of measurable functions $ f: \Omega \rightarrow \mathbb R^N$ such that $\int_\Omega |f|^p d \mu<+\infty$. When $\mu$ is the Lebesgue measure on $\Omega\subset \mathbb R^n$ we suppress the dependence, as well as when $N=1$.
$L^p(\Omega;\mathbb R^N;\mu)$ is a Banach space with norm
$\|f\|_{L^p}:=\left[\int_{\Omega}|f|^p d \mu\right]^{1 / p}.$
In particular,  $L^2(\Omega; \mathbb R^N;\mu)$ is a Hilbert space with scalar product
$\langle f,g \rangle_{L^2}=\int_\Omega \langle f,g \rangle d \mu.
$ When $\Omega=\mathbb{N}$ endowed with the counting measure $\#$, we denote $\ell^p=L^p\left(\Omega ; \mathbb{R};\#\right)$.

\paragraph{Borel measures.} Let $X,Y$ be Banach spaces. We say that a (positive) Borel measure $\mu$ on $X$ is finite, or, simply, $\mu$ is a finite measure on $X$ (dropping Borel for brevity), if $\mu(X)<\infty$; if $\mu(X)=1$ it is called  a Borel probability measure. {A measure $\mu$ on $X$ is called Radon if it is locally finite and inner regular, i.e.\ $\mu(B)=\sup\{\mu(C):C\subset B,\ C\text{ compact}\}$ for every $B\in\mathcal B(X)$; if $\mu$ is finite, this is equivalent to tightness,  i.e. for every $\epsilon>0$ there exists a compact set $C\subset X$ such that $\mu(X\setminus C)<\epsilon$. Every finite measure on a separable Banach space is Radon.} Throughout the paper, \textit{all finite measures are assumed to be Borel and Radon, even when not specified.}

 Let $T:X\to Y$ be Borel measurable. We define the pushforward measure $T_\#\mu$ on $Y$ by
$
(T_\#\mu)(A):=\mu\big(T^{-1}(A)\big),$ $ A\in \mathcal B(Y).$
Then, for every measurable function $\varphi:Y\to\mathbb R$,
$
\int_Y \varphi(y)\,d(T_\#\mu)(y)
=
\int_X \varphi\big(T(x)\big)\,d\mu(x).$ We define $\operatorname{supp}(\mu):=\left\{x \in X : \forall N_x :\left(x \in N_x \Rightarrow \mu\left(N_x\right)>0\right)\right\}$, where $N_x$ denotes a neighborhood of $x\in X.$

Let $\lambda,\mu$ be finite measures on $X$. We say that $\lambda$ is absolutely continuous with respect to $\mu$ and write
$
\lambda \ll \mu
$
if $\lambda(A)=0$ for every $A \subset X$ Borel for which $\mu(A)=0$. If $
\lambda \ll \mu
$,  by the Radon-Nikodym  Theorem, there exists  a unique (up to a.e. equivalence relation) $\frac{d\lambda}{d\mu}: X \rightarrow[0, \infty)$ Borel-measurable and $\int_X \frac{d\lambda}{d\mu} d\mu<\infty$, called Radon-Nikodym derivative, such that 
$
\lambda(A)=\int_A \frac{d\lambda}{d\mu} d \mu,
$
for any  $A  \subset X$ Borel.

\begin{notation}\label{not:marginals}
Let $k\in \mathbb N$. Let $\mu$ be a finite measure on $X\times X^k$.  We denote by $\mu^{0:i}$ the marginal of $\mu$ over $X\times X^i$, $0\leq i\leq k$. We denote $\mu^0:=\mu^{0:0}$; that is, given $(x,h_1,\ldots,h_k)\in X \times X^k$ we interpret $x$ as the $0$-th variable and $h_i$ as the $i$-th variable (appearing in $D^iF(x)$ for $F:X\to Y$ smooth; see, e.g., Definition \ref{def:Gâteaux}). Coherently, if  $\eta$ is  a finite measure on $X^k$, we denote by $\eta^{1:i}$ the marginal of $\eta$ over $X^i$, $i\leq k$. Moreover, for $q\geq 0,p\geq 1$, we set
\begin{equation}\label{eq:moments_mu}
\begin{aligned}
   & \|\mu\|_{k,q,p} =\int_{X\times X^k} (1+|x|_X^q)\prod_{j=1}^k(1+|h_j|_X^p)d\mu(x,h_1,\ldots,h_k).\\
   & \|\eta\|_{k,p} =\int_{X^k} \prod_{j=1}^k(1+|h_j|_X^p)d\eta(h_1,\ldots,h_k).
\end{aligned}
\end{equation}

Define measures $\hat \mu^{0:i}$, $0\leq i\leq k$ on $X\times X^i$  by (dropping dependence on $p$)
$$
\hat \mu^{0:i}(U):=\int_{U} \prod_{r=1}^i\left|h^r\right|_X^p d \mu^{0:i}\left(x, h^1, \ldots, h^i\right), \quad U \subset X\times X^i \text { Borel  measurable, }
$$
If $\int_{X\times X^i} \prod_{r=1}^i\left|h^r\right|_X^p d \mu^{0:i}\left(x, h^1, \ldots, h^i\right)<\infty$, then, by \cite[Corollary 3.6]{folland1999}, $\hat\mu^{0:i} $  is a finite measure such that $\hat\mu^{0:i} \ll \mu^{0:i}$ with $d\hat\mu^{0:i}/d\mu^{0:i}(x,h^1,\dots,h^i)=\prod_{r=1}^i\left|h^r\right|_X^p$.  
We also define measures $\bar \mu^{0:i}$, $0\leq i\leq k$ on $X$  by
$$
\bar \mu^{0:i}(A):=\hat\mu^{0:i}(A\times X^i)=\int_{A \times X^i} \prod_{r=1}^i\left|h^r\right|_X^p d \mu^{0:i}\left(x, h^1, \ldots, h^i\right), \quad A \subset X \text { Borel  measurable, }
$$
If $\int_{X\times X^i} \prod_{r=1}^i\left|h^r\right|_X^p d \mu^{0:i}\left(x, h^1, \ldots, h^i\right)<\infty$, then $\bar \mu^{0:i}$  is finite and $\bar \mu^{0:i}\ll \mu^{0}$ (since, using $\hat\mu^{0:i} \ll \mu^{0:i}$, if $0=\mu^0(A)=\mu^{0:i}(A\times X^i)$, we have $\bar \mu^{0:i}(A)=\hat \mu^{0:i}(A\times X^i)=0$).
Note that
\begin{equation}\label{eq:mu_marginals_bound}
\begin{aligned}
    &\|\mu^{0:i}\|_{i,q,p} \leq \|\mu\|_{k,q,p},\quad \forall 0\leq i\leq k,\quad \textit{Therefore, }\|\mu^{0}\|_{0,q,p} \leq \|\mu\|_{k,q,p}.\\
    & \|\eta^{1:i}\|_{i,p} \leq \|\eta\|_{k,p},\quad \forall 1\leq i\leq k
\end{aligned}
\end{equation}

Let ${\mathcal E}^X \in \mathcal L( X , \R^{N}),$ ${\mathcal D}^{X} \in \mathcal L( \R^{N} , X)$ for some $ N\in \mathbb N$ be encoders and decoders (see Definition \ref{def:encoder_decoder}).
Given a finite measure $\nu$ on $X$ we denote the pushforward measure on $\mathbb R^N$ by $$\nu^{{\mathcal E}^X}:=\left({\mathcal E}^X\right)_{\#}\nu.$$
When $\nu=\mu^0$ we set $\mu^{0,{\mathcal E}^X}:=(\mu^0)^{{\mathcal E}^X}$.
\end{notation}
Note that, by \eqref{eq:S=DE}, we have (as measures on $X$)
\begin{equation}\label{eq:push_back_D}
\begin{aligned}
   ({\mathcal S}^X)_{\#} \nu&=({\mathcal D}^X \circ {\mathcal E}^X)_{\#} \nu= \nu\{({\mathcal D}^X \circ {\mathcal E}^X)^{-1}(\cdot)\}=\nu\{( {\mathcal E}^X)^{-1}[({\mathcal D}^X)^{-1}(\cdot)]\}\\
   &=\nu^{{\mathcal E}^X}\{({\mathcal D}^X)^{-1}(\cdot)\}=\left({\mathcal D}^X\right)_{\#} \nu^{{\mathcal E}^X}.
\end{aligned}
\end{equation}
\subsection{$\gamma$-radonifying operators}
{Let $H$ be a Hilbert space and $Y$ a Banach space. Identify the algebraic
tensor product $H\otimes Y$ with the finite-rank operators from $H$ to $Y$ by
$(h\otimes y)h':=[h',h]_H y$. If $h_1,\ldots,h_N$ are orthonormal in $H$,
$y_1,\ldots,y_N\in Y$, and $(g_n)$ are independent standard Gaussian
variables, set
$\left\|\sum_{n=1}^N h_n\otimes y_n\right\|_{\gamma(H,Y)}^2
  :=\E\left|\sum_{n=1}^N g_ny_n\right|_Y^2.$
This norm is independent of the orthonormal representation. Following
\cite[Definition~3.7]{van-neerven2010}, $\gamma(H,Y)$ is the completion
of $H\otimes Y$ in this norm. The canonical inclusion into $\mathcal L(H,Y)$
extends to an injective contraction $\gamma(H,Y)\hookrightarrow\mathcal L(H,Y)$;
after this identification, $T\in\mathcal L(H,Y)$ is called $\gamma$-radonifying
if $T\in\gamma(H,Y)$ \cite[Definition~3.13]{van-neerven2010}.
For $i\geq1$, define the iterated spaces by
$\gamma^1(H,Y):=\gamma(H,Y)$ and, recursively,
$\gamma^i(H,Y)
  :=\gamma\bigl(H,\gamma^{i-1}(H,Y)\bigr),
 i\geq2.$ If $H,Y$ are {separable Hilbert spaces}, we have $\mathcal L_2^i(H,Y)=\gamma^i(H,Y)$.}
\section{$L^p_\xi$-norms on spaces of bounded multilinear operators}\label{subsec:norm_Lp-multilinear}
{Let $X,Y$ be Banach spaces. We introduce the following integral norm/seminorm on $\mathcal L^i(X;Y)$, {i.e. the Bochner $L^p$-seminorm restricted to bounded multilinear operators,} for a measure $\xi$ on $X^i$ {satisfying the mixed moment condition below}. {To the best of our knowledge, this family of $L^p_\xi$-norms has not previously been introduced in this form on spaces of bounded multilinear operators between Banach spaces.} {Thanks} to linearity and continuity, if the support of the measure contains the {$i$-fold} cartesian product of the unit sphere of $X$, it is automatically a norm on $\mathcal L^i(X,Y)$ without the need to take a quotient. }
\begin{definition}[Norm $ \|\cdot\|_{L^p(X^i,Y,\xi)}$ on $\mathcal L^i(X,Y)$]\label{def:Lpnorm}
Let {$i\geq1$, $1\leq p<\infty$, and let} $\xi$ be a {positive Borel} measure on $X^i$  such that $\int_{X^i} \prod_{t=1}^i |h^t|_X^pd\xi(h^1,\ldots,h^i)<\infty$. Define
 \begin{equation}\label{norm:Lp_operators}
\|L\|_{L^p(X^i,Y,\xi)}:=\left(\int_{X^i} |L(h^1,...,h^i)|^p_Yd\xi(h^1,\ldots,h^i)\right)^{1/p},\quad L\in \mathcal L^i(X,Y).
\end{equation}
This is a seminorm on $\mathcal L^i(X;Y)$, in general. 
However, for instance, if {$\operatorname{supp}\xi\supset {(S^{X,1})^i}$, then (see below) it is a norm on $\mathcal L^i(X;Y)$ (without the need of taking the quotient with respect to equivalence relations  as in standard {$L^p$}-spaces)}.
It is weaker than $\|\cdot\|_{\mathcal L^i(X,Y)}$, since
\begin{equation}\label{eq:normr_bounded_opnorm}
    \|L\|_{L^p(X^i,Y,\xi)}\leq \|L\|_{\mathcal L^i(X,Y)}\left(\int_{X^i} \prod_{t=1}^i |h^t|_X^pd\xi(h^1,\ldots,h^i)\right)^{1/p}.
\end{equation}
When $X,Y$ are finite dimensional, whenever $ \|\cdot\|_{L^p(X^i,Y,\xi)}$ is a norm, it is equivalent to  $\|\cdot\|_{\mathcal L^i(X,Y)}$.
\end{definition}
{To check that $ \|\cdot\|_{L^p(X^i,Y,\xi)}$ is a norm when $\operatorname{supp}\xi\supset {(S^{X,1})^i}$, it is enough to check that $\|L\|_{L^p(X^i,Y,\xi)}=0$ implies $L=0$. Indeed,  assume $\|L\|_{L^p(X^i,Y,\xi)}{=}0$. 
Denote $g(h^1,\ldots,h^i):=|L(h^1,\ldots,h^i)|_Y^p$. Since $L\in \mathcal L^i(X,Y)$, $g$ is continuous and nonnegative. 
If $g(h^1_0,\ldots,h^i_0)>0$  at some $(h^1_0,\ldots,h^i_0) \in \operatorname{supp} \xi$, continuity     would give an open neighborhood $U\subset X^i$ such that $(h^1_0,\ldots,h^i_0) \in U$ and $g>g\left(h^1_0,\ldots,h^i_0\right) / 2$ on $U$; by the definition of support, $\xi(U)>0$, from which $\int_{X^i} g d \xi\geq \int_U g d \xi \geq \xi(U) g\left(h^1_0,\ldots,h^i_0\right) / 2$, contradicting $\int_{X^i} g d \xi=0$. Hence $g=0$ on $\operatorname{supp}\xi\supset {(S^{X,1})^i}$. Then $\|L\|_{\mathcal L^i(X,Y)}=0$, i.e. $L=0$.} 
{\begin{proposition}\label{prop:Lp_norm_separable}
Let $X,Y\neq\{0\}$, $i\geq1$ and $1\leq p<\infty$. Then there exists a 
finite Radon measure $\xi$ on $X^i$, with $\int_{X^i}\prod_{t=1}^i|h^t|_X^p\,d\xi<\infty$, such    that $\|\cdot\|_{L^p(X^i,Y,\xi)}$ is a norm on $\mathcal L^i(X,Y)$ if and only if $X$ is separable. Consequently, if $X$ is nonseparable, then $\|\cdot\|_{L^p(X^i,Y,\xi)}$ is a proper seminorm for every finite Radon measure $\xi$ on $X^i$ satisfying $\int_{X^i}\prod_{t=1}^i|h^t|_X^p\,d\xi<\infty$.
\end{proposition}
\begin{proof}
($\Leftarrow$) Let $X$ be separable and let $\{x_n\}_{n\in\mathbb N}$ be dense in $S^{X,1}$. Set $\nu:=\sum_{n\in\mathbb N}2^{-n}\delta_{x_n}$ and $\xi:=\otimes_{j=1}^i\nu$, a Radon probability measure on $X^i$ with $\int_{X^i}\prod_{t=1}^i|h^t|_X^p\,d\xi=1$. If $\|L\|_{L^p(X^i,Y,\xi)}=0$, then $L(x_{n_1},\ldots,x_{n_i})=0$ for all $n_1,\ldots,n_i\in\mathbb N$, since $\xi(\{(x_{n_1},\ldots,x_{n_i})\})\geq\prod_{j=1}^i2^{-n_j}>0$. By continuity, $L=0$ on $(S^{X,1})^i$, hence, by multilinearity, on $X^i$. (Alternatively, one can take, e.g., $\xi:=\otimes_{j=1}^i\gamma$, for a Gaussian measure $\gamma$ on $X$ with full support \cite{bogachev1998gaussian}.)

($\Rightarrow$) Let $\xi$ be a finite Radon measure on $X^i$ such that $\|\cdot\|_{L^p(X^i,Y,\xi)}$ is a norm.  Being    Radon, $\xi$ is concentrated    on a $\sigma$-compact, hence separable, set $K\subset X^i$ (i.e. $\xi(X^i\setminus K)=0$, with $K=\bigcup_{n\in\mathbb N}K_n$ for some compact sets $K_n\subset X^i$; by tightness, one can take $K_n$ such that $\xi(X^i\setminus K_n)<1/n$),     so that $$Z:=\overline{\operatorname{span}\,\pi_1(K)}, \qquad  \textit{where } \pi_1(h^1,\ldots,h^i):=h^1,$$ is a separable       closed subspace of $X$. We show that 
$$X=Z,$$
which, in turn, implies   that $X$ is separable,     concluding the proof.
 Indeed, assume by contradiction    that $Z\neq X$. Then, by the Hahn--Banach theorem, there exists $x^*\in X^*\setminus\{0\}$ vanishing on $Z$; fixing $y_0\in Y\setminus\{0\}$, the map\footnote{the product is needed: for $i\geq2$, the map $(h^1,\ldots,h^i)\mapsto x^*(h^1)\,y_0$ is constant in $h^2,\ldots,h^i$, hence not $i$-linear} $L(h^1,\ldots,h^i):=\big(\prod_{t=1}^i x^*(h^t)\big)\,y_0$ is    nonzero    and vanishes on $K$ (since, for $(h^1,\ldots,h^i)\in K$ we have $h^1\in\pi_1(K)\subset Z$, hence $x^*(h^1)=0$, which implies that $L$ vanishes on $K$) so that $$\|L\|_{L^p(X^i,Y,\xi)}^p=\int_K|L(h^1,\ldots,h^i)|_Y^p\,d\xi=0,$$ contradicting     the fact that $\|\cdot\|_{L^p(X^i,Y,\xi)}$ is a norm. Hence $X=Z$.
\end{proof}}

For $p=2$ and $Y$ {a} separable Hilbert space, we have the following characterization:
\begin{example}[Hilbert--Schmidt norms on Reproducing Kernel Hilbert Spaces]\label{ex:HS_norms_RKHS}
{Let $X$ be a separable Banach space, let $Y$ be a
 separable
Hilbert space, and let $\nu$ be a (Borel) probability measure on
$X$ (hence, Radon) with $\int_X| h|_X^2\,d\nu(h)<\infty$. For simplicity, assume that $\nu$ is centered (i.e. with zero mean).
Denote the   covariance kernel by $K_\nu :X^*\times X^* \to \mathbb R$ and the covariance operator by $R_\nu:X^*\to X$,  characterized 
by  \cite[Section III.2.1]{vakhania-tarialadze}
\[
\langle y^*,R_\nu x^*\rangle_{X^*,X}=\int_X\langle x^*,h\rangle_{X^*,X}\langle y^*,h\rangle_{X^*,X}\,d\nu(h)=K_\nu(x^*,y^*).
\]
We define the associated  Reproducing Kernel Hilbert Space (RKHS) $X_\nu$ (equivalently identified with the Hilbert space associated with $R_\nu$ via $R_\nu x^* \longleftrightarrow K_\nu\left(x^*, \cdot\right)$)  as the completion of $R_\nu(X^*)$ under the scalar product $\langle R_\nu y^*,R_\nu x^*\rangle_{X_\nu}:=\langle y^*,R_\nu x^*\rangle_{X^*,X}=K_\nu(x^*,y^*)$; the reproducing identity holds $\langle z,R_\nu x^*\rangle_{X_\nu}=\langle z,K_\nu(x^*,\cdot)\rangle_{X_\nu}=\langle x^*,z\rangle_{X^*,X}$, for all $z\in X_\nu,x^*\in X^*$; we denote by
$\iota_\nu:X_\nu\hookrightarrow X$ the canonical embedding (see \cite[Section III.1.2-3]{vakhania-tarialadze}). Then
$R_\nu=\iota_\nu\iota_\nu^*$. Since $X$ is separable, $X_\nu$ is separable; let {us choose} an
orthonormal basis $\{e_n^\nu\}$. Then, denoting $\otimes_{j=1}^i\nu$  the product measure and, denoting the restriction of $L\in \mathcal L^i(X,Y)$ to $X^i_\nu$ by  $L_\nu :=L\circ(\iota_\nu,\ldots,\iota_\nu)$, we have
\begin{align}\label{eq:norm2=HS}
    \|L\|_{L^2(X^i,Y,\otimes_{j=1}^i\nu)}=\left\| L_\nu\right\|_{\mathcal L_2^i(X_\nu,Y)},\quad \forall  L\in \mathcal L^i(X,Y).
\end{align} }

{Indeed, let $i=1$. For
$L\in\mathcal{L}(X;Y)$, set $\lambda=\lambda_L:=L_{\#}\nu$. Then $\lambda$ is
{a} centered finite measure on $Y$ with finite second moment. By \cite[Section III.2.1 (e)]{vakhania-tarialadze}, we have $R_\lambda=LR_\nu L^*=L\iota_\nu(L\iota_\nu)^*$; by \cite[Section III.2.3, Remark]{vakhania-tarialadze} (using Riesz identification $Y^*\sim Y$), 
\begin{align*}
 \|L\|_{L^2( X, Y, \nu)}^2=\int_X| Lh|_Y^2\,d\nu(h)
 &=\int_Y| y|_Y^2\,d\lambda(y)
 =\operatorname{Tr}_Y(R_\lambda)=\operatorname{Tr}_Y\!\bigl((L\iota_\nu)(L\iota_\nu)^*\bigr) =\lVert (L_\nu)^*\rVert_{\mathcal{L}_2(Y,X_\nu)}^2
 =\lVert L_\nu\rVert_{\mathcal{L}_2(X_\nu,Y)}^2.
\end{align*}
For $i\geq2$, assume that the assertion holds for $i-1$; define
$A:X\to \mathcal L_2^{i-1}(X_\nu,Y)$, 
$
A(h^1)(h^2,\ldots,h^i)
:=L(h^1,\iota_\nu h^2,\ldots,\iota_\nu h^i).
$
By the induction hypothesis,
\[
\begin{aligned}
\|A(h^1)\|_{\mathcal L_2^{i-1}(X_\nu,Y)}^2
&=\int_{X^{i-1}}
   |L(h^1,h^2,\ldots,h^i)|_Y^2\,
   d\nu(h^2)\cdots d\nu(h^i)\leq
\|L\|_{\mathcal L^i(X,Y)}^2 |h^1|_X^2
\left(\int_X |h|_X^2\,d\nu(h)\right)^{i-1}.
\end{aligned}
\]
Thus $A$ is well defined and bounded. Moreover, under the
isometric identification
$
\mathcal L_2\bigl(X_\nu,\mathcal L_2^{i-1}(X_\nu,Y)\bigr)=\mathcal L_2^i(X_\nu,Y),
$
we have $\left(A \iota_\nu\right)\left(h^1\right)\left(h^2, \ldots, h^i\right)=L_\nu\left(h^1, \ldots, h^i\right)$. Thus, applying the
case $i=1$ to $A$, with Hilbert target
$\mathcal L_2^{i-1}(X_\nu,Y)$, and the induction hypothesis, we have
\[
\begin{aligned}
\|L\|_{L^2(X^i,Y,\otimes_{j=1}^i\nu)}^2&=\int_{X}\int_{X^{i-1}}|L(h^1,\ldots,h^i)|_Y^2\,
   d\nu(h^1)d\nu(h^2)\cdots d\nu(h^i)\\
   &=\int_X
   \|A(h^1)\|_{\mathcal L_2^{i-1}(X_\nu,Y)}^2\,d\nu(h^1)=\|A\iota_\nu\|_{
   \mathcal L_2(X_\nu,\mathcal L_2^{i-1}(X_\nu,Y))}^2=\|L_\nu\|_{\mathcal L_2^i(X_\nu,Y)}^2.
\end{aligned}
\]}
\end{example}
\begin{example}[Gaussian measures: $\gamma$-radonifying operator norms]\label{ex:gaussian_measures_radonifying}Let $X$, $Y$ be Banach spaces, and let $\gamma$ be a centered Gaussian
Radon measure on $X$ (for the definition, see \cite{bogachev1998gaussian}). Its RKHS $X_\gamma$ coincides with the
Cameron--Martin space (therefore,  we will use the same notation $X_\gamma$ to indicate both) and is separable. Then, for all $L\in\mathcal L^i(X;Y)$, we have $L_\gamma\in\gamma^i(X_\gamma,Y)$ {(recall the iterated spaces $\gamma^i(H,Y)=\gamma(H,\gamma^{i-1}(H,Y))$ from Appendix \ref{app:basics})} with\footnote{If $Y$ is Hilbert, then
$\gamma^i(X_\gamma,Y)=\mathcal L_2^i(X_\gamma,Y)$ isometrically, so this
reduces to~\eqref{eq:norm2=HS}}
\[
  \|L\|_{L^2(X^i,Y,\otimes_{j=1}^i\gamma)}
  =\|L_\gamma\|_{\gamma^i(X_\gamma,Y)}.
\]

Indeed, let  $i=1$; define the  
 Gaussian measure $\lambda:=L_\#\gamma$ on $Y$.  Its covariance is
$R_\lambda=LR_\gamma L^*=TT^*$, where we have used that
$R_\gamma=\iota_\gamma\iota_\gamma^*$ and we have set $T:=L\iota_\gamma$. 
Since $\iota_\gamma \in \gamma(X_\gamma,X)$ (see \cite[Proposition 8.6]{van-neerven2010}), by \cite[Theorem 6.2]{van-neerven2010}, we have $T\in\gamma(X_\gamma,Y)$.
Thus, by  \cite[Theorem~7.4]{van-neerven2010}, we have 
\[
  \|T\|_{\gamma(X_\gamma,Y)}^2
  =\int_Y|y|_Y^2d\lambda(y)
  =\int_X|Lh|_Y^2d\gamma(h)
  =\|L\|_{L^2(X,Y,\gamma)}^2.
\]
For $i\geq2$, we proceed by induction similarly to Example \ref{ex:HS_norms_RKHS} to prove the claim.
\end{example}
\section{Spaces of differentiable maps on Banach spaces}\label{app:smooth}
Let $(X,|\cdot|_X),(Y,|\cdot|_Y)$   be Banach spaces. In this work, typically, we denote maps between infinite-dimensional Banach spaces by capital letters, and maps between finite-dimensional spaces by lowercase letters.
   \subsection{Spaces of continuously differentiable maps}
\begin{definition}[Gâteaux derivative]\label{def:Gâteaux}
    We define the  Gâteaux derivative of $F: X \rightarrow Y$ at $x\in X$, denoted by $DF(x)h$, provided that the following limit exists
$$DF(x)h:=\lim_{t\to 0} \frac{F(x+th)-F(x)}{t},$$
for all $h\in X$ and, in addition, $DF(x)\in \mathcal L(X,Y)$. Recursively, for $k\ge2$, we define the $k$-th order Gâteaux derivative
of $F$ at $x$, denoted by $D^k F(x)(h_1,\ldots,h_k)$,  provided that the following limit exists  
\begin{equation}\label{eq:gaeaux}
  D^k F(x)(h_1,\ldots,h_k)
:=
\lim_{t\to0}
\frac{
D^{k-1}F(x+t h_k)(h_1,\ldots,h_{k-1})
-
D^{k-1}F(x)(h_1,\ldots,h_{k-1})
}{t}  ,
\end{equation}for all $h_1,\dots,h_k\in X$, and, in addition, $D^kF(x)\in \mathcal L^k(X,Y)$.
   \end{definition}
We introduce the notion of continuous Bastiani differentiability which will be  natural for the universal approximation. We refer to \cite{bastiani1964,keller2006,glockner2002infinite,schmeding2022} for a complete introduction to the so-called Bastiani calculus, giving, e.g. chain rule, Schwarz theorem, etc. 
   \begin{definition}[$C^k_B(X;Y)$]\label{def:CkB}
    We say that $F: X \rightarrow Y$ {is continuously Bastiani differentiable  on $X$} if $F\in C^0( X ; Y)$; for any $x\in X$, $F$ has Gâteaux derivative $DF$;  and  $DF(\cdot)(\cdot) \in C^0(X \times X; Y)$ {(or, equivalently, by \cite{keller2006}, since $X,Y$ are Banach spaces, if $DF(\cdot) \in C^0(X; \mathcal L(X,Y)^{so})$, or equivalently, if  $DF(\cdot) \in C^0(X; \mathcal L(X,Y)^{co})$).}
    We denote by $C^1_B(X;Y)$ the space of these maps. 

Let $k\ge2$. We say that $F$ {is $k$-times continuously Bastiani differentiable} if $F\in C^{k-1}_B(X;Y)$; for any $x\in X$, $F$ has $k$-th order Gâteaux derivative $D^kF(x)$; and
$D^kF (\cdot)(\cdot,\ldots,\cdot)\in C^0(X \times X^k; Y)$ {(or, equivalently, by \cite{keller2006}, if $D^kF(\cdot) \in C^0(X; \mathcal L^k(X,Y)^{so})$, or equivalently, if  $D^kF(\cdot)\in C^0(X; \mathcal L^k(X,Y)^{co})$).}
We denote by $C^k_B(X;Y)$ the space of these maps. If $Y=\mathbb R$ we write $C^k_B(X):=C^k_B(X;\mathbb R)$. 
   \end{definition}
The class $C^k_B$ provides calculus {under minimal regularity and} beyond Banach spaces, i.e. in general locally convex spaces, and it is widely used, e.g in infinite-dimensional differential geometry, e.g. see \cite{bastiani1964,glockner2002infinite,schmeding2022}.   However,  it is not a popular choice in functional analysis in Banach or Hilbert spaces, where the following stronger notion of continuous Fréchet differentiability is the standard choice. 
   \begin{definition}[${C^k_F}(X;Y)$]
    {We say that $F: X \rightarrow Y$ is continuously Fréchet differentiable on $X$ if $F\in C^0( X ; Y)$; for any $x\in X$, $F$ has Gâteaux derivative $DF(x);$  $DF\in C^0(X;\mathcal L(X,Y))$. 
    We denote by ${C^1_F}(X;Y)$ the space {of} these maps.}

Let  $k\ge2$. We say that $F$ is continuously $k$-times Fréchet differentiable on $ X$ if $F\in {C^{k-1}_F}(X;Y)$; for any $x\in X$, $F$ has $k$-th order Gâteaux derivative $D^kF(x)$; $D^kF\in C^0(X;\mathcal L^k(X,Y))$. We denote by ${C^k_F}(X;Y)$ the space of these maps. If $Y=\mathbb R$ we write ${C^k_F}(X)={C^k_F}(X;\mathbb R)$.
   \end{definition}

\begin{remark}\label{rem:Bastiani}
\begin{enumerate}
\item[(i)] We have ${C^k_F}(X; Y)\subset C_B^k(X; Y)$. However, {vice versa} is false, making the inclusion strict, in general, e.g. see \cite{walther2021} (see also \cite{krylov-rockner-zabczyk}).
\item[(ii)] Although continuous Bastiani differentiability is weaker than continuous Fréchet differentiability, we have the following partial converse result \cite[Lemma A.3.3]{walter2014}: if $F\in C^{k}_B(X;Y)$ then $F\in {C^{k-1}_F}(X;Y)$. In particular $C^{\infty}_B(X;Y)={C^{\infty}_F}(X;Y)$. Moreover, if $F\in C^{k}_B(X;Y)$ and $D^kF\in C^0(X;\mathcal L^k(X,Y))$, then $F\in {C^{k}_F}(X;Y);$
    \item[(iii)] If $X,Y$ are finite dimensional, then  $C^k_B(X; Y)={C^k_F}(X; Y)$. {More generally, the same equality holds whenever $X$ is finite dimensional and $Y$ is an arbitrary Banach space.}
\end{enumerate}
\end{remark}
\begin{remark}\label{rem:borel_measu_opnorms_CB}For all $F\in C^{k}_B(X;Y)$, we have that
$x\mapsto \|D^iF(x)\|_{\mathcal L^i(X,Y)}$ is lower-semicontinuous, therefore Borel measurable. Indeed, $x\mapsto \|D^iF(x)\|_{\mathcal L^i(X,Y)}$ is the supremum of the continuous map $(x,h^1,\ldots,h^i)\mapsto |D^iF(x)(h_1,\ldots,h_i)|_Y$ over $|h_1|_X,\ldots,|h_i|_X\leq 1$; the claim follows. 
\end{remark}
\begin{notation}\label{not:notation_D0f}
Throughout the paper, when $i=0$ we will use the convention   $\mathcal L^0(X,Y):=Y$, with
$
\|y\|_{\mathcal L^0(X,Y)}:=|y|_Y, y\in Y,$
and we use the convention $D^0F:=F$. Thus
$
\|D^0F(x)\|_{\mathcal L^0(X,Y)}=|F(x)|_Y.
$
\end{notation}
\subsection{Cylindrical maps on Banach spaces}
Let $(X,|\cdot|_X),(Y,|\cdot|_Y)$ be Banach spaces. Recall Definition \ref{def:encoder_decoder}.
\begin{definition}[Cylindrical map]\label{def:cyl_map}
 A cylindrical map  is a map $F^{{\mathcal E}^X,{\mathcal D}^Y}:X\to Y$ of the form 
 \begin{equation}\label{eq:cylindrical_map}
   F^{{\mathcal E}^X,{\mathcal D}^Y}(x)={\mathcal D}^Yf^{N,N'}({\mathcal E}^Xx)=\sum_{j=1}^{N'} f^{N,N'}_j\left[ (\langle a^{X}_i,x\rangle_{X^*,X}  )_{i=1}^{N}\right] e^{Y}_j,  
 \end{equation}
for some encoder operator on $X$, ${\mathcal E}^X \in \mathcal L( X , \R^{N}),{\mathcal E}^X(x) =  (\langle a^{X}_i,x\rangle_{X^*,X}  )_{i=1}^{N}$, for some $N\in \mathbb N$; for some decoder operator into $Y$, ${\mathcal D^{Y}} \in \mathcal L( \R^{N'} , Y),  {\mathcal D}^{Y}((x_j)_{j=1}^{N'}) = \sum_{j=1}^{N'} x_j e^{Y}_j$, for some $N'\in \mathbb N$; function
   $f^{N,N'}:\mathbb R^N\to \mathbb R^{N'}$, where $f^{N,N'}_j , j\leq N'$ denote its components.
We denote by $Cyl(X,Y):=\{F^{{\mathcal E}^X,{\mathcal D}^Y}:X\to Y \textit{ of the form }\eqref{eq:cylindrical_map} \textit{ for some } N,N' \in \mathbb N,{\mathcal E}^X\in \mathcal L(X,\mathbb R^N),{\mathcal D}^Y\in \mathcal L(\mathbb R^{N'},Y)\}$ the space of cylindrical maps from $X$ to $Y$. We define $C^k_{Cyl}(X,Y):=\{F^{{\mathcal E}^X,{\mathcal D}^Y}\in Cyl(X,Y):f^{N,N'}\in C^k(\mathbb R^{N}, \mathbb R^{N'})\}$. We set $C_{b,Cyl}^\infty(X,Y):=\left\{F^{{\mathcal E}^X,{\mathcal D}^Y}\in C^{\infty}_{Cyl}(X,Y):f^{N,N'}\in C_b^{\infty}\left(\mathbb{R}^N;\mathbb R^{N'}\right)\right\}$ (the subscript ``b'' stands for bounded derivatives of all orders). 
    \end{definition}

If $F\in C^k_{Cyl}(X,Y)$, then $F\in {C^k_F}(X,Y)$ with, by chain rule, for every $1 \leq i \leq k$, 
\begin{align}\label{eq:derivativeDrv_dh^theta_dY_Banach_cyl}
    D^i F^{{\mathcal E}^X,{\mathcal D}^Y}(x)\left[h^1, \ldots, h^i\right]&={\mathcal D}^Y\left(D^i f\left({\mathcal E}^X x\right)\left[{\mathcal E}^X h^1, \ldots, {\mathcal E}^X h^i\right]\right).
\end{align}
    \subsection{Cylindrical approximations on Banach spaces}\label{subsec:cylindrical_approx}
   Let $(X,|\cdot|_X),(Y,|\cdot|_Y)$ be Banach spaces  with the Approximation Property, i.e. we are given $\left\{{\mathcal S}^X_\alpha\right\}_{\alpha \in A }\subset \mathcal L(X)$, $\left\{{\mathcal S}^Y_\beta\right\}_{\beta \in B }\subset \mathcal L(Y)$ satisfying Assumption \ref{ass:approx_pro}. We use the notations defined there.
    \begin{definition}[Cylindrical approximations]\label{def:cyl_approx}
        Let $k\in \mathbb N$. Let $F \in C^k_B(X;Y)$.
 For   $\alpha \in A, \beta \in B $,  define the finite dimensional function   
$${f^{N_\alpha ,N_\beta}} : \mathbb R^{N_\alpha} \to \mathbb R^{N_\beta}, \quad f^{N_\alpha ,N_\beta}(y):={\mathcal E}^Y _{\beta}\left[ F( {\mathcal D}^X _{\alpha}(y))\right],$$
which is $C^k(\mathbb R^{N_\alpha} ; \mathbb R^{N_\beta})$  (recall Remark \ref{rem:Bastiani}) with derivatives computed by chain rule,
for every $1 \leq i \leq k$, as
\begin{equation}\label{eq:derivativeDf_Banach}
D^i f^{N_\alpha ,N_\beta}(y)\left[\eta^1, \ldots, \eta^i\right]= {\mathcal E}^Y _{\beta}\left(D^i F\left({\mathcal D}^X _{\alpha} y\right)\left[{\mathcal D}^X _{\alpha}\eta^1, \ldots, {\mathcal D}^X _{\alpha} \eta^i\right]\right) ,\quad \forall y , \eta^1, \ldots, \eta^i \in \mathbb{R}^{N_\alpha},
\end{equation}
We define the cylindrical approximation of $F$ as (recall also \eqref{eq:S=DE})
\begin{equation}\label{eq:cyl_approx_is_cyl}
    F_{\alpha ,\beta} \colon X \to Y, \quad   F_{\alpha ,\beta}(x):={\mathcal S}^Y_{\beta} F({\mathcal S}^X_{\alpha}x)={\mathcal D}^Y_\beta  {\mathcal E}^Y _{\beta}\left[ F( {\mathcal D}^X _{\alpha}{\mathcal E}^X_\alpha x)\right]={\mathcal D}^Y_\beta f^{N_\alpha ,N_\beta}({\mathcal E}^X_\alpha x)=: F^{{\mathcal E}^X_\alpha ,{\mathcal D}^Y_\beta }(x),
\end{equation}
so that $F_{\alpha ,\beta} \in Cyl(X,Y)$.
Then, $F_{\alpha ,\beta} \in {C^k_F}(X;Y)$  with derivatives computed by chain rule (and \eqref{eq:S=DE}), for every $1\leq i \leq k$,  as
\begin{equation}\label{eq:derivativeDrv_dh_dY_Banach}
\begin{aligned}
    D^i F_{\alpha ,\beta}(x)\left[h^1, \ldots, h^i\right]\equiv D^i F^{{\mathcal E}^X_\alpha ,{\mathcal D}^Y_\beta }(x)\left[h^1, \ldots, h^i\right]&={\mathcal S}^Y_{\beta}\left(D^i F\left({\mathcal S}^X_{\alpha} x\right)\left[{\mathcal S}^X_{\alpha} h^1, \ldots, {\mathcal S}^X_{\alpha} h^i\right]\right)\\
    &={\mathcal D}^Y _{\beta}\left(D^i f^{N_\alpha ,N_\beta}\left({\mathcal E}^X _{\alpha} x\right)\left[{\mathcal E}^X _{\alpha} h^1, \ldots, {\mathcal E}^X _{\alpha} h^i\right]\right)
\end{aligned}
\end{equation}

When $X,Y$ are separable Banach spaces satisfying the BAP (Assumption \ref{ass:bounded_approx_pro}), we  replace $\left \{{\mathcal S}^X_\alpha\right\}_{\alpha \in A }$, $\left\{{\mathcal S}^Y_\beta\right\}_{\beta \in B }$, by sequences $\left \{{\mathcal S}^X_d\right\}_{d \in \mathbb N }, \left \{{\mathcal S}^Y_m\right\}_{m \in \mathbb N }$  and we use notations there.
    \end{definition}

\subsection{Differentiable bump functions on Banach spaces}\label{sec:bump}
For some of our results we will need smooth bump functions on Banach spaces. We refer to \cite{fry-mcmanus2002} for a complete survey.

Let $(X,|\cdot|)$ be a Banach space.
\begin{definition}[Bump function]
We say that $X$ admits a $C^k_B$ bump function if there exists $b \in C^k_B(X)$ with non-empty and bounded support.
\end{definition}
To be consistent with our analysis in the paper (where we require $C^k_B$), with respect to the \cite[Definition 3]{fry-mcmanus2002}, we relax the condition by requiring the bump function $b$ to be $C^k_B$ (and not Fréchet ${C^k_F}$). 

We {now discuss} differentiable norms $|\cdot|$ on Banach spaces. Notice that norms cannot be differentiable at $x=0$, but only away from $x=0$. However, only {differentiability} away from $x=0$ is required in the subsequent analysis.  In the following, we denote $g_r:X\to \mathbb R, g_r(x):=|x|_X^r$, $r\geq 1$.
Note that  $\|Dg_1(x)\|_{\mathcal L(X,\mathbb R)}\leq 1$ for any $x\in X$ where $D|\cdot|_X$ is well defined. Define $\psi \in C_0^{\infty}\left(\mathbb{R}\right)$ by
$$
 \psi(y):=\frac{\chi(2-|y|)}{\chi(|y|-1)+\chi(2-|y|)} ,\quad \textit{where } \chi(y):=\left\{\begin{array}{ccc}
e^{-1 / y^2} & \text { if } & y>0 \\
0 & \text { if } & y \leq 0.
\end{array} \right.
$$
Then  $\psi(y)=1$ if $|y| \leq 1$,  $\psi(y)=0$ if $|y| \geq 2$. Moreover,    there exists $C>0$ such that
  $\left|D^j \psi(y)\right| \leq C$  for all $y\in \mathbb R$, $0 \leq j \leq k$. Defining $$\psi_{\eta}(y)=\psi(\eta y)$$ for $\eta>0$, we have $\psi_{\eta}(y)=1$ if $|y| \leq 1 / \eta$, $\operatorname{supp}\psi_\eta \subset \{0\leq |y| \leq 2/\eta\}$ and  
  \begin{equation}\label{eq:bound_Di_psi_eta}
      \left|D^i \psi_{\eta}(y)\right|=\eta^i | D^i \psi(\eta y)|  \leq C \eta^i,\quad \forall \eta >0.
  \end{equation}
This construction leads us to the following proposition. 
\begin{proposition}\label{prop:bump_rescaled}
    Assume that $X$ admits an equivalent norm, still denoted by $|\cdot|_X$, for which $g_r \in C_B^k(X\setminus\{0\})$ for some $r \geq 1$. Then 
    \begin{enumerate}
        \item[(i)] $X$ admits a $C^k_B$ bump function $b\in C^k_B(X)$, given by $b(x):=\psi(g_r(x))=\psi(|x|_X^r)$ with $\operatorname{supp}b \subset \{|x|_X\leq 2^{1/r}\}$.
        \item[(ii)]  Moreover, assume  that there exists $C>0$ such that 
        \begin{equation}\label{eq:Dig_bound}
            \|D^ig_r(x)\|_{\mathcal L^i(X,\mathbb R)}\leq C |x|_X^{r-i},\quad \forall 0\neq x\in X,0\leq i\leq k.
        \end{equation}
   Then, denoting the rescaled bump function 
   $$b_{\eta}(x):=b(\eta x)=\psi(g_r(\eta x))=\psi(\eta^r|x|_X^r)=\psi_{\eta^r}(|x|_X^r)=\psi_{\eta^r}(g_r(x)),\quad \eta>0,\quad \operatorname{supp}b_\eta \subset \{|x|_X\leq 2^{1/r}/\eta\},$$ we have
    $b_{\eta}(x)\equiv 1$ if $|x|_X \leq 1 / \eta$ and there exists $C>0$ such that  $$\left\|D^i b_{\eta}(x)\right\|_{\mathcal L^i(X,\mathbb R)} \leq C \eta^{i}\leq C,\quad \forall x\in X, \forall \eta \leq 1,0\leq i\leq k.$$ 
    \end{enumerate}
    \end{proposition}
    \begin{proof}
    (i) This is immediate by chain rule.  (ii) 
    We have  $b_{\eta}(x)\equiv 0$ if $|x|_X^r \geq 2 / {\eta}^r$, or, equivalently,  if $|x|_X \geq 2^{1/r} / {\eta}$. Similarly,  $b_{\eta}(x)\equiv 1$ if $|x|_X \leq 1 / {\eta}$.  For $x \neq 0$, by chain rule (\cite[Chapter 3, Theorem 2.1 and its Proof]{bastiani1964}), for every $i \leq k$, we have
$$
D^i b_\eta(x)\left(h_1, \ldots, h_i\right)=\sum_{\pi \in \Pi_i}D^{|\pi|}\psi_{\eta^r}(g_r(x))\prod_{A \in \pi}D^{|A|} g_r(x)\left(\left(h_j\right)_{j \in A}\right),
$$
where $\Pi_i$ denotes the set of partitions of $\{1, \ldots, i\}$. 
Then $D^i b_\eta(x)\equiv 0,$ for all $x\in  \{|x|\leq 1/\eta \} \cup\{ |x|_X\geq 2^{1/r}/\eta\}$, $i \geq 1$.  Moreover, there exists $C>0$ such that for all $\eta \leq 1$, $x\in \{1/\eta\leq |x|\leq 2^{1/r}/\eta\}$,
\begin{align*}
    \|D^i b_\eta(x)\|_{\mathcal L^i(X,\mathbb R)}&\leq \sum_{\pi \in \Pi_i}|D^{|\pi|}{\psi_{\eta^r}}(g_r(x))|\prod_{A \in \pi}\|D^{|A|} g_r(x)\|_{\mathcal L^{|A|}(X,\mathbb R)}\leq C\sum_{\pi \in \Pi_i}\eta^{r|\pi|}\prod_{A \in \pi}|x|_X^{r-|A|}\\
    &\leq C\sum_{\pi \in \Pi_i}\eta^{r|\pi|} |x|_X^{\sum_{A\in \pi}(r-|A|)}\leq C\sum_{\pi \in \Pi_i}\eta^{r|\pi|} |x|_X^{|\pi|r-i}\leq C \eta^{i}\leq C,
\end{align*}
where we have used \eqref{eq:bound_Di_psi_eta}, \eqref{eq:Dig_bound},  and the fact that, for $\pi\in\Pi_i$, we have  $\sum_{A\in \pi}|A|=i$.
    \end{proof}
\begin{example}\label{ex:norm_Fréchet}
\begin{enumerate}[(i)]
    \item If $X=H$ is a Hilbert space, then, for the  norm $|\cdot|_H$ induced by the scalar product, we have $g_2:=|\cdot|_H^2\in {C^\infty_F}(X)$ with $\langle Dg_2(x),h^1\rangle_H=2\langle x,h^1\rangle_H,  D^2g_2(x)(h^1,h^2)=2\langle h^1,h^2\rangle_H,D^ig_2(x)\equiv 0$, $i\geq 3$ and Proposition  \ref{prop:bump_rescaled} applies (in particular, this implies that for the standard finite-dimensional case $X=\mathbb R^N$ Proposition  \ref{prop:bump_rescaled} applies).
    \item Let $X=L^r(\Omega, \mathcal{F}, \nu)$ or $X=\ell^r$, with $1<r<\infty$ (and usual norms $|\cdot|_X$), and set $g_r(x)=|x|_X^r$. By \cite[Proposition 5]{bonic-frampton}, if $r$ is an even integer then $g_r \in {C^{\infty}_F}(X)$; if $r$ is an odd integer then $g_r \in {C^{r-1}_F}(X)$ and $D^{r-1} g_r$ is Lipschitz; if $r \notin \mathbb{N}$, then $g_r \in {C^{\lfloor r\rfloor}_F}(X)$ and $D^{\lfloor r\rfloor} g_r$ is $(r-\lfloor r\rfloor)$-Hölder continuous. Moreover, by \cite[Proof of Proposition 5]{bonic-frampton}, for every admissible $i$, $\left\|D^i g_r(x)\right\|_{\mathcal L^i(X,\mathbb R)} \leq C|x|_X^{r-i},$ $ x \neq 0 .$
Hence Proposition  \ref{prop:bump_rescaled} applies with $r$, for the corresponding range of $k$.
\item For more examples of Banach spaces admitting differentiable bump functions, we refer to \cite{fry-mcmanus2002}.
\end{enumerate}
\end{example}
\subsection{Compact-open topologies on spaces of differentiable maps} \label{subsec:compact-open}
In the following recall Notation \ref{not:notation_D0f}. We introduce the standard ${C^k_F}$ compact-open topology.
\begin{definition}[${C^k_F}(X,Y)^{co}$]\label{def:compact_open_C2}
The usual compact-open topology on ${C^k_F}(X,Y)$ is generated by  the directed family of seminorms $ {\mathbf {\mathcal {\mathbf P}}}^{k,co}=\{{\mathbf p}_{K}=\max( p^i_{K}:0\leq i\leq k): K \subset X \textit{ compact} \}$, where $ p^i_{K}(F)=\sup _{x \in K}\|D^iF(x)\|_{\mathcal L^i(X,Y)}$. {Equivalently, this is the initial topology with respect to the maps $D^i: {C^k_F}(X,Y)\to C^0(X,\mathcal L^i(X,Y))^{co},$ $ F\mapsto  {D^iF}$, $\forall {0\leq i\leq k}$.}
We denote by ${C^k_F}(X,Y)^{co}$ this topological space.
\end{definition}
The previous ${C^k_F}$ compact-open topology will be too strong for our purposes (and the class ${C^k_F}$ too restrictive). Therefore, in a similar way to \cite{schmeding2022} (see also \cite{llavona1986}), we introduce the  compact-open topology of $C_B^k(X,Y)$ as follows.
\begin{definition}[$C_B^k(X,Y)^{co}$]\label{def:compact_open_C2B}
We define the compact-open topology of $C_B^k(X,Y)$ as the locally convex Hausdorff topology generated  by one of the following equivalent directed {families} of separating seminorms  
\begin{itemize}
    \item $\mathcal P_{B}^{k,co,3}=\{p_{\mathcal  K_0,\ldots,\mathcal K_k}^3=\max(p^{i,3}_{\mathcal  K_i}:0\leq i\leq k): \mathcal K_i \subset X^{i+1} \textit{ compact, } 0\leq i\leq k\},$\\
where $p^{i,3}_{\mathcal  K_i}(F):=\sup_{(x,h^1,...,h^i) \in \mathcal  K_i}  |D^iF(x)(h^1,...,h^i)|_Y;$
\item ${\mathcal P}_{B}^{k,co,2}=\{ p^2_{\mathcal  K}=\max( p^{i,2}_{\mathcal  K}:0\leq i\leq k): \mathcal K \subset X^{k+1} \textit{ compact} \}$,\\
where $ p^{i,2}_{\mathcal  K}(F):=\sup_{x \in \pi^0(\mathcal  K),h^1 \in \pi^1(\mathcal  K)...,h^i\in \pi^i(\mathcal  K)}  |D^iF(x)(h^1,...,h^i)|_Y,$ with $\pi^i:X^{k+1}\to X, $ $\pi^i(y^0,\ldots,y^k)=y^i$, $0\leq i\leq k$;
\item ${\mathcal P}_{B}^{k,co,1}=\{ p_{K,K'}^1=\max( p^{i,1}_{K,K'}:0\leq i\leq k): K,K' \subset X \textit{ compact} \}$,\\
where $ p^{i,1}_{K,K'}(F):=\sup_{x \in K,h^1 \in K'...,h^i\in K'}  |D^iF(x)(h^1,...,h^i)|_Y.$
\item ${\mathcal P}_{B}^{k,co,0}=\{ p^0_{K}=\max( p^{i,0}_{ K}:0\leq i\leq k): K \subset X \textit{ compact} \},$\\
where $ p^{i,0}_{K}(F):=\sup_{x \in K,h^1 \in K,...,h^i\in K}  |D^iF(x)(h^1,...,h^i)|_Y.$
\end{itemize}
{Equivalently, this is the initial topology with respect to the maps $D^i: C^k_B(X,Y)\to C^0(X\times X^i ; Y)^{co}, F\mapsto  D^iF$, $\forall 0\leq i\leq k$ or, equivalently, with respect to the maps $D^i: C^k_B(X,Y)\to C^0(X,\mathcal L^i(X,Y)^{co})^{co},$ $ F\mapsto  {D^iF}$, $ 0\leq i\leq k$.
We denote by $C_B^k(X,Y)^{co}$ this topological space. We endow ${C^\infty_B}(X,Y)$ with the locally convex Hausdorff topology generated by
\(
\mathcal P_B^{\infty,co,1}:=\bigcup_{k\in\mathbb N}\mathcal P_B^{k,co,1}.
\)
We denote the resulting topological space by $C_B^\infty(X,Y)^{co}$.}
\end{definition}
To show equivalence (recall Definition \ref{def:eq_seminorms}) of ${\mathcal P}_{B}^{k,co,i}$, $0\leq i\leq  3,$ we first notice that (identifying two seminorms when they define the same function on $C^k_B(X,Y)$) 
${\mathcal P}_{B}^{k,co,i}\supset {\mathcal P}_{B}^{k,co,i-1}$, for all $1\leq i\leq 3$. Then, it is enough to show that given a seminorm $p^3_{\mathcal  K_0,\ldots,\mathcal K_k}$ we can find $K\subset X$ compact and $C>0$ such that $p^3_{\mathcal  K_0,\ldots,\mathcal K_k}(F)\leq C p^0_{K}(F)$ for all $F\in C^k_B(X,Y)$. To this purpose we can choose $C=1$ and $K:=\bigcup_{0\leq j\leq i \leq k}\pi^j(\mathcal K_i)\subset X$, which is compact as {a} finite union of compact sets. Indeed, this choice implies that 
$\mathcal K_i\subset K^{i+1}$, for all $i\leq k$, from which the claim follows.
\begin{remark}\label{rem:finite_dimensional_CO}
\begin{enumerate}
    \item The restriction of the topological space $C_B^k(X,Y)^{co}$ to ${C^k_F}(X,Y)\subset C_B^k(X,Y)$  is weaker than ${C^k_F}(X,Y)^{co}$.
    \item However, if $X$ and $Y$ are finite dimensional, then (recall also Remark \ref{rem:Bastiani}) ${C^k_F}(X,Y)^{co}=C_B^k(X,Y)^{co}$. 
\end{enumerate}
   
\end{remark}
\subsection{Continuous Bastiani differentiability of Nemytskii operators}\label{subsec:nemistkii-bastiani}
It is well known that Nemytskii/superposition operators are not Fréchet differentiable in general (e.g. for $X=Y=L^2$, except for trivial cases) \cite[Theorem 3.12]{appell-zabrejko}, raising a severe limitation of the concept of Fréchet differentiability in infinite-dimensions. However, they are {continuously} Bastiani differentiable, e.g., on $L^2(\mathbb T^d)$ (see also, e.g., \cite{krylov-rockner-zabczyk}).  This is directly relevant for Operator Learning as these  appear in Neural Operators (e.g. see Remark \ref{rem:comparison_UA_compact_yao}, point 2).
{\begin{lemma}[Bastiani differentiability of Nemytskii/superposition operators on $L^2(\mathbb T^d)$]\label{lemma:Nemytskii_bastiani}
\label{lem:bastiani-nemytskii-L2}
Let $X:=L^2(\mathbb T^d)$ and
$Y:=L^2(\mathbb T^d)$. Assume that
$\sigma\in C^1(\mathbb R)$ with  
$
|D\sigma(z)|\leq M<\infty .$
Define  $
N:X\to Y, N(x)(\xi):=\sigma(x(\xi)).
$
Then $N$  is Lipschitz continuous, $N\in C^1_B(X,Y)$ with
$
DN(u;h)(\xi)=D\sigma(u(\xi))h(\xi).$
\end{lemma}
\begin{proof}
Since $D\sigma$ is bounded, $\sigma$ has at most linear growth.
Then $N(u)\in Y$ for
every $u\in X$. By assumption,
$
|\sigma(z_1)-\sigma(z_2)|\le M|z_1-z_2|,
$
so that
$
|N(x)-N(y)|_Y\le M|x-y|_X .
$
We now show that the Gâteaux derivative is $
DN(x)h=D\sigma(x)h 
$. Indeed, for $t\neq0$, by Jensen's
inequality,
\begin{align*}
    \left|
\frac{N(x+th)-N(x)}{t}-D\sigma(x)h
\right|_Y^2&=\int_{\mathbb T^d}\left|
\frac{\sigma(x(\xi)+th(\xi))-\sigma(x(\xi))}{t}-{D\sigma(x(\xi))}h(\xi)\right|^2 d\xi\\
&=\int_{\mathbb T^d}\left|\int_0^1[
D\sigma(x(\xi)+\theta th(\xi))-D\sigma(x(\xi))]h(\xi)d\theta \right|^2 d\xi\\
&
\le
\int_{\mathbb T^d}\int_0^1
|D\sigma(x(\xi)+\theta th(\xi))-D\sigma(x(\xi))|^2 |h(\xi)|^2d\theta d\xi,
\end{align*}
where the right-hand-side goes to zero by the dominated convergence theorem as $t\to 0$. The claim follows.
It remains to prove $N\in C^1_B(X,Y)$. First fix $h\in X$ and let $x_n\to x\in X$, so convergence in measure holds. By the dominated convergence theorem
$$
|(DN(x_n)-DN(x))h|_Y^2
=|(D\sigma(x_n)-D\sigma(x))h|_Y^2
=
\int_{\mathbb T^d}
|D\sigma(x_n(\xi))-D\sigma(x(\xi))|^2 |h(\xi)|^2\,d\xi\xrightarrow{n\to \infty}0 .
$$
Thus $x\mapsto DN(x)$ is continuous in the strong operator topology. Then $C^1_B$ follows by Definition \ref{def:CkB}.
\end{proof}}
\section{Weighted Bastiani--Sobolev spaces on Banach spaces}\label{subsec:Sobolev} 
In this section, we construct Bastiani--Sobolev spaces on Banach spaces, designed ad-hoc for the universal approximation. We will use notations from Subsection \ref{subsubsec:borel_meas}.
\subsection{Motivation of the  construction}\label{subsec:motivation_sobolev}
{
Let $X,Y$ be Banach spaces. We seek Sobolev spaces in which $k$-times differentiable operators $F:X\to Y$ can be universally approximated by EDAs (Definition \ref{def:EDN}), generalizing \cite[Theorem 4]{hornik1991approximation} (see Subsections \ref{subsec:sobolev_norms}, \ref{subsec:sobolev_sp}).

\paragraph{Remarks on the finite-dimensional case.} The paper \cite{hornik1991approximation} considers $f\in C^k(\mathbb R^N)$ with finite weighted Sobolev norm (up to quotient) $\big[\sum_{|\alpha|\leq k}\int_{\mathbb R^N}|D^\alpha f|^p\,d\nu\big]^{1/p}$, for a finite input measure $\nu$. This space need not be complete and, unlike under the Lebesgue measure (Meyers--Serrin theorem \cite{adams2003sobolev}), its completion need not be well behaved: without integration by parts formulas, derivative operators need not be closable \cite[Example 2.6.1]{bogachev2010}, i.e.
\begin{equation}\label{eq:closure_fail_derivatives_measure}
    \textit{if }f_n\to f \textit{ and } D^\alpha f_n\to g_\alpha \in L^p(\mathbb R^N,\nu),\textit{ the limit }g_\alpha\textit{ need not be determined uniquely by }f.
\end{equation}
Classical properties require conditions on $\nu$ (e.g. suitable weights \cite{kufner-opic}; see also \cite{alberti-bate-marhcese}), unnecessary for density results: the above normed space (still called weighted Sobolev space in \cite{hornik1991approximation}) suffices \textit{for any finite input measure $\nu$}.

\paragraph{Choice of the weighted Sobolev norm.} Generalizing this norm directly, via operator norms, gives $\|\cdot\|_{\mathcal W^{k,p}_{\nu}(X,Y)}$ in \eqref{eq:sobolev_norm}; but, in general, UA via EDAs fails under it, already for linear maps (Theorem \ref{th:counterexample}). Instead, inspired by Theorem \ref{th:UAT_K_Banach}, we consider $F\in C^k_B(X,Y)$ and view $D^iF$ as the maps $(x,h^1,\ldots,h^i)\mapsto D^iF(x)(h^1,\ldots,h^i)$, integrating also in the directions. This leads to the (weaker) \textbf{weighted Bastiani--Sobolev norms} $\|\cdot\|_{\mathcal W^{k,p}_{B,\mu}(X,Y)}$ in \eqref{eq:sobolev_seminorm}, for a finite measure $\mu$ on $X^{k+1}$, and $\|\cdot\|_{\mathcal W^{k,p,r}_{B,\nu,\eta}(X,Y)}$, for a finite measure $\nu$ on $X$ and a probability measure $\eta$ on $X^k$. They generalize the norms of Gaussian Sobolev spaces \cite[Section 5.2]{bogachev1998gaussian} (used in DIOL, e.g. \cite{luo-Roseberry-chen-Ghattas}): {if $X$ is separable, $Y$ is a separable Hilbert space, and $\gamma$ is a centered Gaussian measure on $X$}, $\nu:=\gamma$, $\eta:=\otimes_{i=1}^k\gamma$, $\mu:=\otimes_{i=0}^k\gamma$, then $\|\cdot\|_{\mathcal W^{k,p}_{\gamma,\mathcal L_2(X_\gamma,Y)}(X,Y)}$ in \eqref{eq:def_sobolev_norm_mu0_Lp_bogachev_HS} equals $\|\cdot\|_{\mathcal W^{k,p,2}_{B,\gamma,\eta}(X,Y)}$, and $\|\cdot\|_{\mathcal W^{k,2}_{B,\mu}(X,Y)}$ if $p=2$.

\paragraph{Construction of the weighted Bastiani--Sobolev spaces.} To apply \cite{hornik1991approximation} (via Assumption \ref{ass:finite_dim_sobolev_approx}), we must approximate $F$ by cylindrical maps in $C^{k;p}_{\mu^0,Cyl}(X,Y)$ (resp. $C^{k;p}_{\nu,Cyl}(X,Y)$) in Bastiani--Sobolev norm. The space $C^{k;p}_{B,\mu}(X,Y)$ (resp. $C^{k;p,r}_{B,\nu,\eta}(X,Y)$) of all $F\in C^k_B(X,Y)$ with finite Bastiani--Sobolev norm, i.e. the direct extension of the space of \cite{hornik1991approximation}, is, in general, too large for this: it contains smooth maps that cannot be so approximated (see \eqref{eq:cantapproximate-in-bastiani-sobolev} for $F$ in Example \ref{ex:strict_inclusions}); even for differentiable measures, such approximations are known only for $k=1$, $p>1$, under technical conditions \cite[Theorem 8.7.2 (iii)]{bogachev2010}. By \eqref{eq:closure_fail_derivatives_measure}, we cannot proceed as in the Gaussian case either, where Sobolev spaces are completions of smooth cylindrical maps \cite[Section 5.2]{bogachev1998gaussian} (see \eqref{eq:gaussian_sobolev}, \eqref{eq:Wp=Dp} and \cite{maas,pronk-veraar}). However, inspired by the latter, we use $C^{k;p}_{B,\mu}(X,Y)$ as an ambient space and define the \textbf{weighted Bastiani--Sobolev space} $C^{k;p}_{B,\mu,A}(X,Y)$ as the closure therein of $C^{k;p}_{\mu^0,Cyl}(X,Y)$, i.e. of cylindrical maps whose finite-dimensional representatives have finite weighted Sobolev norm; ``A'' stands for approximable. Similarly for $C^{k;p,r}_{B,\nu,\eta,A}(X,Y)$. As in \cite{hornik1991approximation}, these normed spaces need not be complete, which is not an issue for UATs.

\paragraph{Goals achieved via Bastiani--Sobolev spaces.} Bastiani--Sobolev spaces allow us to achieve simultaneously the following goals:
\begin{enumerate}[(i)]
    \item \textbf{Universal Approximation.} \textit{The spaces $C^{k;p}_{B,\mu,A}(X,Y), $ $C^{k;p,r}_{B,\nu,\eta,A}(X,Y)$ are  abstract constructions, where UA is  achieved naturally} (Theorem \ref{th:UAT_Sobolev_Banach}).
However, we are then faced with the challenge of justifying that our space contains interesting maps. We address this by giving concrete \textit{criteria for membership}, based on two natural approximation methods:
\begin{enumerate}[(a)]
    \item for separable $X,Y$ with BAP (Assumption \ref{ass:bounded_approx_pro}), we approximate $F$ by its cylindrical approximations $F^{{\mathcal E}^X_d,{\mathcal D}^Y_m}$ (Definition \ref{def:cyl_approx}) and pass to the limit, e.g. by dominated convergence, under additional conditions. This leads to the space $C^{k;p}_{B,\mu,CC}(X,Y)$, where $F^{{\mathcal E}^X_d,{\mathcal D}^Y_m}\to F$, containing the (concrete) space $C^{k;p,q}_{B}(X,Y)$ of maps with derivatives of polynomial growth if $\mu$ has suitable finite moments;
    \item we first approximate $F$ by bounded maps with bounded derivatives via smooth bump functions, when available (Subsection \ref{sec:bump}, Assumption \ref{ass:density_CB}), and then proceed as in (a). This removes the growth conditions in (a) (and moment conditions on the input measure) and leads to the (concrete) space $\widetilde C^{k;p}_{B,\mu}(X,Y)$; however, smooth bump functions are not always available in Banach spaces.\end{enumerate}
Similarly, we construct the corresponding spaces for $C^{k;p,r}_{B,\nu,\eta,A}(X,Y)$. Indeed, we will show in Theorem \ref{th:relations_sobolve_norms} that all these spaces are contained in $C^{k;p}_{B,\mu,A}(X,Y)$, $C^{k;p,r}_{B,\nu,\eta,A}(X,Y)$, respectively, under suitable conditions.  Finally, our spaces may contain interesting maps even when BAP fails: e.g. in the Gaussian case (see (iii)), for differentiable measures ($k=1$, $p>1$, under the conditions of \cite[Theorem 8.7.2 (iii)]{bogachev2010}), and for the non-cylindrical map $\tilde F$ in Example \ref{ex:strict_inclusions}, defined on a Banach space without AP (see \eqref{eq:cantapproximate-in-bastiani-sobolev}).
\item \textbf{The finite-dimensional case is recovered.}  When $X,Y$ are finite dimensional, we recover the weighted Sobolev space used in the classical finite-dimensional UATs in \cite{hornik1991approximation} (see Proposition \ref{prop:relations_sobolev_norms_1}).
\item \textbf{Gaussian Sobolev spaces on Banach spaces are recovered.}  When the underlying measures are regular enough for closing derivative operators, we can complete the weighted Bastiani--Sobolev space. For instance, thanks to the identity of the norms in the Gaussian case shown above, \textit{our construction of these Sobolev spaces generalizes the one of classical Gaussian Sobolev spaces \cite{bogachev1998gaussian} to general finite measures on $X^{k+1}$}, in the following sense: we prove that, when $\nu=\gamma, \eta=\otimes_{i=1}^k\gamma, \mu=\otimes_{i=0}^k\gamma$, {with $X$ separable, $Y$ a separable Hilbert space, and $\gamma$ a centered Gaussian measure on $X$,} \textit{we recover the classical Gaussian Sobolev space in \cite[Section 5.2]{bogachev1998gaussian} via completion} (see  Theorem \ref{th:characterization_bogachev}).
\item \textbf{Bastiani--Sobolev training in DIOL.} The UATs in Bastiani--Sobolev spaces naturally lead us to Bastiani--Sobolev training, i.e. a natural infinite-dimensional extension of the Sobolev training in high, but finite, dimensional spaces \cite{czarnecki2017}, while also recovering the Sobolev training studied in DIOL (e.g. \cite{luo-oleary-chen-ghattas}), see Section \ref{sec:bastiani-sobolev-training}.
\end{enumerate}
\paragraph{Open completion problem.} An open question posed to us by Giuseppe Savaré and Giacomo Sodini is whether, at least for $k=1$, the completions of Bastiani--Sobolev spaces under a general finite input measure can be characterized by Cheeger-type energies, in the spirit of the scalar-valued results \cite{fornasier-savare-sodini2023}.
}
\subsection{Weighted Bastiani--Sobolev norms}\label{subsec:sobolev_norms}
Let $(X,|\cdot|_X),(Y,|\cdot|_Y)$ be  Banach spaces; let $\mu$ be a finite   measure on $X\times X^k$, $k\in \mathbb N,$ $\nu$ be a finite Radon  measure on $X$, and $p\geq 1$. For $F \in C_B^k(X,Y)$, we define the following quantities  which will define weighted Sobolev norms for  the weighted Sobolev spaces in Section \ref{subsec:sobolev_sp} (under suitable conditions).
\begin{itemize}[leftmargin=*]  
    \item We define the standard weighted Sobolev norm on ${C^k_F}(X,Y)$ (recall  Notation \ref{not:notation_D0f}):
\begin{equation}\label{eq:sobolev_norm}
    \|F\|_{\mathcal W^{k,p}_{\nu}(X,Y)}:=\left(\sum_{i=0}^k\left[\int_{X } \|D^iF(x)\|_{\mathcal L^i(X,Y)}^pd\nu (x)\right] \right)^{1/p}.
\end{equation}
The above quantity \eqref{eq:sobolev_norm} is well-defined (but possibly infinite) for all $F \in C_B^k(X,Y)$ by Remark \ref{rem:borel_measu_opnorms_CB}. 

Recalling Notation \ref{not:marginals} for $\nu^{{\mathcal E}^X}$, by \eqref{eq:derivativeDrv_dh^theta_dY_Banach_cyl}, for $F\in C^k_{Cyl}(X,Y)$, we have
\begin{equation}\label{eq:sobolev-norm-ineq-cyl}
    \begin{aligned}
\|F^{{\mathcal E}^X,{\mathcal D}^Y}\|_{\mathcal W^{k,p}_{ \nu}(X,Y)}&=\left(\sum_{i=0}^k\left[\int_{X } \|D^iF^{{\mathcal E}^X,{\mathcal D}^Y}(x)\|_{\mathcal L^i(X,Y)}^pd \nu (x)\right] \right)^{1/p}\\
&=\left(\sum_{i=0}^k\left[\int_{X } \left \|{\mathcal D}^Y \left[D^i  f^{N ,N'}\left({\mathcal E}^X  x\right)\right]\left[{\mathcal E}^X (\cdot), \ldots, {\mathcal E}^X  (\cdot)\right]\right \|_{\mathcal L^i(X,Y)}^pd \nu (x)\right] \right)^{1/p}\\
&\leq C_{{\mathcal E}^X,{\mathcal D}^Y}\left(\sum_{i=0}^k\left[\int_{X } \left \|D^i  f^{N ,N'}\left({\mathcal E}^X  x\right)\right \|_{\mathcal L^i(\mathbb R^N,\mathbb R^{N'})}^pd\nu (x)\right] \right)^{1/p}\\
&= C_{{\mathcal E}^X,{\mathcal D}^Y}\left(\sum_{i=0}^k\left[\int_{{\mathbb R^{N}} } \left\|D^i  f^{N ,N'}\left(y\right)\right\|_{\mathcal L^i(\mathbb R^N,\mathbb R^{N'})}^pd \nu^{{\mathcal E}^X} (y)\right] \right)^{1/p}\\
&= C_{{\mathcal E}^X,{\mathcal D}^Y}  \  \| f^{N,N'}\|_{\mathcal  W^{k,p}_{ \nu^{{\mathcal E}^X}}(\mathbb R^{N},\mathbb R^{N'})}.
    \end{aligned}
\end{equation}
\item The above norm \eqref{eq:sobolev_norm} will be too strong for most of our results. Therefore, we introduce weaker  norms, which we call weighted Bastiani--Sobolev norms.  For $A\subset X$ Borel measurable, $F \in C^k_B(X,Y)$  we define (recall  Notations \ref{not:notation_D0f}, \ref{not:marginals})
\begin{align}
    &\|F\|_{\mathcal W^{k,p}_{B,\mu}(A,X,Y)}:=\left(\sum_{i=0}^k\left[\int_{A \times X^i} |D^iF(x)(h^1,...,h^i)|^p_Yd\mu^{0:i}(x,h^1,\ldots,h^i)\right] \right)^{1/p},\label{eq:sobolev_seminorm_O}\\
   & \|F\|_{\mathcal W^{k,p}_{B,\mu}(X,Y)}:=\|F\|_{\mathcal W^{k,p}_{B,\mu}(X,X,Y)}\equiv \left(\sum_{i=0}^k\left[\int_{X \times X^i} |D^iF(x)(h^1,...,h^i)|^p_Yd\mu^{0:i}(x,h^1,\ldots,h^i)\right] \right)^{1/p}.\label{eq:sobolev_seminorm}
\end{align}
\item 
To extend Sobolev bump-type arguments to these infinite-dimensional settings (see Proposition \ref{pro:ass_density_holds}), we introduce the following additional  norms.
For $A\subset X$ Borel, we set (recall Notations \ref{not:notation_D0f}, \ref{not:marginals}; here, we denote $J^c:=\{1,\ldots,i\}\setminus J$)
\begin{align}
   & \|F\|_{\widetilde{\mathcal W}^{k,p}_{B,\mu}(A,X,Y)}:=
   \left(\sum_{0\leq i \leq k}\sum_{J\subset \{1,\ldots,i\}}\left[\int_{A \times X^i} \left|D^{|J|} F(x)((h^j)_{j\in J})\right|_Y^p\prod_{t \in J^c}|h^t|_X^pd\mu^{0:i}(x,h^1,\ldots,h^i)\right] \right)^{1/p},\label{eq:sobolev_seminorm_2_O}\\
    & \|F\|_{\widetilde{\mathcal W}^{k,p}_{B,\mu}(X,Y)}:= \|F\|_{\widetilde{\mathcal W}^{k,p}_{B,\mu}(X,X,Y)}.\label{eq:sobolev_seminorm_2_O_X}
\end{align}
Note that 
\begin{align}\label{eq:sobolev_seminorm_2_O_geq}
    \|F\|_{\mathcal W^{k,p}_{B,\mu}(A,X,Y)}\leq \|F\|_{\widetilde{\mathcal W}^{k,p}_{B,\mu}(A,X,Y)},\quad \textit{Therefore, } \|F\|_{\mathcal W^{k,p}_{B,\mu}(X,Y)}\leq \|F\|_{\widetilde{\mathcal W}^{k,p}_{B,\mu}(X,Y)}.
\end{align}

\begin{remark}\label{rem:seminorm_sobolev_bastiani_tilde_reverse-ineq}
{We observe that the reverse inequality in \eqref{eq:sobolev_seminorm_2_O_geq}
holds, e.g., when
\(
\mu=\nu\otimes\xi^{\otimes k},
\)
where \(\xi\) is a probability measure such that $ M_{p,\xi}:=\int_X|h|_X^p\,d\xi(h)<\infty$. Indeed, if
\(J\subseteq\{1,\ldots,i\}\) and \(j:=|J|\), then, after permuting the
direction coordinates, Tonelli's theorem gives
\begin{align*}
&\int_{A\times X^i}
 \bigl|D^jF(x)((h_s)_{s\in J})\bigr|_Y^p
 \prod_{t\in J^c}|h_t|_X^p
 \,d\mu^{0:i}(x,h_1,\ldots,h_i)
\\
&=\int_A
 \left(\int_{X^j}
  \bigl|D^jF(x)(u_1,\ldots,u_j)\bigr|_Y^p
  \,d\xi^{\otimes j}(u_1,\ldots,u_j)
 \right)
 \left(
  \int_{X^{i-j}}
  \prod_{\ell=1}^{i-j}|v_\ell|_X^p
  \,d\xi^{\otimes(i-j)}(v_1,\ldots,v_{i-j})
 \right)d\nu(x)
\\
&=\int_A
 \left( \int_{X^j}
  \bigl|D^jF(x)(u_1,\ldots,u_j)\bigr|_Y^p
  \,d\xi^{\otimes j} (u_1,\ldots,u_j)
 \right)
 \prod_{\ell=1}^{i-j}
 \left(
  \int_X|v_\ell|_X^p \,d\xi(v_\ell)
 \right)d\nu(x)
\\
&=M_{p,\xi}^{\,i-j}
 \int_{A\times X^j}
 \bigl|D^jF(x)(u_1,\ldots,u_j)\bigr|_Y^p
 \,d(\nu\otimes\xi^{\otimes j})(x,u_1,\ldots,u_j).
\end{align*}
Summing the preceding identity,  for some $C>0$, we have
\begin{align*}
\|F\|_{\widetilde{\mathcal W}^{k,p}_{B,\mu}(A,X,Y)}^p
&\leq C
 \sum_{j=0}^k
 \int_{A\times X^j}
 \bigl|D^jF(x)(u_1,\ldots,u_j)\bigr|_Y^p
 \,d(\nu\otimes\xi^{\otimes j})
=C
 \|F\|_{\mathcal W^{k,p}_{B,\mu}(A,X,Y)}^p.
\end{align*}
Combining this inequality with
\eqref{eq:sobolev_seminorm_2_O_geq}, we conclude that the two
seminorms are equivalent.}
\end{remark}
\item 
Important cases will be when $\mu$ is defined via a product measure $\mu=\nu\otimes\eta$, where $\nu$ is a finite measure on $X$ and $\eta$ is a probability measure on $X^k$. Since $F \in C^k_B(X,Y)$, we apply Tonelli's theorem (recall Definition \ref{def:Lpnorm}; here, when $i=0$ we use the convention $\|D^iF(x)\|_{L^p(X^i,Y,\eta^{1:i})}=|F(x)|_Y$), to have 
\begin{equation}\label{eq:sobolev_norm_mu0_Lp}
\begin{aligned}
    \|F\|_{\mathcal W^{k,p}_{B,\mu}(X,Y)}&:=\left(\sum_{i=0}^k\left[\int_{ X}\int_{X^i } |D^iF(x)(h^1,...,h^i)|^p_Yd\eta^{1:i}(h^1,\ldots,h^i)d\nu(x)\right] \right)^{1/p}\\
    &=\left(\sum_{i=0}^k\left[\int_{X} \|D^iF(x)\|_{L^p(X^i,Y,\eta^{1:i})}^pd\nu(x)\right] \right)^{1/p}.
\end{aligned}
\end{equation}
\item For $r\geq 1,$ we will also consider the following generalization of \eqref{eq:sobolev_norm_mu0_Lp} (note that $\|\cdot\|_{\mathcal W^{k,p}_{B,\mu}(X,Y)}=\|\cdot\|_{\mathcal W^{k,p,p}_{B,\nu,\eta}(X,Y)}$):
\begin{equation}\label{eq:sobolev_norm_mu0_Lp_bogachev}
\begin{aligned}
    \|F\|_{\mathcal W^{k,p,r}_{B,\nu,\eta}(X,Y)}
    :={\left(\sum_{i=0}^k\left[\int_{X} \|D^iF(x)\|_{L^r(X^i,Y,\eta^{1:i})}^pd\nu(x)\right] \right)^{1/p}}.
\end{aligned}
\end{equation}
The Sobolev norm $\|\cdot\|_{\widetilde{\mathcal W}^{k,p,r}_{B,\nu,\eta}(X,Y)}$ is obtained by modifying \eqref{eq:sobolev_seminorm_2_O_X} accordingly, i.e. 
for $A\subset X$ Borel,
\begin{align*}
   & \|F\|_{\widetilde{\mathcal W}^{k,p,r}_{B,\nu,\eta}(A,X,Y)}:=
   \left(\sum_{0\leq i \leq k}\sum_{J\subset \{1,\ldots,i\}}\left[\int_{A}\left( \int_{X^i}  \left|D^{|J|} F(x)((h^j)_{j\in J})\right|_Y^r\prod_{t \in J^c}|h^t|_X^rd\eta^{1:i}(h^1,\ldots,h^i)\right)^{p/r}d\nu(x)\right] \right)^{1/p},\\
    & \|F\|_{\widetilde{\mathcal W}^{k,p,r}_{B,\nu,\eta}(X,Y)}:= \|F\|_{\widetilde{\mathcal W}^{k,p,r}_{B,\nu,\eta}(X,X,Y)}.
\end{align*}
\item 
{Let $X$ be a separable Banach space, $Y$ be {a} separable Hilbert space, $\nu$ {a} finite measure on $X$, and   $\gamma$ be a centered {Gaussian} measure on $X$ (recall Examples \ref{ex:HS_norms_RKHS}, \ref{ex:gaussian_measures_radonifying}).  We define
\begin{align}\label{eq:def_sobolev_norm_mu0_Lp_bogachev_HS}
   & \|F\|_{\mathcal W^{k,p}_{\nu,\mathcal L_2(X_\gamma,Y)}(X,Y)}
   : =\left(\sum_{i=0}^k\left[\int_{X} \|D_{{\gamma}}^iF(x)\|_{\mathcal L^i_2(X_{\gamma},Y)}^pd\nu(x)\right] \right)^{1/p},
\end{align}
where in the above we have denoted by $D_{{\gamma}}^iF(x)=D^iF(x)\circ(\iota_{\gamma},\ldots,\iota_{\gamma}),i\geq 1$ the restriction of $D^iF(x)$ to the Cameron--Martin space/RKHS $X_{\gamma}$.}

We {preliminarily} note that, setting  $\eta:=\otimes_{i=1}^k\gamma$ (so that $\mu:=\nu\otimes(\otimes_{i=1}^k\gamma)$), by \eqref{eq:norm2=HS}, we have 
\begin{align}\label{eq:sobolev_norm_mu0_Lp_bogachev_HS_p=2}
   &\|\cdot\|_{\mathcal W^{k,p}_{\nu,\mathcal L_2(X_\gamma,Y)}(X,Y)}=\|\cdot\|_{\mathcal W^{k,p,2}_{B,\nu,\eta}(X,Y)} 
    ,\quad \|\cdot\|_{\mathcal W^{k,2}_{\nu,\mathcal L_2(X_\gamma,Y)} (X,Y)}=\|\cdot\|_{\mathcal W^{k,2}_{B,\mu}(X,Y)}\equiv \|\cdot\|_{\mathcal W^{k,2,2}_{B,\nu,\eta}(X,Y)}
   .
\end{align}
{Due to Example \ref{ex:HS_norms_RKHS}, \eqref{eq:sobolev_norm_mu0_Lp_bogachev_HS_p=2}  still holds when $\gamma$ in \eqref{eq:def_sobolev_norm_mu0_Lp_bogachev_HS} is replaced by {any centered probability measure $\rho$ on $X$ with $\int_X|h|_X^2\,d\rho(h)<\infty$}.}
\end{itemize}

\subsection{Weighted Bastiani--Sobolev spaces}\label{subsec:sobolev_sp}
In this subsection, motivated by Subsection \ref{subsec:motivation_sobolev}, we construct the weighted Bastiani--Sobolev spaces. Let $(X,|\cdot|_X),(Y,|\cdot|_Y)$ be  Banach spaces; let $\mu$ be a finite   measure on $X\times X^k$, $k\in \mathbb N,$ $\nu$ be a finite  measure on $X$, and $p\geq 1$.
\begin{itemize}[leftmargin=*]  
\item We  define the space
\begin{align*} C^{k;p}_{B,\mu}(X,Y):=\{F\in C^k_B(X,Y):\|F\|_{\mathcal  W^{k,p}_{B,\mu}(X,Y)}<\infty\},\end{align*}
and we endow it with the Sobolev norm $\|\cdot\|_{\mathcal  W^{k,p}_{B,\mu}(X,Y)}$.  This is a norm up to a quotient with respect to the equivalence relation $F\sim F'$ defined by $|D^iF(\cdot)(\cdot,...,\cdot)-D^iF'(\cdot)(\cdot,...,\cdot)|_Y=0$, $\mu^{0:i}$-a.e., for all $0\leq i\leq k$. However, if $\mu$ has full support, then, by continuity of all maps, $\|\cdot\|_{\mathcal  W^{k,p}_{B,\mu}(X,Y)}$ is a norm on $C^{k;p}_{B,\mu}(X,Y)$; clearly, the resulting normed space is not complete. \textit{This will act as our ambient space}.
\item Define the space of cylindrical functions with finite Sobolev norm $\|\cdot\|_{\mathcal  W^{k,p}_{ \nu^{{\mathcal E}^X}}(\mathbb R^{N},\mathbb R^{N'})}$ (recall Notation \ref{not:marginals}) by
\begin{align*}C^{k;p}_{\nu,Cyl}(X,Y):=\{&F^{{\mathcal E}^X,{\mathcal D}^Y}\in C^k_{Cyl}(X,Y): \|f^{N,N'}\|_{\mathcal  W^{k,p}_{ \nu^{{\mathcal E}^X}}(\mathbb R^{N},\mathbb R^{N'})}<\infty \}.\end{align*}
Note that, by \eqref{eq:sobolev-norm-ineq-cyl}, we have 
 $\|F^{{\mathcal E}^X,{\mathcal D}^Y}\|_{\mathcal W^{k,p}_{ \nu}(X,Y)}<\infty,$ for all $  F^{{\mathcal E}^X,{\mathcal D}^Y}\in C^{k;p}_{\nu,Cyl}(X,Y).$ Under Assumption \ref{ass:Radon-nik}, by \eqref{eq:bound_Wmu_barmu}, we will have that $C^{k;p}_{\mu^0,Cyl}(X,Y)$ is a subspace of $ C^{k;p}_{B,\mu}(X,Y)$.
    \item {Under Assumption \ref{ass:Radon-nik},} define the weighted Bastiani--Sobolev  space 
\begin{align*} 
C^{k;p}_{B,\mu,A}(X,Y)&:={\overline{C^{k;p}_{\mu^0,Cyl}(X,Y)}}=\{F\in C^k_B(X,Y): \|F\|_{\mathcal W^{k,p}_{B,\mu}(X,Y)}<\infty, \exists \left\{ F_n \right\}_{n\in \mathbb N}\subset  C^{k;p}_{\mu^0,Cyl}(X,Y) : \\
&\quad\quad\quad\quad\quad\quad\quad\quad\quad\quad \|F-F_n\|_{\mathcal W^{k,p}_{B,\mu}(X,Y)}\xrightarrow{n\to \infty}0\},
\end{align*}
{{where the closure is taken in ${C^{k;p}_{B,\mu}(X,Y)}$. As a subspace of $C^{k;p}_{B,\mu}(X,Y)$, it is a normed space.} It {need} not be complete and a completion  encounters, in general, similar issues as in \eqref{eq:closure_fail_derivatives_measure}, as we are dealing with a general finite measure $\mu$.} However, as mentioned in Subsection \ref{subsec:motivation_sobolev}, our goal here is to state density results in universal approximation theorems for as general measures as possible.
To these purposes, the space $C^{k;p}_{B,\mu,A}(X,Y)$ is the space of $C^k_B$ maps with finite $\|\cdot\|_{\mathcal W^{k,p}_{B,\mu}(X,Y)}$-norm that can be approximated (from which the subscript ``A'') via cylindrical maps in $C^{k;p}_{\mu^0,Cyl}(X,Y)$. 
We will state the Sobolev universal approximation theorem for this space (Theorem \ref{th:UAT_Sobolev_Banach}). 

{When the underlying measures are regular enough for closing derivative operators, we can complete the weighted Bastiani--Sobolev space. For instance, \textit{in the Gaussian case, we will   recover the classical Gaussian Sobolev spaces} in \cite[Chapter 5]{bogachev1998gaussian} via a completion (Subsection \ref{subsec:gaussian_recover}).}
\end{itemize} 
\emph{The Sobolev space $C^{k;p}_{B,\mu,A}(X,Y)$ is an abstract construction, where UA is  achieved naturally} (Theorem \ref{th:UAT_Sobolev_Banach}). However, we want to show that  $C^{k;p}_{B,\mu,A}(X,Y)$ contains many interesting maps. To this purpose, we construct the following concrete Sobolev spaces, leading to  Theorem \ref{th:relations_sobolve_norms}, in particular, \eqref{eq:Ckwq_subset_Ckwc}, \eqref{eq:tildeC_subset_C}.
\begin{itemize}[leftmargin=*] 
  \item Let $X,Y$ be separable Banach spaces with BAP, so that we are given $\left\{{\mathcal S}^X_d\right\}_{d\in \mathbb N }$, $\left\{{\mathcal S}^Y_m\right\}_{m \in \mathbb N}$ satisfying Assumption \ref{ass:bounded_approx_pro}.  We define (recall cylindrical approximations $F_{d,m}$ Definition \ref{def:cyl_approx})
\begin{align*}C^{k;p}_{B,\mu,CC}(X,Y):=\{&F\in C^k_B(X,Y):\|F\|_{\mathcal  W^{k,p}_{B,\mu}(X,Y)}<\infty; \|f^{N_d,N_m}\|_{\mathcal  W^{k,p}_{ \mu^{0,{\mathcal E}^X_d}}(\mathbb R^{N_d},\mathbb R^{N_m})}<\infty,\forall d,m; \\
&\|F-F_{d,m}\|_{\mathcal W^{k,p}_{B,\mu}(X,Y)}\xrightarrow{d,m \to \infty}0 \}.\end{align*}
Note that the space depends on $\left\{{\mathcal S}^X_d\right\}_{d\in \mathbb N }$, $\left\{{\mathcal S}^Y_m\right\}_{m \in \mathbb N}$, even if not emphasized in the notation.
The subscript ``CC'' stands for the consistency  of cylindrical approximations $F_{d,m}$ with respect to the norm $\|\cdot\|_{\mathcal W^{k,p}_{B,\mu}(X,Y)}$ (i.e.   $\|F-F_{d,m}\|_{\mathcal W^{k,p}_{B,\mu}(X,Y)}\xrightarrow{d,m \to \infty}0$).   Here, we choose to work  with the Bounded Approximation Property in separable Banach spaces, in place of the weaker Approximation Property in Banach spaces, as it is a lot easier to deal with sequences of operators $\left\{{\mathcal S}^X_d\right\}_{d\in \mathbb N }$, $\left\{{\mathcal S}^Y_m\right\}_{m \in \mathbb N}$, in place of Moore–Smith sequences, when dealing with convergence of integrals.  
  \item
Under $\|\mu\|_{k,q,p}  < \infty$, we  define  the space (notice that it is a subspace of $C^{k;p}_{B,\mu}(X,Y)$; see also \eqref{eq:bound_norm_for_Ckpq})
\begin{align*} C^{k;p,q}_{B}(X,Y):=\{F\in C^k_B(X,Y):\exists C=C(F)>0 :\|D^iF(x)\|_{\mathcal L^i(X,Y)}^p\leq C(1+|x|_X^q) , \forall i\leq k,x \in X\}\end{align*}
 and we endow it with the  norm $\|\cdot\|_{\mathcal W^{k,p}_{B,\mu}(X,Y)}$ (up to the quotienting and similar considerations as above). Here, the notation with the double index $p,q$ (rather than the single growth bound $a=q/p$) is meant to emphasize $L^p$-integrability of $|D^iF(x)(h_1,\ldots,h_i)|_Y^p$,  $i\leq k,$ for $\mu$ such that $\|\mu\|_{k,q,p}  < \infty$ (recall \eqref{eq:moments_mu}). 
 \item We define the space (notice that it is a subspace of $C^{k;p}_{B,\mu}(X,Y)$ by  \eqref{eq:sobolev_seminorm_2_O_geq})
\begin{equation*}
\begin{aligned}\widetilde C^{k;p}_{B,\mu}(X,Y)&:=\{F\in C^k_B(X,Y):\|F\|_{\widetilde{\mathcal  W}^{k,p}_{B,\mu}(X,Y)}<\infty, D^iF \textit{ bounded on bounded subsets of }X, \forall i\leq k\},
\end{aligned}
\end{equation*}
endowed with the norm $ \|\cdot\|_{\mathcal W^{k,p}_{B,\mu}(X,Y)}$ (up to quotienting and similar considerations as above). This space will be particularly important, because we will be able to state the universal approximation theorem Corollary \ref{th:UAT_Sobolev_K_Banach_bump} under  Assumption \ref{ass:density_CB} (see Subsection \ref{subsubsec:ass_density_CB} for its validity), where we have no growth conditions or consistency conditions in the definition of the space.
\item Let $r\geq 1,$  $\nu$ be a finite measure on $X$ and $\eta$ be a probability measure on $X^k$.   Then,  we  define the spaces  $C^{k;p,r}_{B,\nu,\eta}(X,Y),$ $C^{k;p,r}_{B,\nu,\eta,A}(X,Y)$, $C^{k;p,r}_{B,\nu,\eta,CC}(X,Y)$, $C^{k;p,q}_{B}(X,Y),$ $\widetilde C^{k;p,r}_{B,\nu,\eta}(X,Y)$ corresponding to $ C^{k;p}_{B,\mu}(X,Y),$ $C^{k;p}_{B,\mu,A}(X,Y)$, $C^{k;p}_{B,\mu,CC}(X,Y)$, $C^{k;p,q}_{B}(X,Y),$  $\widetilde C^{k;p}_{B,\mu}(X,Y)$, respectively, by  replacing  $ \|\cdot\|_{\mathcal W^{k,p}_{B,\mu}(X,Y)}$, $\|\cdot\|_{\widetilde{\mathcal W}^{k,p}_{B,\mu}(X,Y)}$,  with  $\|\cdot\|_{\mathcal W^{k,p,r}_{B,\nu,\eta}(X,Y)}$, $\|\cdot\|_{\widetilde{\mathcal W}^{k,p,r}_{B,\nu,\eta}(X,Y)}$, respectively and replacing the condition $\|\mu\|_{k,q,p}  < \infty$ with $\int_X(1+|x|^q)d\nu(x)<\infty$ (or equivalently, $\int_X|x|^qd\nu(x)<\infty$) and $\|\eta\|_{k,r}<\infty$.  Recall also that for $p=r$ and $\mu=\nu\otimes\eta$, we have $\|\cdot\|_{\mathcal W^{k,p}_{B,\mu}(X,Y)}=\|\cdot\|_{\mathcal W^{k,p,p}_{B,\nu,\eta}(X,Y)}$, $\|\cdot\|_{\widetilde{\mathcal W}^{k,p}_{B,\mu}(X,Y)}=\|\cdot\|_{\widetilde{\mathcal W}^{k,p,p}_{B,\nu,\eta}(X,Y)}$.
\item {Let $\gamma$ be a centered Gaussian on a separable Banach space $X$.} Let $Y$ be a separable Hilbert space.  Define, as in \cite[Section 5.2]{bogachev1998gaussian} (recall Definition \ref{def:cyl_map}),
\begin{equation}\label{eq:gaussian_sobolev}
   W_{\gamma}^{k,p}(X,Y) :=\widehat{C_{b,Cyl}^\infty(X,Y)} ,
\end{equation}
where the completion is taken under $\|\cdot\|_{\mathcal W^{k,p}_{\gamma,\mathcal L_2(X_{\gamma},Y)}(X,Y)}.$
 {We denote by $D^{1,p}_{\gamma}(X,Y)$ the space of  stochastic Gâteaux differentiable, ray absolutely continuous  maps  $F:X\to Y$, in the sense of \cite[Definitions 5.2.3, 5.2.4]{bogachev1998gaussian} with Borel measurable $D^i_{\gamma}F:X\to \mathcal L_2^i(X_\gamma,Y)$ and finite norm \eqref{eq:def_sobolev_norm_mu0_Lp_bogachev_HS} (for $k=1$). Then, the definition is extended by induction for $k\in \mathbb N$, i.e. $D^{k,p}_{\gamma}(X,Y):=\{F\in D^{k-1,p}_{\gamma}(X,Y): D_{\gamma}F\in D^{k-1,p}_{\gamma}(X,\mathcal L_2(X_{\gamma},Y))\}$ and corresponding norm \eqref{eq:def_sobolev_norm_mu0_Lp_bogachev_HS}.
 We recall that, by \cite[Proposition 5.4.6]{bogachev1998gaussian} for all $p\geq 1,$  
\begin{equation}\label{eq:Wp=Dp}
  W^{k,p}_{\gamma}(X,Y)=  D^{k,p}_{\gamma}(X,Y).
\end{equation}
}
\item We define the space 
\begin{align}\label{eq:growth_v_UATsobolev-op_norm} {C^{k;p}_{F,\nu}}(X,Y):=\{F\in {C^k_F}(X,Y):\|F\|_{\mathcal  W^{k,p}_{\nu}(X,Y)}<\infty\},\end{align}
and we endow it with the norm $\|\cdot\|_{\mathcal  W^{k,p}_{\nu}(X,Y)}$ (up to quotienting with respect to the equivalence relation $F\sim F'$ defined by $\|D^iF(\cdot)-D^iF'(\cdot)\|_{\mathcal L^i(X,Y)}=0$, $\nu$-a.e., for all $0\leq i\leq k$ and similar considerations as above).  \emph{This is a direct extension of the Sobolev space considered in \cite[Theorem 4]{hornik1991approximation} to Banach spaces.}
\end{itemize}
\subsection{Properties of  weighted Bastiani--Sobolev spaces}
In this subsection, we study the Sobolev spaces introduced before. Let $(X,|\cdot|_X),(Y,|\cdot|_Y)$ be  Banach spaces; let $\mu$ be a finite   measure on $X\times X^k$, $k\in \mathbb N,$  $p\geq 1$.

{We use the following assumption.}
 \begin{assumption}\label{ass:Radon-nik}
We assume that $\int_{X\times X^i} \prod_{r=1}^i\left|h^r\right|_X^p d \mu^{0:i}\left(x, h^1, \ldots, h^i\right)<\infty$ for all $1\leq i\leq k$ (so that $\bar \mu^{0:i}$ is a finite measure such that $\bar \mu^{0:i}\ll \mu^{0}$; recall Notation \ref{not:marginals}) and 
$
\frac{d \bar \mu^{0:i}}{d \mu^{0}} \in L^{\infty}\left(X;\mu^{0}\right)
$.
\end{assumption}
\begin{example}\label{ex:ass_radon-nik}
 Assumption \ref{ass:Radon-nik} is  satisfied, e.g., if
    \begin{enumerate}
        \item[(i)] $\mu=\nu\otimes\eta$, where $\nu$ is a finite measure on $X$ and $\eta$ is a  probability measure on $X^k$ such that $\|\eta\|_{k,p}<\infty$; this is the simplest case. In this case, $\mu^0=\nu,$
        $
\bar \mu^{0:i}(A)=\mu^0(A)\int_{X^i} \prod_{r=1}^i\left|h^r\right|_X^p d \eta^{1:i}(h^1,\ldots,h^i),
$
    so that $ \bar \mu^{0:i}\ll\mu^0$ with $\frac{d \bar \mu^{0:i}}{d \mu^0}(x)=\int_{X^i} \prod_{r=1}^i\left|h^r\right|^p d \eta^{1:i}(h^1,\ldots,h^i)$ constant and bounded by $\|\eta\|_{k,p}$ for all $i\leq k$ (recall Notation \ref{not:marginals}). 
   A natural choice for $\eta$ that also respects the symmetric nature of derivatives is to take  $\eta=\otimes_{i=1}^k\hat \eta$, where $\hat \eta$ is a suitable probability measure on $X$ with $\|\hat \eta\|_{1,p}<\infty$. In this case, we interpret $\nu$ as the reference measure in $X$ to measure norms in the input variable $x$ and  $\eta$ or $\hat \eta$ as auxiliary measures  to measure norms in the additional variables $h^j,j\leq k$ appearing from derivatives $D^iF$. 
    
        \item[(ii)] more general conditions for Assumption \ref{ass:Radon-nik} to hold can be derived via disintegration of $\mu$ with respect to $\mu^0$. 
    \end{enumerate}
\end{example}
\begin{remark}\label{rem:ass_RN}
\begin{enumerate}[(i)]
    \item  Assumption \ref{ass:Radon-nik} is a natural condition that allows  us to apply standard results  in finite dimensions \cite{hornik1991approximation} (i.e. Assumption \ref{ass:finite_dim_sobolev_approx}) in the proof of Theorem \ref{th:UAT_Sobolev_Banach}. In particular, Example \ref{ex:ass_radon-nik} (i) shows that (recall also the natural equality \eqref{eq:sobolev_norm_mu0_Lp} in this case),
        \textit{as in the classical case \cite[Theorem 4]{hornik1991approximation}, we allow any finite measure $\nu$ in the input variable $x$.} To measure norms of $D^iF$ {in} the variables $h^i$ instead it is natural to impose some moment conditions.
    \item Assumption \ref{ass:Radon-nik} can be replaced by the following weaker condition:   
$$\textit{for all $1\leq i\leq k$, $\bar \mu^{0:i}$  is a finite measure on $X$.}$$
In this case one needs to adapt Assumption \ref{ass:finite_dim_sobolev_approx}; to justify this then one  can prove a suitable extension of \cite[Theorem 4]{hornik1991approximation} to the case where different measures are used in the definition of the Sobolev norm. In view of the naturality of Example \ref{ex:ass_radon-nik} (i) above we prefer to state the stronger condition Assumption \ref{ass:Radon-nik}. 
\item Another way to weaken Assumption \ref{ass:Radon-nik} is to replace it by a weaker version for pushforward measures with respect to all encoders. However, for practical purposes, it is easier to directly check the validity of Assumption \ref{ass:Radon-nik} directly. 
\end{enumerate}
\end{remark}
\begin{proposition}\label{prop:relations_sobolev_norms_1}Let Assumption \ref{ass:Radon-nik} hold. 
\begin{enumerate}
    \item[(i)] It holds  
    \begin{equation}\label{eq:bound_Wmu_barmu}
        \|F\|_{\mathcal W^{k,p}_{B,\mu}(X,Y)}\leq C  \  \|F\|_{\mathcal  W^{k,p}_{\mu^{0}}(X,Y)},\quad \forall F\in C^{k}_B(X,Y).
    \end{equation}
    \item[(ii)] Let  $X=\mathbb R^N,Y=\mathbb R^{N'}$ be finite dimensional. Let $\mu=\nu\otimes\eta$, where $\nu=\mu^0$ is a finite measure on $X$ and $\eta$ is a probability measure on $X^k$ such that $\|\eta\|_{k,p}<\infty$ (hence,   Example \ref{ex:ass_radon-nik} (i) holds) and  such that ${\|\cdot\|_{L^p(X^i,Y,\eta^{1:i})}}$ defines a norm on $\mathcal L^i(X;Y)$, for all $1\leq i\leq k$ (recall Definition \ref{def:Lpnorm}). Then we have the reverse inequality, i.e. there exists $ C_{N,N'}>0$ such that 
    \begin{equation}\label{eq:bound_revers_Wkmu}
         \|F\|_{\mathcal  W^{k,p}_{\mu^0}(X,Y)}\leq C_{N,N'} \|F\|_{\mathcal W^{k,p}_{B,\mu}(X,Y)},\quad \forall F\in C^{k}_B(X,Y)\equiv {C^{k}_F}(X,Y).
    \end{equation}
    Moreover, 
        \begin{equation}\label{eq:CkBmuc(X,Y)=Ckw-eq}
   C^{k;p}_{B,\mu}(X,Y)= C^{k;p}_{B,\mu,A}(X,Y)=C^{k;p}_{\mu^0,Cyl}(X,Y)=C^{k;p}_{\mu^0}(X,Y).
        \end{equation}
\end{enumerate}
\end{proposition}
\begin{proof}
(i) Using Notation \ref{not:marginals} (in particular, $d\hat\mu^{0:i}/d\mu^{0:i}(x,h^1,\dots,h^i)=\prod_{r=1}^i\left|h^r\right|_X^p$), we have 
    \begin{align*}
   &\int_{X \times X^i} |D^iF(x)(h^1,...,h^i)|^p_Yd\mu^{0:i}(x,h^1,\ldots,h^i)
  \leq \int_{X \times X^i}  \|D^iF(x)\|_{\mathcal L^i(X,Y)}^p\prod_{r=1}^i |h^r|_X^pd\mu^{0:i}(x,h^1,\ldots,h^i)\\
  & = \int_{X \times X^i}  \|D^iF(x)\|_{\mathcal L^i(X,Y)}^pd\hat \mu^{0:i}(x,h^1,\ldots,h^i)=\int_{X }  \|D^iF(x)\|_{\mathcal L^i(X,Y)}^pd\bar \mu^{0:i}(x)\\
  &=\int_X  \|D^iF(x)\|_{\mathcal L^i(X,Y)}^p\frac{d \bar \mu^{0:i}}{d \mu^{0}}(x)d\mu^0(x)\leq  C\int_X  \|D^iF(x)\|_{\mathcal L^i(X,Y)}^pd\mu^0(x),
    \end{align*}  
    where we have used  Assumption  \ref{ass:Radon-nik}, to find  $C>0$ such that, for all $1\leq i\leq k$ the above holds.
Summing over $i$, the claim follows.

 (ii) Since $X,Y$ are finite dimensional,  $\|\cdot\|_{L^p(X^i,Y,\eta^{1:i})}$ is equivalent to the operator norm $\|\cdot\|$ on $\mathcal L^i(X;Y)$, for all $i\leq k$. Therefore, there exists $C_{N,N'}>0$, such that (recall also \eqref{eq:sobolev_norm_mu0_Lp})
 \begin{align*}
 \|F\|_{\mathcal W^{k,p}_{\mu^0}(X,Y)}^p&=\sum_{i=0}^k\left[\int_{X } \|D^iF(x)\|_{\mathcal L^i(X,Y)}^pd\mu^0 (x)\right]\\& \leq C_{N,N'} \sum_{i=0}^k\left[\int_{X } \|D^iF(x)\|_{L^p(X^i,Y,\eta^{1:i})}^pd\mu^0(x)\right]=C_{N,N'} \|F\|_{\mathcal W^{k,p}_{B,\mu}(X,Y)}^p.
\end{align*}
Finally we prove \eqref{eq:CkBmuc(X,Y)=Ckw-eq}. By Remark  \ref{rem:Bastiani} we have $C^k_B(X; Y)={C^k_F}(X; Y)$. Moreover, taking ${\mathcal E}^X=I$ (on $X$), ${\mathcal D}^Y=I$ (on $Y$) all spaces trivially collapse in the finite-dimensional case, proving \eqref{eq:CkBmuc(X,Y)=Ckw-eq}. Then the result follows from \eqref{eq:bound_revers_Wkmu}.
\end{proof}

Next we give  criteria ensuring membership in the weighted Bastiani--Sobolev space. To prove one of them (and Corollary \ref{th:UAT_Sobolev_K_Banach_bump}), we will  use the following assumption. The validity of this assumption will be discussed in Subsection \ref{subsubsec:ass_density_CB}.
\begin{assumption}\label{ass:density_CB}
Let $\mu$ be a finite measure on $X\times X^k$ such that $\|\mu\|_{k,0,p}<\infty$.   Assume that $C^{k;1,0}_{B}(X,Y)$ is dense in $  \widetilde C^{k;p}_{B,\mu}(X,Y)$ (notice also that $C^{k;1,0}_{B}(X,Y)\equiv C^{k;p,0}_{B}(X,Y)$).
\end{assumption}
\begin{remark}
 Assumption \ref{ass:density_CB} is well-posed by \eqref{eq:CBkw0_in_CBkw} below.
\end{remark}
\begin{theorem}[Criteria for membership in the weighted Bastiani--Sobolev space]\label{th:relations_sobolve_norms}Let $\mu$ be a finite   measure on $X\times X^k$.  Let Assumption \ref{ass:Radon-nik} hold. 

(i) We have the following inclusion, possibly, strict:
\begin{align}\label{eq:A_in_ambient}
& C^{k;p}_{B,\mu,A}(X,Y)\subset C^{k;p}_{B,\mu}(X,Y).
\end{align}

(ii) We have the following inclusion, possibly, strict:
\begin{align}\label{eq:cyl_in_A}
&C^{k;p}_{\mu^0,Cyl}(X,Y)\subset C^{k;p}_{B,\mu,A}(X,Y).
\end{align}

(iii) It holds
\begin{align}
& C^{k;p}_{B,\mu,CC}(X,Y)\subset C^{k;p}_{B,\mu,A}(X,Y).
\end{align}

(iv)  Assume that   $\|\mu\|_{k,0,p}<\infty$. Then 
  \begin{align}\label{eq:CBkw0_in_CBkw}
   C^{k;p,0}_{B}(X,Y)\subset   \widetilde C^{k;p}_{B,\mu}(X,Y).
   \end{align}

For (v), (vi) below, assume that  $(X,|\cdot|_X),(Y,|\cdot|_Y)$ are separable Banach spaces  with the BAP, i.e. we are given $\left\{{\mathcal S}^X_d\right\}_{d\in \mathbb N }\subset \mathcal L(X)$, $\left\{{\mathcal S}^Y_m\right\}_{m \in \mathbb N}\subset \mathcal L(Y)$ satisfying Assumption \ref{ass:bounded_approx_pro}.\\
 (v) If, in addition,  $\|\mu\|_{k,q,p}  < \infty$, for $q\geq 0$, we have   
   \begin{align}
      & C^{k;p,q}_{B}(X,Y)\subset C^{k;p}_{B,\mu,CC}(X,Y)\left(\subset C^{k;p}_{B,\mu,A}(X,Y)\right).\label{eq:Ckwq_subset_Ckwc}
   \end{align}
   (vi) If, in addition,  Assumption \ref{ass:density_CB} holds, we have
   \begin{equation}\label{eq:tildeC_subset_C}
       \widetilde C^{k;p}_{B,\mu}(X,Y)\subset C^{k;p}_{B,\mu,A}(X,Y).
   \end{equation}
   \end{theorem}
\begin{proof}(i) The inclusion is immediate from the definition.  It can be strict in view of Example \ref{ex:strict_inclusions}.
(ii) The  inclusion follows  from \eqref{eq:bound_Wmu_barmu} and the definitions of the spaces. It can be strict in view of Example \ref{ex:strict_inclusions}.
 (iii)  This follows from the definitions of the spaces once you recall that, by \eqref{eq:cyl_approx_is_cyl}, cylindrical approximations $F_{d,m}=F^{{\mathcal E}^X_d,{\mathcal D}^Y_m}$ are cylindrical maps.

(iv) Let $F\in C^{k;p,0}_{B}(X,Y)$. Then, using  \eqref{eq:mu_marginals_bound}, we have
    {\begin{align*}
    \|F\|^p_{\widetilde{\mathcal W}^{k,p}_{B,\mu}(X,Y)}&\leq C\sum_{i=0}^k \sum_{J\subset \{1,\ldots,i\}}\left[\int_{X \times X^i} \|D^{|J|}F(x)\|_{\mathcal L^{|J|}(X,Y)}^p \prod_{j\in J}|h^j|_X^p\prod_{ t\in J^c}|h^t|_X^pd\mu^{0:i}(x,h^1,\ldots,h^i)\right]\\
    &\leq C\sum_{i=0}^k\left[\int_{X \times X^i}  \prod_{j=1}^i(1+|h^j|_X^p) d\mu^{0:i}(x,h^1,\ldots,h^i)\right]=C\sum_{i\leq k}\left[\|\mu^{0:i}\|_{i,0,p}\right]\leq C\|\mu\|_{k,0,p}<\infty.
\end{align*}}

(v) Let  $F \in C^{k;p,q}_{B}(X,Y)$. Using \eqref{eq:mu_marginals_bound}, we have  
\begin{equation}\label{eq:bound_norm_for_Ckpq}
\begin{aligned}
   & \|F\|_{\mathcal  W^{k,p}_{B,\mu}(X,Y)}^p\leq C\sum_{i=0}^k\left[\int_{X \times X^i} (1+|x|_X^q) \prod_{j=1}^i(1+|h_j|_X^p)d\mu^{0:i}(x,h^1,\ldots,h^i)\right]\\
   &\quad\quad \quad\quad\quad\quad =C\sum_{i=0}^k\|\mu^{0:i}\|_{i,q,p}\leq C \|\mu\|_{k,q,p}<\infty.
\end{aligned}
\end{equation}
Using \eqref{eq:derivativeDf_Banach} and \eqref{eq:push_back_D}, for all $d,m,$ we have for $C>0$ (possibly depending on $d,m$ and changing from line to line)
    \begin{align*}
        \| f^{N_d,N_m}\|_{\mathcal W^{k,p}_{\mu^{0,{\mathcal E}^X_d}}(\mathbb R^{N_d},\mathbb R^{N_m})}^p
       &= \sum_{i=0}^k\left[\int_{\mathbb R^{N_d} } \left\|D^i  f^{N_d,N_m}\left(y\right)\right\|_{\mathcal L^i(\mathbb R^{N_d},\mathbb R^{N_m})}^pd\mu^{0,{\mathcal E}^X_d} (y)\right]\\
       &= \sum_{i=0}^k\left[\int_{\mathbb R^{N_d} } \left\|{\mathcal E}^Y _{m}\left(D^i  F\left({\mathcal D}^X _{d} y\right)\left[{\mathcal D}^X _{d}(\cdot), \ldots, {\mathcal D}^X _{d} (\cdot)\right]\right)\right\|_{\mathcal L^i(\mathbb R^{N_d},\mathbb R^{N_m})}^pd\mu^{0,{\mathcal E}^X_d} (y)\right]\\
       &\leq C\sum_{i=0}^k\left[\int_{\mathbb R^{N_d} } \left\|D^i  F\left({\mathcal D}^X _{d} y\right)\right\|_{\mathcal L^i(X,Y)}^pd\mu^{0,{\mathcal E}^X_d} (y)\right] \\
       &=C\sum_{i=0}^k\left[\int_{X } \left\|D^i  F\left(x\right)\right\|_{\mathcal L^i(X,Y)}^pd\left[\left({\mathcal D}_d^X\right)_{\#} \mu^{0,{\mathcal E}^X_d}\right](x)\right]\\
       &= C\sum_{i=0}^k\left[\int_{X} \left\|D^i  F\left(x\right)\right\|_{\mathcal L^i(X,Y)}^pd\left[({\mathcal S}_d^X)_{\#} \mu^{0} \right](x)\right] \\&=C\sum_{i=0}^k\left[\int_{X } \left\|D^i  F\left({\mathcal S}_d^Xx\right)\right\|_{\mathcal L^i(X,Y)}^pd \mu^{0} (x)\right]\leq C\sum_{i=0}^k\left[\int_{X } (1+|{\mathcal S}_d^X x|_X^q)d\mu^{0} (x) \right]\\&\leq C\sum_{i=0}^k\left[\int_{X } (1+| x|_X^q)d\mu^{0} (x) \right]\leq C \|\mu\|_{k,q,p}<\infty.
    \end{align*}
 By \eqref{eq:derivativeDrv_dh_dY_Banach} and exploiting  $C^k_B$ and BAP of $X,Y$, for all $0\leq i\leq k$, for all $x,h^1,...,h^i\in X,$ we have (recall also Notation \ref{not:notation_D0f})
 \begin{small}
 \begin{equation}\label{eq:conv_DiDdm}
 \begin{aligned}
    & |D^iF_{d,m}(x)(h^1,...,h^i)|_Y^p\leq C(1+|S_{d}^Xx|_X^q)\prod_{j=1}^i|S_{d}^Xh_j|_X^p\leq C(1+|x|_X^q)\prod_{j=1}^i|h_j|_X^p,\quad (\textit{and similarly for the term }D^iF)\\
    &|[D^iF-D^iF_{d,m}](x)(h^1,...,h^i)|_Y^p=|D^iF(x)(h^1,...,h^i)-{\mathcal S}^Y_{m}\left(D^i F\left({\mathcal S}^X_{d} x\right)\left[{\mathcal S}^X_{d} h^1, \ldots, {\mathcal S}^X_{d} h^i\right]\right)|_Y^p\xrightarrow{d,m \to \infty}0.
 \end{aligned}
  \end{equation}
   \end{small}
Then, since $\|\mu\|_{k,q,p}  < \infty$, we can apply the dominated convergence theorem, to have $\|F-F_{d,m}\|_{\mathcal W^{k,p}_{B,\mu}(X,Y)}\xrightarrow{d,m \to \infty}0$. Then $F\in C^{k;p}_{B,\mu,CC}(X,Y)$.  The claim  follows. 

(vi) Let $F\in \widetilde C^{k;p}_{B,\mu}(X,Y)$. By \eqref{eq:sobolev_seminorm_2_O_geq}, we have $\|F\|_{\mathcal W^{k,p}_{B,\mu}(X,Y)}<\infty$. Let $\epsilon>0$. By  Assumption \ref{ass:density_CB}, there exists $\hat F \in {C^{k;1,0}_{B}(X,Y)}\equiv C^{k;p,0}_{B}(X,Y)\subset C^{k;p}_{B,\mu,CC}(X,Y)$ (recall also \eqref{eq:Ckwq_subset_Ckwc} with $q=0$)  such that 
    $$ \|F- \hat F\|_{\mathcal W^{k,p}_{B,\mu}(X,Y)}<\epsilon/2.$$ Then, by definition of the space $C^{k;p}_{B,\mu,CC}(X,Y)$, there {exist} $d,m$ such that $$\|\hat F-\hat F^{{\mathcal E}_d^X,{\mathcal D}^Y_m}\|_{\mathcal W^{k,p}_{B,\mu}(X,Y)}= \|\hat F-\hat F_{d,m}\|_{\mathcal W^{k,p}_{B,\mu}(X,Y)}<\epsilon/2.$$Therefore, we can construct a sequence $ \left\{ F_n \right\}_{n\in \mathbb N}\subset  C^{k;p}_{\mu^0,Cyl}(X,Y)$ such that $\|F-F_n\|_{\mathcal W^{k,p}_{B,\mu}(X,Y)}\xrightarrow{n\to \infty}0$.
\end{proof}
   \begin{remark}\label{rem:Sobole_sp_norm_modified}
Let $r\geq 1,$ $\nu$ be a finite measure on $X$ and $\eta$ be a probability measure on $X^k$ with $\|\eta\|_{k,r}<\infty$.  Then the corresponding statements to the ones in Propositions \ref{prop:relations_sobolev_norms_1},  \ref{pro:ass_density_holds} and Theorem \ref{th:relations_sobolve_norms} hold for  $C^{k;p,r}_{B,\nu,\eta,A}(X,Y)$, $C^{k;p,r}_{B,\nu,\eta,CC}(X,Y)$, $C^{k;p,q}_{B}(X,Y)$, $ C^{k;p,r}_{B,\nu,\eta}(X,Y),$ $\widetilde C^{k;p,r}_{B,\nu,\eta}(X,Y)$. The proofs follow the same steps. For instance, to prove \eqref{eq:bound_Wmu_barmu}, we adapt the calculation in the proof there as follows:
        \begin{align*}
  & \int_{X }\left(\int_{ X^i}  |D^iF(x)(h^1,...,h^i)|^r_Yd\eta^{1:i}(h^1,\ldots,h^i)\right)^{p/r}d\nu(x)\\
  \leq & \int_{X }  \|D^iF(x)\|_{\mathcal L^i(X,Y)}^p\left(\int_{ X^i} \prod_{t=1}^i |h^t|_X^rd\eta^{1:i}(h^1,\ldots,h^i)\right)^{p/r}d\nu(x)
   \leq  \|\eta\|_{k,r}^{p/r}\int_{X }  \|D^iF(x)\|_{\mathcal L^i(X,Y)}^pd \nu(x).
    \end{align*}  
   To prove $\|F-F_{d,m}\|_{\mathcal W^{k,p,r}_{B,\nu,\eta}(X,Y)}\xrightarrow{d,m \to \infty}0$ in the corresponding part of the proof of Theorem \ref{th:relations_sobolve_norms} (v), we proceed as follows: by \eqref{eq:conv_DiDdm} with $p$ replaced by $r$, we have
    $\int_{X^i}|[D^iF-D^iF_{d,m}](x)(h^1,...,h^i)|_Y^r d\eta^{1:i}(h^1,\ldots,h^i) \xrightarrow{d,m \to \infty}0,$ for all $x \in X$. Then, again by dominated convergence, $\|F-F_{d,m}\|_{\mathcal W^{k,p,r}_{B,\nu,\eta}(X,Y)}\xrightarrow{d,m \to \infty}0$.
   \end{remark}
   \subsection{Consistency with Gaussian Sobolev spaces}\label{subsec:gaussian_recover}
   {When the underlying measures are regular enough for closing derivative operators, we can complete the weighted Bastiani--Sobolev space. As an example, we show consistency with the classical Gaussian Sobolev spaces \cite{bogachev1998gaussian}.}
\begin{theorem}[Consistency with  Gaussian Sobolev spaces]\label{th:characterization_bogachev}
Let $\gamma$ be a centered Gaussian measure on a separable Banach space $X$,  $\mu:=\otimes_{i=0}^k\gamma$, $\eta:=\otimes_{i=1}^k\gamma$, $r=2$, and $p\geq 1.$ Let $Y$ be a separable Hilbert space. Then:
    \begin{align}
  &W_{\gamma}^{k,p}(X,Y) =\widehat{C_{b,Cyl}^\infty(X,Y)} =\widehat{C^{k;p}_{\gamma,\mathrm{Cyl}}(X,Y) }=\widehat{C^{k;p,2}_{B,\gamma,\eta,A}(X,Y) }=\widehat{C^{k;p,2}_{B,\gamma,\eta}(X,Y) }=D^{k,p}_{\gamma}(X,Y)\label{eq:gaussian_sobolev_characterization}\\
  & W_{\gamma}^{k,2}(X,Y) =\widehat{C_{b,Cyl}^\infty(X,Y) }=\widehat{C^{k;2}_{\gamma,\mathrm{Cyl}}(X,Y) }=\widehat{C^{k;2}_{B,\mu,A}(X,Y)}=\widehat{C^{k;2}_{B,\mu}(X,Y)}=D^{k,2}_{\gamma}(X,Y),\label{eq:gaussian_sobolev_characterization_p=2}
\end{align}
where the completions in \eqref{eq:gaussian_sobolev_characterization}, \eqref{eq:gaussian_sobolev_characterization_p=2} are taken, respectively, under $\|\cdot\|_{\mathcal W^{k,p,2}_{B,\gamma,\eta}(X,Y)}$,  $\|\cdot\|_{\mathcal W^{k,2}_{B,\mu}(X,Y)}$. In particular:
    \begin{align}\label{eq:A=ambient_gauss}
C^{k;p,2}_{B,\gamma,\eta,A}(X,Y) =C^{k;p,2}_{B,\gamma,\eta}(X,Y),\quad   C^{k;2}_{B,\mu,A}(X,Y)=C^{k;2}_{B,\mu}(X,Y)
\end{align}
\end{theorem}
\begin{remark}
We refer to \cite[Proposition 5.4.5, 5.4.6]{bogachev1998gaussian} for methods based on conditional expectations with respect to finite-dimensional sigma-fields generated by the first
Cameron--Martin coordinates providing cylindrical approximations 
for maps in $D^{k,p}_{\gamma}(X,Y)(=W_{\gamma}^{k,p}(X,Y))$  converging in Gaussian Sobolev norms.
\end{remark}
\begin{remark}
   In view of Example \ref{ex:gaussian_measures_radonifying}, when $Y$ is also a Banach space (possibly UMD), one can try to extend Theorem \ref{th:characterization_bogachev}, by showing consistency with the Gaussian Sobolev space defined in \cite{pronk-veraar}.
\end{remark}
\begin{proof}Recall first the norm identification \eqref{eq:sobolev_norm_mu0_Lp_bogachev_HS_p=2}, as well as \eqref{eq:cyl_in_A} (notice that Assumption \ref{ass:Radon-nik} holds by Example \ref{ex:ass_radon-nik}) and  \eqref{eq:Wp=Dp}, and consider the inclusions 
\begin{equation}\label{eq:inclusions_gaussian_proof}
C_{b,\mathrm{Cyl}}^\infty(X,Y) \subset C^{k;p}_{\gamma,\mathrm{Cyl}}(X,Y)\subset   C^{k;p,2}_{B,\gamma,\eta,A}(X,Y) \subset C^{k;p,2}_{B,\gamma,\eta}(X,Y) \subset D^{k,p}_{\gamma}(X,Y)=W^{k,p}_{\gamma}(X,Y).
\end{equation}
Taking the completion and using \eqref{eq:gaussian_sobolev}, we have  \eqref{eq:gaussian_sobolev_characterization}. 
Notice that, to get inclusion $C^{k;p,2}_{B,\gamma,\eta}(X,Y) \subset D_{\gamma}^{k,p}(X,Y)$, we proceed as follows. Let $F\in C^{k;p,2}_{B,\gamma,\eta}(X,Y)$. We    first claim     that, for $0\leq i\leq k$, the    map $D^i_\gamma F:X\to\mathcal L^i_2(X_\gamma,Y)$ (recall \eqref{eq:def_sobolev_norm_mu0_Lp_bogachev_HS}) is continuous (consequently, each $D^i_\gamma F$ is Borel measurable, hence Bochner measurable, since     $\mathcal L^i_2(X_\gamma,Y)$ is separable). This is trivial for $i=0$. Let $1\leq i\leq k$, let $x_n\to x$ in $X$, and set $A_n:=D^iF(x_n)$, $A:=D^iF(x)$. By continuous Bastiani differentiability, $A_n(u)\to A(u)$ for every $u\in X^i$; hence, by 
the Banach--Steinhaus theorem     for multilinear operators, $\sup_n\|A_n\|_{\mathcal L^i(X,Y)}<\infty$.       Therefore, by \eqref{eq:norm2=HS}     and dominated convergence (the integrand below is      bounded by $(\sup_n\|A_n\|_{\mathcal L^i(X,Y)}+\|A\|_{\mathcal L^i(X,Y)})^2\prod_{l=1}^i|u_l|_X^2$,    which is $\otimes_{l=1}^i\gamma$-integrable),
$$\|D^i_\gamma F(x_n)-D^i_\gamma F(x)\|^2_{\mathcal L^i_2(X_\gamma,Y)}=\int_{X^i}|(A_n-A)(u_1,\ldots,u_i)|_Y^2\,d\big(\otimes_{l=1}^i\gamma\big)(u_1,\ldots,u_i)\xrightarrow{n\to\infty}0,$$
which proves the claim.

Now let $0\leq i\leq k-1$, $x\in X$ and $h\in X_\gamma$. By the $Y$-valued fundamental theorem of calculus and identifying $\mathcal L^{i+1}_2(X_\gamma,Y)\cong\mathcal L_2(X_\gamma,\mathcal L^i_2(X_\gamma,Y))$, we have
$$D^i_\gamma F(x+th)-D^i_\gamma F(x)=\int_0^tD^{i+1}_\gamma F(x+sh)(h,\cdot)\,ds,\quad t\in\mathbb R,$$
 where  the right-hand side is an $\mathcal L^i_2(X_\gamma,Y)$-valued Bochner integral. Hence, $D^i_\gamma F$ is Gâteaux differentiable along $X_\gamma$ with derivative $D^{i+1}_\gamma F$, so that it is stochastically Gâteaux differentiable, and it is ray absolutely continuous. Since, moreover, $\int_X\|D^i_\gamma F\|^p_{\mathcal L^i_2(X_\gamma,Y)}d\gamma\leq\|F\|^p_{\mathcal W^{k,p,2}_{B,\gamma,\eta}(X,Y)}<\infty$ for $i\leq k$, by \eqref{eq:sobolev_norm_mu0_Lp_bogachev_HS_p=2}, we have $D^i_\gamma F\in D^{1,p}_\gamma(X,\mathcal L^i_2(X_\gamma,Y))$ with stochastic derivative $D^{i+1}_\gamma F$ for all $0\leq i\leq k-1$, i.e. $F\in D^{k,p}_\gamma(X,Y)$ by the inductive definition of $D^{k,p}_\gamma(X,Y)$. Then, \eqref{eq:gaussian_sobolev_characterization} follows.

Equation \eqref{eq:gaussian_sobolev_characterization_p=2} is a particular case of  \eqref{eq:gaussian_sobolev_characterization} by \eqref{eq:sobolev_norm_mu0_Lp_bogachev_HS_p=2} and the equality $C^{k;2}_{B,\mu,A}(X,Y)\equiv C^{k;2,2}_{B,\gamma,\eta,A}(X,Y)$. Finally \eqref{eq:A=ambient_gauss} follows from the inclusion $C^{k;p,2}_{B,\gamma,\eta}(X,Y)\subset  W^{k,p}_{\gamma}(X,Y)$  (recall \eqref{eq:inclusions_gaussian_proof}) and the definition of $W^{k,p}_{\gamma}(X,Y)$.
\end{proof}
\subsection{Validity of Assumption \ref{ass:density_CB}}\label{subsubsec:ass_density_CB}
Assumption \ref{ass:density_CB} is a suitable generalization to our Sobolev spaces in infinite dimensions of the following condition used in \cite[Proof of Theorem 4]{hornik1991approximation} when $X=\mathbb R^N$: \textit{the space of functions  
 $f \in C^k(\mathbb R^N)$, which, together with all their derivatives up to order $k$,  are bounded, is dense in $C^{k;p}_\nu(\mathbb R^{N},\mathbb R)$.  } Here $\nu$ is a finite measure on $\mathbb R^N.$ This condition holds true in the finite dimensional case and similar conditions are typically used  in the approximation theory of standard Sobolev spaces in $\mathbb R^N$ \cite[Section 3.19]{adams2003sobolev}. We notice that typical  proofs of these conditions make use of smooth bump functions (recall Appendix \ref{sec:bump}), which  are  easily  available for $X=\mathbb R^N$.

Unfortunately, when $X,Y$ are infinite-dimensional and with our new Sobolev spaces (and norms), this is no longer automatically guaranteed. However, under the assumption that the Banach space $X$ has suitable $k$-times differentiable bump functions (recall Appendix \ref{sec:bump}), we can  construct relevant cases where Assumption \ref{ass:density_CB} holds.
\begin{assumption}\label{ass:bump}
    Assume that there exists a bump function $b \in C^k_B(X)$  such that, denoting the rescaled bump function $b_{\eta}(x):=b(\eta x)$, $\eta>0,$ we have
    $b_{\eta}(x)\equiv 1$ if $|x|_X \leq 1 / \eta$ and there exists $C>0$ such that 
    \begin{equation}\label{eq:Db_eta_bdd}
        \left\|D^i b_{\eta}(x)\right\|_{\mathcal L^i(X,\mathbb R)} \leq C,\quad \forall x\in X, \forall \eta \leq 1,i\leq k.
    \end{equation}
\end{assumption}
\begin{remark}\label{rem:prop_rescaledd_bump}
    Recall Proposition \ref{prop:bump_rescaled} and Example \ref{ex:norm_Fréchet} for relevant cases under which Assumption \ref{ass:bump} is satisfied.
\end{remark}
\begin{proposition}\label{pro:ass_density_holds}
    Let Assumption \ref{ass:bump} hold. Let $\mu$  be a finite measure on $X\times X^k$ and assume that $\|\mu\|_{k,0,p}<\infty$. Then Assumption \ref{ass:density_CB} holds.
\end{proposition}
\begin{proof} We adapt the classical finite-dimensional arguments in \cite{adams2003sobolev} to our  Bastiani--Sobolev spaces on Banach spaces.

\emph{Step 1.} Let $F\in  \widetilde C^{k;p}_{B,\mu}(X,Y)$. We define and study suitable approximations 
    $F_{\eta}\in C^{k;1,0}_{B}(X,Y),$ for all $ \eta>0.$

Define $F_{\eta}:= F\cdot b_{\eta} \in C^{k}_{B}(X,Y)$, $\eta>0$.  By \cite[Chapter 3, Theorem 2.2]{bastiani1964} and Leibniz rule, we have
   $$D^i F_{\eta}(x)(h^1,...,h^i)=\sum_{J\subset \{1,\ldots,i\}}D^{|J|} F(x)((h^j)_{j\in J})D^{i-|J|}b_{\eta}(x)((h^t)_{ t\in J^c}), \quad i\leq k, $$
   $\operatorname{supp} F_\eta$  is bounded, and $F=F_\eta$ on $A_{\eta}:=\{x \in X:|x|_X\leq 1 / \eta\}$.
    Since $D^iF$ is  bounded on bounded subsets of $X$, $D^iF$ is bounded on $\operatorname{supp} b_\eta;$ then, for all $\eta >0$ there exists $C_\eta>0$ such that $\|D^i F_{\eta}(x)\|_{\mathcal L^i(X,Y)}\leq C_\eta$, for all $x\in X$, $i\leq k$. Then $F_{\eta}\in C^{k;1,0}_{B}(X,Y),$ for all $  \eta>0$. Moreover,  there exists $C>0$ such that, for all $A\subset X$ Borel, for all $0<\eta \leq 1$, 
    \begin{equation}\label{eq:veta-let-tildev}
        \left\|F_{\eta}\right\|_{\mathcal W^{k,p}_{B,\mu}(A,X,Y)} \leq C\left\|F\right\|_{\widetilde{\mathcal W}^{k,p}_{B,\mu}(A,X,Y)}.
    \end{equation} 
    Indeed, by \eqref{eq:Db_eta_bdd}, for all $0<\eta \leq 1$ and $i \leq k$, we have 
    \begin{equation*}
\begin{aligned}
    \left|D^i F_{\eta}(x)(h^1,...,h^i)\right|_Y&=\left|\sum_{J\subset \{1,\ldots,i\}}D^{|J|} F(x)((h^j)_{j\in J})D^{i-|J|}b_{\eta}(x)((h^t)_{t \in J^c})  \right|_Y\\
    &\leq C \sum_{J\subset \{1,\ldots,i\}}\left|D^{|J|} F(x)((h^j)_{j\in J})\right|_Y\prod_{t \in J^c}|h^t|_X.
\end{aligned}
  \end{equation*} Taking the $p$-th power  and integrating over $A\times X^i$, we have
  \begin{align*}
   &\int_{A \times X^i} \left|D^i F_{\eta}(x)(h^1,...,h^i)\right|_Y^pd\mu^{0:i}(x,h^1,\ldots,h^i)\\
   \leq &C \sum_{J\subset \{1,\ldots,i\}}\int_{A \times X^i} \left|D^{|J|} F(x)((h^j)_{j\in J})\right|_Y^p\prod_{t \in J^c}|h^t|_X^pd\mu^{0:i}(x,h^1,\ldots,h^i)\\
   \leq & C  \sum_{J\subset \{1,\ldots,i\}}  \int_{A \times X^i}\left[\left|D^{|J|} F(x)((h^j)_{j\in J})\right|_Y^p\prod_{t \in J^c}|h^t|_X^p\right]d\mu^{0:i}(x,h^1,\ldots,h^i),
\end{align*}
Summing over $i$, we have \eqref{eq:veta-let-tildev}.

 \emph{Step 2.} We prove the statement, as follows: by  \eqref{eq:sobolev_seminorm_2_O_geq}, \eqref{eq:veta-let-tildev}, we have
$$
\begin{aligned}
\left\|F-F_{\eta}\right\|_{\mathcal W^{k,p}_{B,\mu}(X,Y)}^p & =\left\|F-F_{\eta}\right\|_{\mathcal W^{k,p}_{B,\mu}(A_\eta^c,X,Y)}^p  \leq C (\left\|F\right\|_{\mathcal W^{k,p}_{B,\mu}(A_\eta^c,X,Y)}^p+\left\|F_{\eta}\right\|_{\mathcal W^{k,p}_{B,\mu}(A_\eta^c,X,Y)}^p)\leq C\left\|F\right\|_{\widetilde{\mathcal W}^{k,p}_{B,\mu}(A_\eta^c,X,Y)}^p\\
&=C\sum_{0\leq i \leq k}\sum_{J\subset \{1,\ldots,i\}}\left[\int_{X \times X^i}I_{A_\eta^c}(x)  \left|D^{|J|} F(x)((h^j)_{j\in J})\right|_Y^p\prod_{t \in J^c}|h^t|_X^pd\mu^{0:i}(x,h^1,\ldots,h^i)\right]\xrightarrow{\eta\to 0}0 ,
\end{aligned}
$$
where we have used the dominated convergence theorem. We observe that the dominated convergence theorem can be applied, for all $J\subset \{1,\ldots,i\},0\leq  i \leq k$, since $A_{\eta}^c=\{x \in X:|x|> 1 / \eta\}$, we have
\begin{align*}
&I_{A^c_\eta}(x)  \left|D^{|J|} F(x)((h^j)_{j\in J})\right|_Y^p\prod_{t \in J^c}|h^t|_X^p\xrightarrow{\eta\to 0}0, &\forall x,h^1,...,h^i\in X,\\
   & I_{A^c_\eta}(x)  \left|D^{|J|} F(x)((h^j)_{j\in J})\right|_Y^p\prod_{t \in J^c}|h^t|_X^p\leq \left|D^{|J|} F(x)((h^j)_{j\in J})\right|_Y^p\prod_{t \in J^c}|h^t|_X^p,& \forall x,h^1,...,h^i\in X,
\end{align*}
where the right-hand-side of the second inequality is $\mu^{0:i}$-integrable,  since $F\in  \widetilde C^{k;p}_{B,\mu}(X,Y)$.
\end{proof}
\subsection{Strict inclusions for Bastiani--Sobolev spaces}\label{subsec:ambient-not-approximable}
The following example shows that we can construct maps that are in the ambient space but not in the Bastiani--Sobolev space, while the Bastiani--Sobolev space does not only contain cylindrical maps.
\begin{example}\label{ex:strict_inclusions}
Let $X$ be a Banach space without the approximation property (Assumption \ref{ass:approx_pro}), and let $k\geq1$, $p\geq1${, and let $\nu\neq0$ be a finite Borel measure on $X$ with $\int_X|x|^p_X\,d\nu(x)<\infty$}. Then there {exists} a probability measure $\eta$ on $X^k${, with $\|\eta\|_{k,p}<\infty$, such that the finite measure} $\mu=\nu\otimes\eta$ on $X\times X^k$ {satisfies} Assumption \ref{ass:Radon-nik}, such that
setting   $ F\in {C^\infty_F}(X,X)$, $F(x):=x$, and $\tilde F(x)=\tilde L x$, for a compact infinite-rank operator $\tilde L\in\mathcal A(X)\cap\mathcal F(X)^c$, we have
\begin{equation}\label{eq:cantapproximate-in-bastiani-sobolev}
 \begin{aligned}
&\| F\|_{\mathcal W^{k,p}_{B,\mu}(X,X)}<\infty, \quad \inf_{G\in C^{k;p}_{\nu,Cyl}(X,X)}
 \|F-G\|_{\mathcal W^{k,p}_{B,\mu}(X,X)}>0,\quad  \\
 & \|\tilde F\|_{\mathcal W^{k,p}_{B,\mu}(X,X)}<\infty,\quad \inf_{G\in C^{k;p}_{\nu,Cyl}(X,X)}
 \|\tilde F-G\|_{\mathcal W^{k,p}_{B,\mu}(X,X)}=0.
 \end{aligned}
\end{equation}
Consequently, $F\in C^{k;p}_{B,\mu}(X,X) \setminus C^{k;p}_{B,\mu,A}(X,X)$, whereas  $ \tilde F\in C^{k;p}_{B,\mu,A}(X,X)\setminus C^{k;p}_{\nu,Cyl}(X,X)$ {(the latter without taking a quotient in Bastiani--Sobolev spaces; otherwise, we ask non-degeneracy of the measures)}, so that the following inclusions are strict:
$$C^{k;p}_{\nu,Cyl}(X,X)\subsetneq C^{k;p}_{B,\mu,A}(X,X)\subsetneq C^{k;p}_{B,\mu}(X,X).$$
The same statements hold for $C^{k;p,r}_{B,\nu,\eta,A}(X,X), C^{k;p,r}_{B,\nu,\eta}(X,X)$.

We first show \eqref{eq:cantapproximate-in-bastiani-sobolev}.
By the failure of AP, there exist a compact $K\subset X$, $\delta>0$ such that $\sup_{h\in K}|(I-T)h|_X\geq\delta$ for every $T\in\mathcal F(X)$. Setting $Z:=C(K,X)$ and $\mathcal T:=\overline{\{T|_K:T\in\mathcal F(X)\}}^{Z}${, a closed subspace of $Z$,} we have {$\operatorname{dist}(I|_K,\mathcal T):=\inf _{G \in \mathcal T}| I|_K-G |_{Z}\geq\delta$, whence $I|_K\notin \mathcal T$}. By a corollary of the Hahn--Banach theorem \cite[Proposition~2.7, p.~57]{fabian-etal}, there exists $\Lambda\in Z^*$ such that $\|\Lambda\|_{Z^*}=1$, $\Lambda|_\mathcal T=0$, and $\Lambda(I|_K)=\operatorname{dist}(I|_K,\mathcal T)\geq \delta$.
By Singer's representation theorem \cite{Dinculeanu1967} there exists a regular $X^*$-valued Borel measure $m$ on $K$, of bounded variation, such that
\begin{equation}\label{eq:singer-domination}
 \Lambda(H)=\int_K H\,dm,\qquad
 |\Lambda(H)|\leq\int_K|H(h)|_X\,d\rho(h),\qquad
 \rho(K)=\|\Lambda\|_{Z^*}=1,
\end{equation}
where its total-variation measure $\rho:=|m|$ is a probability measure on $X$ supported on $K$. Define
$\eta:=\otimes_{i=1}^k\rho$ {and $\mu:=\nu\otimes\eta$; since $\rho$ is supported on the compact    set $K$,  we have  $\|\eta\|_{k,p}<\infty$ and, $\eta$ being a probability measure on $X^k$, Assumption \ref{ass:Radon-nik} holds. We also set $M:=\int_X|x|^p_X\,d\nu(x)+\nu(X)\int_X|h|^p_X\,d\rho(h)<\infty$.}
Since $F(x)=x$, $DF(x)=I$, and $D^iF=0$ for $i\geq2$, we have $\|F\|_{\mathcal W^{k,p}_{B,\mu}(X,X)}^p={M}<\infty$. Hence ${F}\in C^{k;p}_{B,\mu}(X,X)$. Let $G\in C^{k;p}_{\nu,Cyl}(X,X)$ and put {$T_G(x):=DG(x)\in\mathcal F(X)$, $x\in X$; recall that $T_G(x)$ has finite rank for every $x\in X$, $G$ being cylindrical}. Since {$T_G(x)|_K$}$\in \mathcal T$, we have  {$|\Lambda((I-T_G(x))|_K)|=|\Lambda(I|_K)-\Lambda(T_G(x)|_K)|=|\Lambda(I|_K)|\geq \delta$, for every $x\in X$}. Therefore, {retaining only the term $i=1$ in $\|\cdot\|_{\mathcal W^{k,p}_{B,\mu}(X,X)}$ and using Jensen's inequality ($\rho$ being a probability measure  and  $p\geq1$)  together with} \eqref{eq:singer-domination}, we have
{
\begin{align*}
\|F-G\|_{\mathcal W^{k,p}_{B,\mu}(X,X)}^p&\geq \int_X\left(\int_K|(I-T_G(x))h|_X^p\,d\rho(h)\right)d\nu(x)\geq \int_X\left(\int_K|(I-T_G(x))h|_X\,d\rho(h)\right)^p d\nu(x)\\
&\geq\int_X|\Lambda((I-T_G(x))|_K)|^p\,d\nu(x)\geq \delta^p\,\nu(X)>0.
\end{align*}
}
It remains to construct $\tilde F$. First, \cite[Theorem 4.19, p. 191]{fabian-etal} gives a basic sequence $\{e_n\}\subset X$, i.e. $\{e_n\}$ is a Schauder basis for $X_0:=\overline{\operatorname{span}\{e_n:n\geq1\}}$. Let $\{\varphi_n^0\}\subset X_0^*$ be its coordinate functionals, and choose ${C>0}$ such that $\|\varphi_n^0\|_{X_0^*} |e_n|_{X} \leq C$ for every $n$ (see \cite[Theorem 3.1, p.20]{singer1981}). By  the Hahn--Banach  theorem, extend each $\varphi_n^0$ to $\varphi_n\in X^*${, with $\|\varphi_n\|_{X^*}=\|\varphi_n^0\|_{X_0^*}$}. Following standard constructions (e.g. \cite[Ex. 5.75]{fabian-etal}), we  define the  linear maps
 $\tilde L(x):=\sum_{n=1}^\infty2^{-n}\varphi_n(x)e_n,$ $
 \tilde L_N(x):=\sum_{n=1}^N2^{-n}\varphi_n(x)e_n$, to have
$$|[\tilde L-\tilde L_N](x)|_{X}\leq \sum_{n>N}2^{-n}|\varphi_n(x)| |e_n|_X \leq  \sum_{n>N}2^{-n}\|\varphi_n\|_{X^*} |e_n|_X |x|_X \leq C|x|_X\sum_{n>N}2^{-n},$$
where we have used that  {$\|\varphi_n\|_{X^*} |e_n|_X=\|\varphi_n^0\|_{X_0^*} |e_n|_X\leq C$}
so that $\|\tilde L-\tilde L_N\|_{\mathcal L(X)}\leq C\sum_{n>N}2^{-n}\xrightarrow{N\to \infty}0$ {and $\tilde L\in\mathcal A(X)$; in particular, $\tilde L$ is compact. Moreover, $\varphi_n(e_m)=\varphi_n^0(e_m)=\delta_{nm}$, so that $\tilde Le_m=2^{-m}e_m$ for every $m\in\mathbb N$ and $\tilde L$ has infinite rank; hence $\tilde L\in\mathcal A(X)\cap\mathcal F(X)^c$ and $\tilde F(x):=\tilde L x$ is not a cylindrical map from $X$ to $X$,  cylindrical maps  taking values in a finite-dimensional subspace of $X$}. Moreover, $\tilde F\in C_B^\infty(X,X)$, and $\|\tilde F\|_{\mathcal W^{k,p}_{B,\mu}(X,X)}^p{{}\leq \|\tilde L\|_{\mathcal L(X)}^pM}<\infty$. Instead $\tilde L_N$ is finite rank, so that, setting $\tilde F_N(x)=\tilde L_N x$, we have $\tilde F_N\in C^{k;p}_{\nu,Cyl}(X,X)$ {(here $\int_X|x|^p_X\,d\nu(x)<\infty$ is used)}. Since {$\tilde F(x)-\tilde F_N(x)=(\tilde L-\tilde L_N)x$}, $D\tilde F=\tilde L,D\tilde F_N=\tilde L_N$, $D^i\tilde F=D^i\tilde F_N=0$, $i\geq 2,$ we have
{
\begin{equation}\label{eq:compact-map-sobolev-convergence}
 \|\tilde F-\tilde F_N\|_{\mathcal W^{k,p}_{B,\mu}(X,X)}^p
 \leq \|\tilde L-\tilde L_N\|_{\mathcal L(X)}^p\,M\xrightarrow{N\to\infty}0.
\end{equation}
}
Therefore $ \tilde F\in C^{k;p}_{B,\mu,A}(X,X)\setminus C^{k;p}_{\nu,Cyl}(X,X)$. 
\end{example}
\section{Proofs of Universal Approximation Theorems}\label{sec:proofs}
\begin{proof}[Proof of Theorem \ref{th:counterexample}]
{\emph{Case (a).}}
    Let $L\in \mathcal L(X,Y) \cap \mathcal A(X,Y)^c$ and let $F\in {C^{\infty}_F}(X,Y),F(x):=Lx$ (hence $DF(x)\equiv L, D^iF(x)\equiv 0,$ for all $i\geq 2$). Let {$G\in\mathcal N$}. Fix $x\in X$; then ${DG(x)}\in \mathcal A(X,Y)$ {by (a)}, so that 
    $$\left \|DF(x)-{DG(x)}\right\|_{\mathcal L(X,Y)}=\left\|L-{DG(x)}\right\|_{\mathcal L(X,Y)}\geq \inf_{T\in \mathcal A(X,Y)}\left\|L-T\right\|_{\mathcal L(X,Y)}\geq \delta,$$
     where we have used the fact  that
     $\mathcal A(X,Y)$ is a closed subspace of $\mathcal L(X,Y)$ to find $\delta>0$. Then (i) follows. To show (ii), let  $p\geq 1$ and let $\nu\neq 0$ be a finite measure on $X$ such that $\int_{X} |x|_X^pd\nu(x)<\infty$; it follows that $\int_{X} |Lx|_Y^pd\nu(x)<\infty$.
     Then
 $F\in {C^{k;p}_{F,\nu}}(X,Y)$ (since $\|F\|_{\mathcal W^{k,p}_{\nu}(X,Y)}^p=\left(\int_{X } |Lx|_Y^pd\nu (x)+\int_{X } \|L\|_{\mathcal L(X,Y)}^pd\nu\right)<\infty$). 
However:  $$\|F-{G}\|^p_{\mathcal W^{k,p}_{\nu}(X,Y)}\geq \int_X \left \|DF(x)-{DG(x)}\right\|_{\mathcal L(X,Y)}^p d\nu(x) \geq \delta^p \nu(X).$$ 
The claim follows by the arbitrariness of {$G\in\mathcal N$}.

{\emph{Case (b).}} {Notice that, for $x_0\in X$, the set $\mathfrak D_{x_0}:=\{DF^{c,\zeta}(x_0):c\in\mathbb N,\ \zeta\in\mathcal Z_c\}$ and, if $X$ is separable, the set $\mathfrak D:=\{DF^{c,\zeta}(x):c\in\mathbb N,\ \zeta\in\mathcal Z_c,\ x\in X\}$ are separable in $\mathcal L(X,Y)$, as countable unions of continuous images (Definition \ref{def:CP_class}) of the separable metric spaces $\mathcal Z_c\times\{x_0\}$ and $\mathcal Z_c\times X$, respectively.} {Thus, since $\mathcal L(X,Y)$ is nonseparable, the closure of $\mathfrak D_{x_0}$ and, if $X$ is separable, the closure of $\mathfrak D$ are proper subsets of $\mathcal L(X,Y)$. Hence there exist $L_0\in\mathcal L(X,Y)$ with $\varepsilon_0:=\inf_{T\in\mathfrak D_{x_0}}\|L_0-T\|_{\mathcal L(X,Y)}>0$ and, if $X$ is separable, $L\in\mathcal L(X,Y)$ with $\varepsilon:=\inf_{T\in\mathfrak D}\|L-T\|_{\mathcal L(X,Y)}>0$.}

{(i) Let $F(x):=L_0x$, so that $DF\equiv L_0$. Then $\|DF(x_0)-DF^{c,\zeta}(x_0)\|_{\mathcal L(X,Y)}=\|L_0-DF^{c,\zeta}(x_0)\|_{\mathcal L(X,Y)}\geq\varepsilon_0$ for all $c\in\mathbb N$ and $\zeta\in\mathcal Z_c$.}

{(ii) Let $F(x):=Lx$. As in {case (a)}, $F\in C^{k;p}_{F,\nu}(X,Y)$. Moreover, $\|DF(x)-DF^{c,\zeta}(x)\|_{\mathcal L(X,Y)}=\|L-DF^{c,\zeta}(x)\|_{\mathcal L(X,Y)}\geq\varepsilon$ for all $c\in\mathbb N$, $\zeta\in\mathcal Z_c$ and $x\in X$, whence the claim follows.}
\end{proof}
\begin{proof}[Proof of Theorem \ref{th:UAT_K_Banach}]We consider {the} equivalent family of  seminorms ${\mathcal P}_{B}^{k,co,0}$ in Definition \ref{def:compact_open_C2B}. 

(ii)  Fix a compact $K\subset X$ and $\epsilon>0$. Since ${\mathcal S}^X_\alpha\xrightarrow{co} I$, we can define a sequence $\{{\mathcal S}^X_{\alpha(d)}\}_{d\in \mathbb N}$, such that 
\begin{equation}\label{eq:proof_sequance_strongly_conv}
    \sup_{x \in K}|({\mathcal S}^X_{\alpha(d)}-I)x|\xrightarrow{d\to\infty} 0.
\end{equation}
We denote the compact set $\tilde K := \overline{\bigcup_{d=1}^\infty {\mathcal S}^X_{\alpha(d)}(K)}\subset X$ (recall  Lemma \ref{rem:relatively-compact_Pd-I_Banach}). For all $i\leq k,$ we have
    \begin{align*}&\sup_{x \in K, h^1 \in K,...,h^i \in K}|[D^iF(x)-D^iF^{{\mathcal E}^X_{\alpha(d)},\theta,{\mathcal D}^Y_\beta}(x)](h^1,...,h^i)|\\
    \leq &\sup_{x \in K, h^1 \in K,...,h^i \in K}|D^i F(x)\left[ h^1, \ldots,  h^i\right] -{\mathcal S}_{\beta}^YD^i F\left({\mathcal S}_{{\alpha(d)}}^X x\right)\left[{\mathcal S}_{{\alpha(d)}}^X h^1, \ldots, {\mathcal S}_{{\alpha(d)}}^X h^i\right] |\\
    &+\sup_{x \in K, h^1,...,h^i \in K}|[D^iF_{{\alpha(d)},\beta}\left(x\right)-D^iF^{{\mathcal E}^X_{\alpha(d)},\theta,{\mathcal D}^Y_\beta}(x)]\left[ h^1, \ldots,  h^i\right]|\\
     \leq& \sup_{x \in K, h^1 \in K,...,h^i \in K}|[I-{\mathcal S}_{\beta}^Y]D^i F\left(x\right)\left[ h^1, \ldots,  h^i\right]|\\
     & +\sup_{x \in K, h^1 \in K,...,h^i \in K}|{\mathcal S}_{\beta}^Y\left[D^i F(x)\left[ h^1, \ldots,  h^i\right] -D^i F\left({\mathcal S}_{{\alpha(d)}}^X x\right)\left[{\mathcal S}_{{\alpha(d)}}^X h^1, \ldots, {\mathcal S}_{{\alpha(d)}}^X h^i\right]\right ]|\\
    &+\sup_{x \in K, h^1,...,h^i \in K}|[D^iF^{{\mathcal E}^X_{\alpha(d)},{\mathcal D}^Y_\beta}\left(x\right)-D^iF^{{\mathcal E}^X_{\alpha(d)},\theta,{\mathcal D}^Y_\beta}(x)]\left[ h^1, \ldots,  h^i\right]|\\
=&:I_1^{\beta,i}+I_2^{{\alpha(d)},\beta,i}+I_3^{{\alpha(d)},\beta,\theta,i}.
    \end{align*}
\begin{itemize}
  \item For  $I_{1}^{\beta,i}$ we have 
            \begin{align*}
                I_{1}^{\beta,i}= \sup_{y \in Z_K^i}|[I-{\mathcal S}_{\beta}^Y]y|\leq \sup_{y \in Z_K}|[I-{\mathcal S}_{\beta}^Y]y|,
            \end{align*}
            where, exploiting $C^k_B$, $Z^i_{K}:=D^i F\left(K\right)\left[ K, \ldots,  K\right]\subset Y$ is a compact set as the continuous image of compact sets and $Z_K:=\bigcup_{i\leq k}Z_K^i$ is also compact. 
            \item For $I_{2}^{{\alpha(d)},\beta,i},$ we have
            \begin{align*}
                I_{2}^{{\alpha(d)},\beta,i}&\leq \|{\mathcal S}_{\beta}^Y\| \sup_{x \in K, h^1 \in K,...,h^i \in K}|\left[D^i F(x)\left[ h^1, \ldots,  h^i\right] -D^i F\left({\mathcal S}_{{\alpha(d)}}^X x\right)\left[{\mathcal S}_{{\alpha(d)}}^X h^1, \ldots, {\mathcal S}_{{\alpha(d)}}^X h^i\right]\right ]|\\
                &\leq \|{\mathcal S}_{\beta}^Y\| \sup_{x \in K, h^1 \in K,...,h^i \in K} \omega_{\tilde K}^i\left \{{\Big(}|(I-{\mathcal S}_{{\alpha(d)}}^X)x|^2+\sum^i_{j=1}|(I-{\mathcal S}_{{\alpha(d)}}^X)h^j|^2{\Big)^{1/2}}\right \}\\
                &\leq \|{\mathcal S}_{\beta}^Y\| \omega_{\tilde K}\left \{\sup_{x \in K, h^1 \in K,...,h^i \in K} \left[|(I-{\mathcal S}_{{\alpha(d)}}^X)x|^2+\sum^i_{j=1}|(I-{\mathcal S}_{{\alpha(d)}}^X)h^j|^2\right]^{{1/2}}\right \},
            \end{align*}
            where, exploiting $C^k_B$, $\omega_{\tilde K}^i$ is the modulus of continuity of $D^iF(\cdot)(\cdot,\ldots,\cdot)$ over the compact set $\tilde K^{i+1}$ and $\omega_{\tilde K}(r)=\max_{ i\leq k}{\omega_{\tilde K}^i(r)}$, $r\geq 0$ is again a modulus of continuity. 
    \item For $I_3^{{\alpha(d)},\beta,\theta,i},$ by \eqref{eq:derivativeDrv_dh_dY_Banach}, \eqref{eq:derivativeDrv_dh^theta_dY_Banach}, and exploiting Assumption \ref{ass:approx_pro}, for some $C_{K,{\alpha(d)},\beta}>0,$ we have 
    \begin{align*}
        I_3^{{\alpha(d)},\beta,\theta,i}
        &= \sup_{x \in K, h^1 \in K,...,h^i \in K}|{{\mathcal D}^Y _{\beta}}[D^if^{{N_{\alpha(d)}},N_\beta}\left({\mathcal E}^X _{{\alpha(d)}}(x)\right)-D^if^{N_{\alpha(d)},\theta,N_\beta}({\mathcal E}^X _{{\alpha(d)}}(x))]\left[{\mathcal E}^X _{{\alpha(d)}} h^1, \ldots, {\mathcal E}^X _{{\alpha(d)}} h^i\right]|\\
        &\leq C_{K,{\alpha(d)},\beta} \sup_{x \in K}\|D^if^{{N_{\alpha(d)}},N_\beta}\left({\mathcal E}^X _{{\alpha(d)}}(x)\right)-D^if^{N_{\alpha(d)},\theta,N_\beta}({\mathcal E}^X _{{\alpha(d)}}(x))\|_{\mathcal L^i(\mathbb R^{N_{\alpha(d)}},\mathbb R^{N_\beta})}\\
        &=C_{K,{\alpha(d)},\beta} \sup_{y \in {\mathcal E}^X _{{\alpha(d)}}(K)}\|D^if^{{N_{\alpha(d)}},N_\beta}\left(y\right)-D^if^{N_{\alpha(d)},\theta,N_\beta}(y)\|_{\mathcal L^i(\mathbb R^{N_{\alpha(d)}},\mathbb R^{N_\beta})}.
    \end{align*}
Notice that ${\mathcal E}^X_{{\alpha(d)}}(K)\subset \mathbb R^{N_{\alpha(d)}}$ is compact. 
\end{itemize}
Since ${\mathcal S}^Y_\beta\xrightarrow{co} I$ (recall Assumption \ref{ass:approx_pro}), we can find $\beta \in B$ such that $I_{1}^{\beta,i}<\epsilon/3$. For this $\beta$, using \eqref{eq:proof_sequance_strongly_conv},  we find $d\in \mathbb N $ such that 
$I_{2}^{{\alpha(d)},\beta,i}<\epsilon/3$, for all $i \leq k$. Then for these ${\alpha(d)},\beta $, by  Assumption \ref{ass:finite_dim_approx}, we can find $\theta \in \Theta_{N_{\alpha(d)},N_\beta}$ such that $I_3^{{\alpha(d)},\beta,\theta,i}<\epsilon/3$, for all $i \leq k$. The proof of (ii) follows.

{
(i) The proof follows by a similar finite-dimensional reduction procedure to \cite{llavona1986}.
 Let  $F\in C_B^k(X,Y)$,  $K\subset X$ be compact, and  $\epsilon>0$. Thanks to Assumption \ref{ass:approx_pro} and the joint continuity of $(x,h^1,\ldots,h^i)\mapsto D^iF(x)(h^1,\ldots,h^i)$, {with the same procedure as in the proof of \cite[Lemma 3.1.3]{llavona1986} (which only needs joint continuity)}, we find $\alpha\in A$ such that, for all $0\leq i\leq k$,
\begin{align*}
\sup_{x\in K,\,h^1\in K,\ldots,h^i\in K}
\left|
D^iF(x)[h^1,\ldots,h^i]
-
D^iF({\mathcal S}^X_\alpha x)
[{\mathcal S}^X_\alpha h^1,\ldots,{\mathcal S}^X_\alpha h^i]
\right|_Y
<\epsilon/3.
\end{align*}
Set $g:=F\circ{\mathcal D}^X_\alpha\in C^k(\mathbb R^{N_\alpha},Y)$ {(by the chain rule, $g\in C^k_B(\mathbb R^{N_\alpha},Y)$, hence $g\in {C^k_F}(\mathbb R^{N_\alpha},Y)$ by Remark \ref{rem:Bastiani} (iii))} and let $N'\in\mathbb N$, $e^Y_1,\ldots,e^Y_{N'}\in Y$, $\varphi_1,\ldots,\varphi_{N'}\in C^k(\mathbb R^{N_\alpha})$, $\varphi:=(\varphi_1,\ldots,\varphi_{N'})$. Then, for every $0\leq i\leq k$, we have
\begin{equation}\label{eq:reduction_compact_open_APX}
\begin{aligned}
&\sup_{x\in K,\,h^1\in K,\ldots,h^i\in K}
\left|
D^iF({\mathcal S}^X_\alpha x)
[{\mathcal S}^X_\alpha h^1,\ldots,{\mathcal S}^X_\alpha h^i]
-
{\mathcal D}^Y
\left(
D^i\varphi({\mathcal E}^X_\alpha x)
[{\mathcal E}^X_\alpha h^1,\ldots,{\mathcal E}^X_\alpha h^i]
\right)
\right|_Y
\\
&=\sup_{x\in K,\,h^1\in K,\ldots,h^i\in K}\left|
D^iF({\mathcal D}^X_\alpha{\mathcal E}^X_\alpha x)
[{\mathcal D}^X_\alpha{\mathcal E}^X_\alpha h^1,\ldots,{\mathcal D}^X_\alpha{\mathcal E}^X_\alpha h^i]
-
{\mathcal D}^Y
\left(
D^i\varphi({\mathcal E}^X_\alpha x)
[{\mathcal E}^X_\alpha h^1,\ldots,{\mathcal E}^X_\alpha h^i]
\right)
\right|_Y
\\
&=
\sup_{x\in K,\,h^1\in K,\ldots,h^i\in K}
\left|
\left[
D^ig({\mathcal E}^X_\alpha x)
-
D^i\widetilde g({\mathcal E}^X_\alpha x)
\right]
[{\mathcal E}^X_\alpha h^1,\ldots,{\mathcal E}^X_\alpha h^i]
\right|_Y
\\
&\leq
C_{\alpha,K}
\sup_{y\in{\mathcal E}^X_\alpha(K)}
\left\|
D^ig(y)-D^i\widetilde g(y)
\right\|_{\mathcal L^i(\mathbb R^{N_\alpha},Y)}
<\epsilon/3,
\end{aligned}
\end{equation}
where  we have applied \cite[Theorem 3.1.2 (1)]{llavona1986}
(i.e. $C^k(\mathbb R^{N_\alpha})\otimes Y$ is dense in
$C^k(\mathbb R^{N_\alpha},Y)^{co}$) to $g$ on the compact set ${\mathcal E}^X_\alpha(K)$, to find $N'\in\mathbb N$, $e^Y_1,\ldots,e^Y_{N'}\in Y$, and $\varphi_1,\ldots,\varphi_{N'}\in C^k(\mathbb R^{N_\alpha})$ such that, setting
$
\widetilde g(y):={\mathcal D^Y(\varphi(y))}=\sum_{j=1}^{N'}\varphi_j(y)e_j^Y
$ and 
$
{\mathcal D}^Y:\mathbb R^{N'}\to Y,
$ $
{\mathcal D}^Y(z):=\sum_{j=1}^{N'}z_je_j^Y,
$ 
the inequality \eqref{eq:reduction_compact_open_APX} holds (here, $C_{\alpha,K}
:=
\max_{0\leq i\leq k}
\sup_{h^1\in K,\ldots,h^i\in K}
\prod_{\ell=1}^i
\left|{\mathcal E}^X_\alpha h^\ell\right|_{\mathbb R^{N_\alpha}}$).
Finally, Assumption \ref{ass:finite_dim_approx} gives $\theta\in\Theta_{N_\alpha,N'}$ such that, for every $0\leq i\leq k$,
\begin{align*}
\sup_{x\in K,\,h^1\in K,\ldots,h^i\in K}
\left|
{\mathcal D}^Y
\left[
D^i\varphi({\mathcal E}^X_\alpha x)
-
D^if^{N_\alpha,\theta,N'}({\mathcal E}^X_\alpha x)
\right]
[{\mathcal E}^X_\alpha h^1,\ldots,{\mathcal E}^X_\alpha h^i]
\right|_Y
<\epsilon/3.
\end{align*}
This proves (i).
}
\end{proof}
\begin{proof}[Proof of Theorem \ref{th:UAT_Sobolev_Banach}] (i)
Let  $F \in C^{k;p}_{B,\mu,A}(X,Y)$, $\epsilon>0$.  By  definition, we pick $F^{{\mathcal E}^X,{\mathcal D}^Y}\in C^{k;p}_{\mu^0,Cyl}(X,Y)$ (with ${\mathcal E}^X \in \mathcal L( X , \R^{N}),$ for some $ N\in \mathbb N$ and ${\mathcal D^{Y}} \in \mathcal L( \R^{N'} , Y)$ for some $ N'\in \mathbb N$) such that
\begin{equation}\label{eq:F-FEDleqeps_sobo}
    \|F-F^{{\mathcal E}^X,{\mathcal D}^Y}\|_{\mathcal W^{k,p}_{B,\mu}(X,Y)}<\epsilon/2.
\end{equation}
    Consider $F^{{\mathcal E}^X,\theta,{\mathcal D}^Y}$, for $\theta \in \Theta_{N,N'}$. Then, using \eqref{eq:bound_Wmu_barmu}, \eqref{eq:sobolev-norm-ineq-cyl} (recall also that $\mathcal {A}_{ED}(X,Y)\subset Cyl(X,Y)$), we have
    \begin{align*}
\|F^{{\mathcal E}^X,{\mathcal D}^Y}- F^{{\mathcal E}^X,\theta,{\mathcal D}^Y}\|_{\mathcal W^{k,p}_{B,\mu}(X,Y)}
&\leq  C \|F^{{\mathcal E}^X,{\mathcal D}^Y}- F^{{\mathcal E}^X,\theta,{\mathcal D}^Y}\|_{\mathcal W^{k,p}_{ \mu^{0}}(X,Y)}\\
&\leq  C_{{\mathcal E}^X,{\mathcal D}^Y}  \  \| f^{N,N'}- f^{N,\theta,N'}\|_{\mathcal  W^{k,p}_{ \mu^{0,{\mathcal E}^X}}(\mathbb R^{N},\mathbb R^{N'})}<\epsilon/2,
    \end{align*}
    where we have used Assumption \ref{ass:finite_dim_sobolev_approx} to find $\theta \in \Theta_{N,N'}$ such that
    the above is true (Assumption \ref{ass:finite_dim_sobolev_approx} can be used since, by definition of the space $C^{k;p}_{\mu^0,Cyl}(X,Y)$ we have $\|f^{N,N'}\|_{\mathcal W^{k,p}_{\mu^{0,{\mathcal E}^X}}(\mathbb R^{N},\mathbb R^{N'})}<\infty$).
The claim follows.

{(ii) The proof of (ii) is similar to the proof of (i)  using Remark \ref{rem:Sobole_sp_norm_modified}.}
\end{proof}
\begin{proof}[Proof of Corollary \ref{th:UAT_Sobolev_K_Banach}]
(i) The density claim  follows directly from Theorem \ref{th:UAT_Sobolev_Banach} and \eqref{eq:Ckwq_subset_Ckwc}. We only need to show that UA is achieved taking ${\mathcal E}^X={\mathcal E}^X_d, {\mathcal D}^Y={\mathcal D}^Y_m$ for some $d,m$. 
Indeed, let $F \in C^{k;p,q}_{B}(X,Y)$, $\epsilon>0$. By  \eqref{eq:cyl_approx_is_cyl} and \eqref{eq:Ckwq_subset_Ckwc}, we choose $d,m$ such that 
\begin{equation*}
   \|F-F_{d,m}\|_{\mathcal W^{k,p}_{B,\mu}(X,Y)}\equiv \|F-F^{{\mathcal E}_d^X,{\mathcal D}_m^Y}\|_{\mathcal W^{k,p}_{B,\mu}(X,Y)}<\epsilon/2,
\end{equation*}
i.e. \eqref{eq:F-FEDleqeps_sobo} holds with ${\mathcal E}^X={\mathcal E}^X_d, {\mathcal D}^Y={\mathcal D}^Y_m$.
Then, following the proof of  Theorem \ref{th:UAT_Sobolev_Banach}, we have the claim.

{(ii) The proof of (ii) is similar to the proof of (i)  using Remark \ref{rem:Sobole_sp_norm_modified}.}
\end{proof}
\begin{proof}[Proof of Corollary \ref{th:UAT_Sobolev_K_Banach_bump}] The density claim  follows directly from Theorem \ref{th:UAT_Sobolev_Banach} and \eqref{eq:tildeC_subset_C}. To show that UA is achieved taking ${\mathcal E}^X={\mathcal E}^X_d, {\mathcal D}^Y={\mathcal D}^Y_m$ for some $d,m$, we proceed by combining the proof of \eqref{eq:tildeC_subset_C} with the proof of Corollary \ref{th:UAT_Sobolev_K_Banach}.
\end{proof}
\paragraph{Acknowledgments.} The author is grateful to Davide Addona for useful discussions on Sobolev spaces under Gaussian measures in infinite-dimensions.

The author acknowledges funding by the Deutsche Forschungsgemeinschaft (DFG, German Research Foundation) – CRC/TRR 388 ``Rough Analysis, Stochastic Dynamics and Related Fields'' – Project ID 516748464.
\begin{small}
\bibliography{refs}
\bibliographystyle{plain}
\end{small}
\end{document}